\documentclass[11pt]{article}
\usepackage[T1]{fontenc}
\usepackage[utf8]{inputenc}
\usepackage[english]{babel}
\usepackage{csquotes}
\usepackage[margin=1in]{geometry}

\usepackage{amsmath,amssymb,amsfonts,amsthm,mathtools,amsxtra,bm,mathrsfs}
\usepackage{aliascnt}
\usepackage{xparse}
\usepackage{bbm,dsfont}

\usepackage{array,tabularx,booktabs}
\usepackage{graphicx,subfig}
\usepackage{enumitem}
\usepackage{tikz-cd}

\usepackage{microtype}
\usepackage[normalem]{ulem}
\usepackage[dvipsnames]{xcolor}
\usepackage[color=yellow]{todonotes}
\usepackage{mdframed}

\usepackage{hyperref}
\hypersetup{
  colorlinks=true,
  breaklinks=true,
  linkcolor=blue,
  urlcolor=blue,
  anchorcolor=blue,
  citecolor=blue
}
\usepackage[nameinlink,capitalise,noabbrev]{cleveref}

\usepackage[
  backend=bibtex,
  style=alphabetic,
  maxcitenames=3,
  maxbibnames=99,
  sorting=nyt,
  url=false,
  doi=false,
  backref=false
]{biblatex}

\newcommand{\R}{\mathbb R}

\newcommand{\Law}{\mathrm{Law}}
\newcommand{\KL}{\mathrm{KL}}
\newcommand{\KSD}{\mathrm{KSD}}

\newcommand{\E}{\mathbb E}

\newtheorem{theorem}{Theorem}[section]

\newaliascnt{proposition}{theorem}
\newtheorem{proposition}[proposition]{Proposition}
\aliascntresetthe{proposition}

\newaliascnt{lemma}{theorem}
\newtheorem{lemma}[lemma]{Lemma}
\aliascntresetthe{lemma}

\theoremstyle{definition}
\newaliascnt{definition}{theorem}

\aliascntresetthe{definition}

\newaliascnt{exercise}{theorem}

\aliascntresetthe{exercise}

\newaliascnt{assumption}{theorem}
\newtheorem{assumption}[assumption]{Assumption}
\aliascntresetthe{assumption}

\theoremstyle{remark}
\newaliascnt{remark}{theorem}
\newtheorem{remark}[remark]{Remark}
\aliascntresetthe{remark}

\newaliascnt{example}{theorem}

\aliascntresetthe{example}

\newaliascnt{corollary}{theorem}
\newtheorem{corollary}[corollary]{Corollary}
\aliascntresetthe{corollary}

\numberwithin{equation}{section}

\title{Quantitative Target Convergence and Uniform-in-Time\\
Propagation of Chaos for Langevin-Regularized SVGD}
\author{
Sayan Banerjee%
\thanks{Department of Statistics and Operations Research,
University of North Carolina at Chapel Hill,
318 Hanes Hall, CB\# 3260,
Chapel Hill, NC 27599-3260, USA.
Email: \texttt{sayan@email.unc.edu}.}
\and
Dohyeon Kim%
\thanks{Department of Computing and Mathematical Sciences,
California Institute of Technology,
1200 East California Boulevard, MC 305-16,
Pasadena, CA 91125, USA.
Email: \texttt{dohyeon@caltech.edu}.}
}
\date{}

\begin{document}
\maketitle

\begin{abstract}
We establish quantitative convergence to the target and uniform-in-time
propagation of chaos for Langevin-regularized Stein variational gradient
descent.  The Stein interaction need not be small relative to the confining
Langevin drift and does not generally yield a contractive particle coupling.
At the mean-field level, the
Stein and Langevin components dissipate the same relative entropy in the
kernel-induced Stein and \(2\)-Wasserstein geometries, producing,
respectively, the squared kernel Stein discrepancy and the relative Fisher
information.  Under a log-Sobolev inequality for the target, this gives
exponential last-iterate convergence in time.  We derive a
corresponding finite-particle entropy identity relative to the product target,
which yields exponential-in-time convergence for the empirical measure to the target up to polynomial sampling
errors.

For propagation of chaos, we develop two complementary finite-time
approaches.  A synchronous coupling, combined with exponential moment estimates for the nonlinear mean-field diffusion process, yields explicit single-exponential bounds in Wasserstein distance and kernel Stein discrepancy (KSD).
Moving-product entropy gives joint-law relative entropy control with respect to the evolving mean-field product law and,
through entropy superadditivity and concentration, fixed-marginal
relative entropy and total variation bounds and empirical KSD estimates.
Under an additional \(T_2\) inequality for the initial law, it also provides an
alternative route to Wasserstein bounds.  Combining these finite-time
estimates with target convergence rates at a logarithmic (in sample size) cutoff time gives
polynomial uniform-in-time propagation of chaos rates in expectation for empirical KSD and
\(W_2^2\), and for fixed-marginal total variation and \(W_2^2\).  All bounds
control the last iterate in physical time.  We also compare the two
finite-time mechanisms and exhibit regimes in which each produces the
sharper polynomial exponent.\\



\end{abstract}

\section{Introduction}\label{sec:introduction}

\subsection{Background: Deterministic and Langevin-regularized SVGD}

Stein variational gradient descent (SVGD) is an interacting-particle
method for obtaining approximate samples from a target probability law
\(
    \pi(dx)=Z^{-1}e^{-V(x)}\,dx
\)
by evolving particles through a kernelized Stein velocity \cite{liuWang2016}.  Given a positive-definite
kernel \(k\), the particles are transported by a velocity that combines
attraction toward regions of high target probability with a repulsive
kernel interaction.  At the mean-field level, SVGD is a
gradient flow of the relative entropy in a kernel-dependent Stein geometry
\cite{liu2017,duncanNuskenSzpruch2023}.  In particular, if
\[
    v_\rho(x)
    :=\int_{\R^d}
       \{-k(x,y)\nabla V(y)+\nabla_2k(x,y)\}\,\rho(dy),
\]
then the deterministic mean-field equation
\(
    \partial_t\mu_t=-\nabla\cdot(\mu_t v_{\mu_t})
\)
formally satisfies
\[
    \frac{d}{dt}\KL(\mu_t\|\pi)
    =-\KSD^2(\mu_t\|\pi),
\]
where $\KL(\rho\|\pi):=\int_{\R^d} \log\Big(\frac{d\rho}{d\pi}\Big)\,d\rho$ denotes the Kullback-Leibler divergence (relative entropy) and $\KSD(\rho \| \pi)$ denotes the kernel Stein discrepancy~\cite{liu2016kernelized,chwialkowski2016kernel} given by the norm of $v_{\rho}(\cdot)$ in a vector-valued reproducing kernel Hilbert space (RKHS) $\mathcal H^d$ generated by \(k\); see \eqref{eq:one-sample-ksd}.
At the mean-field level, the entropy-dissipation identity relates convergence of $\mu_t$ to $\pi$ to the decay of $\KSD$. Integration by parts represents the projected gradient flow on $\mathcal H^d$ through the velocity $v_\rho$, whose integral representation depends linearly on $\rho$. Replacing $\rho$ by an empirical measure gives an implementable particle approximation of the mean-field flow. The finite-particle continuous-time dynamics takes the following form:
\begin{equation}\label{eq:svgd-continuous}
\dot x_i^N(t)
=
-\frac1N\sum_{j=1}^N
k\bigl(x_i^N(t),x_j^N(t)\bigr)\nabla V\bigl(x_j^N(t)\bigr)
+\frac1N\sum_{j=1}^N
\nabla_2 k\bigl(x_i^N(t),x_j^N(t)\bigr), \ 1\le i\le N, \, t \ge 0.
\end{equation}

This is a \emph{deterministic} evolution equation; randomness enters through the particle initialization. Standard Markov chain Monte Carlo methods typically introduce randomness at successive transitions. SVGD is deterministic conditional on the initial particle configuration: the initial empirical measure is transported by an interacting-particle flow.

The long-time behavior for deterministic SVGD has developed along both mean-field and
finite-particle directions.  At the mean-field level, quantitative
time-averaged KSD convergence has been established under various smoothness
and transport assumptions \cite{lulu2019,korba2020,salimSunRichtarik2022,
sunKaragulyanRichtarik2023}.  More recent work obtains exponential or
last-iterate convergence under Stein log-Sobolev and related coercivity
inequalities for suitably chosen kernels
\cite{carrilloSkrzeczkowskiWarnett2024,
chizatColomboColomboFernandezReal2026}.  At the particle level,
\cite{shiMackey2023} gave the first non-asymptotic finite-particle KSD
guarantee, while \cite{banerjee2024} obtained
near-parametric time-averaged KSD rates, Wasserstein bounds, and long-time
marginal propagation of chaos.  The cutoff approach of \cite{balasubramanian2026}, which is conceptually related to the method of \cite{glasgow2026uniform}, subsequently yielded uniform-in-averaging-time
propagation of chaos in broad discrepancies, together with uniform-in-physical-time rates for special finite-rank target--kernel pairs. Uniform-in-time rates had previously been obtained for Gaussian targets and bilinear kernels by \cite{liu2023towards} using specific properties of such target--kernel pairs. \cite{he2026finite} extended the methods of \cite{banerjee2024} to a resolvent-type regularization of SVGD introduced in \cite{he2025regularized}, and derived finite-particle guarantees in Fisher information.

The purely transport-type dynamics of deterministic SVGD create several
theoretical and practical limitations.  The strongest available
finite-particle bounds, although nearly parametric in \(N\), decay only
algebraically in the time horizon and apply to \emph{time-averaged}
empirical measures \cite{banerjee2024}.  Quantitative last-iterate
convergence to the target in physical time has only recently been established
at the mean-field level, using kernels with strong high-frequency coercivity,
such as weighted Bessel and Riesz--Coulomb-type kernels
\cite{carrilloSkrzeczkowskiWarnett2024,
chizatColomboColomboFernandezReal2026}; no comparable general result is
currently available for the original finite-particle dynamics.  A further
practical concern is \emph{mode collapse}: the kernel-induced repulsion may
become too weak, particularly in high dimensions or for well-separated
multimodal targets, causing the particles to underestimate the target
variance or to concentrate on only a subset of its modes
\cite{zhuo2018message,dangelo2021annealed,ba2022variance}.

Closely related noise-regularized particle optimization dynamics appeared
earlier in stochastic particle-optimization sampling (SPOS)
\cite{zhang2020stochastic}. Here we study the 
\emph{Langevin-regularized (noisy)} SVGD dynamics, which is the continuous-time version of the noisy-SVGD recursion analyzed in \cite{priser2024}. Its
particle dynamics are given by
\begin{equation}\label{eq:intro-particle}
    dX_i^N(t)
    =\left\{
       \frac1N\sum_{j=1}^N B(X_i^N(t),X_j^N(t))
       -\varepsilon\nabla V(X_i^N(t))
     \right\}dt
     +\sqrt{2\varepsilon}\,dW_i(t),
\end{equation}
for $1\le i\le N, \, t \ge 0$, where $\{W_i\}$ are independent standard Brownian motions,
\[
    B(x,y):=-k(x,y)\nabla V(y)+\nabla_2k(x,y),
\]
and \(\varepsilon>0\) is the regularization strength.   
The corresponding
nonlinear law solves
\begin{equation}\label{eq:intro-mean-field}
    \partial_t\mu_t
    =-\nabla\cdot(\mu_t v_{\mu_t})
      +\varepsilon\nabla\cdot\left(
         \mu_t\nabla\log\frac{\mu_t}{\pi}
       \right).
\end{equation}
The two components of the dynamics dissipate the same relative entropy in
different geometries.  In the integrated sense of
Proposition~\ref{prop:mf-entropy}, and equivalently for almost every
\(t>0\),
\begin{equation}\label{eq:intro-entropy-dissipation}
    \frac{d}{dt}\KL(\mu_t\|\pi)
    =-\KSD^2(\mu_t\|\pi)-\varepsilon I(\mu_t\|\pi),
\end{equation}
where $I(\rho \| \pi) := \int_{\R^d}\left\|\nabla \log\Big(\frac{d\rho}{d\pi}\Big)\right\|^2\,d\rho$ denotes the Fisher information. The first term represents dissipation in the kernel-induced Stein geometry, whereas the second retains the usual Wasserstein-2 gradient-flow structure.
Under a log-Sobolev inequality (LSI) for \(\pi\), the second term yields
exponential convergence to the target.  Priser, Bianchi, and Salim
\cite{priser2024} study the corresponding discrete-time noisy-SVGD
recursion.  By comparing its trajectories with the associated
McKean--Vlasov dynamics, they characterize the long-time limit set of the
finite-particle empirical-measure process and show that this set approaches
the target, in a suitable sense, as the number of particles tends to
infinity.  They also show theoretically and numerically that the Langevin
regularization prevents the mode-collapse behavior observed for
deterministic SVGD.

A complementary recent result is due to Hong, Liu, and Yang
\cite{hongLiuYang2026}.  They prove quantitative strong propagation of chaos
on each fixed interval \([0,T]\) for a one-dimensional
stochastic-SVGD-type system with superlinear polynomial kernels, using a
general pseudo-monotonicity framework and high polynomial moments.  Their
result allows substantially less regular interaction growth than the
bounded-kernel setting considered here, but it does not address convergence
to the target or uniformity over an infinite physical-time horizon.

Table~\ref{tab:related-work} distinguishes the results most directly related
to the present paper. Each result concern different dynamics or conditions on the kernel or moments.

\begin{table}[t]
\centering
\caption{Comparison with closely related work and the main-text results of
the present paper.}
\label{tab:related-work}
\scriptsize
\setlength{\tabcolsep}{3pt}
\renewcommand{\arraystretch}{1.18}
\begin{tabular}{p{0.15\textwidth}p{0.22\textwidth}p{0.25\textwidth}p{0.29\textwidth}}
\hline
Work & Dynamics & Principal conditions & Quantitative conclusion \\
\hline
SPOS \cite{zhang2020stochastic}
& Closely related noisy particle-optimization/SVGD system
& Strong monotonicity and a positive contraction margin
& Continuous-time one-particle uniform-in-time \(W_1=O(N^{-1/2})\), together with discrete-time approximation bounds \\
Priser--Bianchi--Salim \cite{priser2024}
& Decreasing-step noisy-SVGD recursion associated with the continuous
Langevin-regularized dynamics \eqref{eq:intro-particle}
& Qualitative well-posedness and long-time assumptions
& Characterization of the empirical-measure limit set and its large-\(N\) target consistency; no explicit rate in \(N\) \\
Deterministic cutoff \cite{balasubramanian2026}
& Deterministic SVGD
& Target--kernel-specific finite- and long-time estimates
& Uniform-in-averaging-time rates in general discrepancies; uniform-in-physical-time parametric rates for special finite-rank target--kernel pairs \\
Hong--Liu--Yang \cite{hongLiuYang2026}
& One-dimensional stochastic-SVGD-type dynamics with superlinear polynomial kernels
& Pseudo/local monotonicity and sufficiently high polynomial moments
& Quantitative strong tagged-particle propagation of chaos on every fixed
\([0,T]\); no target-convergence or infinite-horizon result \\
Current paper
& Langevin-regularized SVGD \eqref{eq:intro-particle}
& Target LSI, smooth bounded kernels, and concentration assumptions for strong finite-time bounds; no small-interaction condition
& Polynomial uniform-in-physical-time rates for expected empirical KSD and
\(W_2^2\), and for fixed-marginal TV and \(W_2^2\); finite-time
fixed-marginal KL control \\
\hline
\end{tabular}
\end{table}

The distinction between averaging time and physical time is important.
For general deterministic kernels, \cite{balasubramanian2026} controls
time-averaged empirical measures uniformly in the averaging horizon.  Its uniform-in-physical-time conclusions rely on finite-dimensional
closure for special finite-rank target--kernel pairs. The bounds proved
here instead concern the last iterate at every physical time.

In the current paper, we quantify the long-time finite-particle behavior of
noisy SVGD by obtaining \emph{rates of convergence to the target} that are
polynomial in $N$ and exponential in $t$.  We combine these estimates with
single-exponential moving-product entropy and synchronous coupling bounds, and then apply the cutoff argument of \cite{balasubramanian2026} to obtain
polynomial rates in $N$ for \emph{uniform-in-time propagation of chaos in
expectation}.
\subsection{Propagation of chaos and the long-time problem}
Let
\(
    \mu_t^N=N^{-1}\sum_{i=1}^N\delta_{X_i^N(t)}
\)
be the empirical measure of \eqref{eq:intro-particle}, let \(P_t^N\) be
the joint particle law, and let \(\mu_t\) solve
\eqref{eq:intro-mean-field}.  Propagation of chaos (PoC) asserts that, for every
fixed time horizon $t$, \(P_t^N\) becomes asymptotically of product-form with one-particle
law \(\mu_t\) as $N \to \infty$, or equivalently that \(\mu_t^N\) approaches \(\mu_t\) in
suitable metrics; see \cite{sznitman1991,chaintronDiez2022}.  Uniform-in-time
PoC asks for this comparison over all physical times: for some metric or discrepancy $d$ over probability measures,
\[
    \sup_{t\ge0}\E d(\mu_t^N,\mu_t)\longrightarrow0,
    \qquad N\to\infty.
\]
This is substantially stronger than a fixed-horizon result because the
constants in classical propagation of chaos estimates usually grow (exponentially or even double-exponentially) with
time.

A standard setting for uniform-in-time PoC consists of
weakly interacting diffusions subject to a confining force. The particle evolution is described by the equations
\begin{equation}\label{mv}
    dX_i^N(t)
    =
    -\left\{
       \nabla U(X_i^N(t))
       +\frac1N\sum_{j=1}^N
          \nabla W\bigl(X_i^N(t),X_j^N(t)\bigr)
     \right\}dt
     +\sigma\,dB_i(t),
    \qquad 1\le i\le N,
\end{equation}
and its mean-field limit is characterized by the so-called \emph{McKean--Vlasov} diffusion
\[
    d\bar X_t
    =
    -\left\{
       \nabla U(\bar X_t)
       +\int_{\mathbb R^d}
          \nabla W(\bar X_t,y)\,\rho_t(dy)
     \right\}dt
     +\sigma\,dB_t,
    \qquad
    \rho_t=\Law(\bar X_t).
\]
Here $\{B_i,B\}$ are standard Brownian motions, $\{B_i\}$ are independent, $\sigma$ is a constant diffusivity, $U$ is called the confinement potential and $W$ is the interaction potential.

When the confinement potential \(U\) is uniformly convex and the interaction potential \(W\) is convex,
or when the nonconvex part of the interaction is sufficiently weak relative
to the convex confinement, the resulting dynamics possess several favorable contractive properties. These are exploited via synchronous or reflection couplings and entropy methods to yield uniform-in-time
control; such arguments underlie many classical
results
\cite{malrieu2003,cattiaux2008granular,durmusEberleGuillinZimmer2020,guillin2021uniform,lackerLeFlem2023}.
Hierarchical relative entropy methods can further provide sharp marginal
rates, but their time-uniform versions similarly rely on log-Sobolev
inequalities arising from convexity or weak-interaction conditions
\cite{lacker2023,lackerLeFlem2023}. 
More recently, \cite{arnese2026sharppropagationchaosmean} extended the hierarchical entropy
framework beyond pairwise interactions to drifts that depend nonlinearly on
the empirical measure, obtaining sharp marginal rates on finite time horizons
and, for mean-field Langevin dynamics under strong displacement convexity,
uniformly in time.

Equation~\eqref{eq:intro-particle} resembles the McKean--Vlasov system
\eqref{mv}, although the interaction component of the drift in \eqref{eq:intro-particle} has no analogous gradient
structure.  The relative size of the two drift components is also different. The Stein velocity
\(v_\rho\) may not be small in comparison to the confining Langevin drift
\(-\varepsilon\nabla V\), particularly when \(\varepsilon\) is small.  Thus, unlike the `convex-confinement-small-interaction' models described above, the Langevin-regularized Stein
drift vector-field does not generally contract distances between
two solutions. For the deterministic SVGD (where $\varepsilon=0$), this lack of contractivity results in $O(1/\log N)$ and $O(1/\log \log N)$ uniform-in-time PoC rates obtained in \cite{balasubramanian2026}.

In this article, we show that the injected Langevin regularization does incorporate a form of contractivity into the SVGD dynamics, albeit in a different way. This results in polynomial rates in strong metrics, as we now describe.

\subsection{Main contributions}

The contributions of the paper are summarized in the following two points.\\

\textbf{1. Entropic stabilization without a small-interaction condition.}
We avoid perturbative assumptions from convexity and recover
stability at the level of entropy.  At the mean-field level this corresponds to \eqref{eq:intro-entropy-dissipation}: the entire `nonconvex' Stein interaction contributes
the nonpositive term \(-\KSD^2(\mu_t\|\pi)\). Adapting the approach in \cite{banerjee2024}, we prove the analogue of this identity at the particle level. We show in Theorem~\ref{prop:full-law} that the evolution of the relative entropy of the \emph{joint particle law} with respect to the $N$-fold product law $\pi^{\otimes N}$ exhibits a similar decay, but the resulting KSD term is computed for the \emph{particle empirical measure}.  For almost every \(t>0\),
\[
    \frac{d}{dt}\KL(P_t^N\|\pi^{\otimes N})
    =-N\E\KSD^2(\mu_t^N\|\pi)
      -\varepsilon I(P_t^N\|\pi^{\otimes N})
      +\text{a diagonal term},
\]
where \eqref{eq:Cstar-bound}, imposed in either kernel regime, bounds the
diagonal correction from above, uniformly in \(N\), by the order-one
constant \(C^\star\).
After tensorization of the target LSI, this yields pointwise exponential-in-time and polynomial-in-$N$
target convergence for the mean-field and particle systems, together with
uniform second-moment bounds.  Stabilization comes from the Stein
entropy-production identity, and the driving vector field need not be convex. Theorem~\ref{prop:full-law} then yields the pointwise particle target rates
in Propositions~\ref{prop:particle-KSD-target} and
\ref{prop:particle-W2-target}.\\

\textbf{2. Moving-product entropy, synchronous coupling, and polynomial
uniform-in-time propagation of chaos.}
A major contribution of this article lies in quantifying finite-time PoC via two complementary approaches. 

In the first approach, we synchronously couple the particle system in \eqref{eq:intro-particle} with $N$-independent copies of the nonlinear McKean--Vlasov process in \eqref{eq:intro-mean-field} driven by the same collection of Brownian motions. A careful
analysis of their averaged squared distance produces a Gr\"onwall
inequality whose random coefficient and forcing depend only on the
independent nonlinear copies, through
\(
    Q_N(t):=\frac1N\sum_{j=1}^N\|\bar X_j(t)\|^2
\)
and the empirical Stein-velocity defect
\(
    G_t^N(x_1,\ldots,x_N)
    :=
    \sum_{i=1}^N
    \left\|
        \frac1N\sum_{j=1}^N B(x_i,x_j)
        -v_{\mu_t}(x_i)
    \right\|^2;
\)
see \eqref{eq:DN-differential}.  Exponential moment bounds for \(Q_N\),
together with concentration of \(G_t^N\), close this estimate in
expectation. Proposition~\ref{prop:refined-coupling} and
Corollaries~\ref{cor:finite-W2} and
\ref{cor:finite-marginal-W2} then give
\begin{align*}
    \E W_2^2(\mu_t^N,\mu_t)
    &\le C\left\{
       \frac{e^{L_\mathrm{cpl}t}-1}{N}+r_{N,d}\right\},\\
     W_2^2(P_{N,t}^{(k)},\mu_t^{\otimes k})
  &\le C\frac{k}{N}\bigl(e^{L_\mathrm{cpl}t}-1\bigr),
\end{align*}
for some finite exponent $L_\mathrm{cpl}$.

The second approach adapts the quantitative relative entropy method of
\cite{jabinWang2018}. We analyze the time evolution of the moving-product relative entropy
\[
  J_N(t):=\KL\!\left(P_t^N\,\middle\|\,\mu_t^{\otimes N}\right).
\]
We derive a differential inequality for \(J_N\). The relative Fisher information generated by the nondegenerate diffusion
absorbs the `interaction error' coming from the different drifts in the particle and mean-field systems. An exponential law-of-large-numbers estimate, proved in Proposition~\ref{prop:exp-lln}, then
closes the differential inequality through the variational
representation of relative entropy. Proposition~\ref{prop:moving-entropy} thus gives
\[
  J_N(t)\le\bigl(e^{\kappa t}-1\bigr)\log C_G,
  \qquad \kappa=(2\varepsilon\lambda)^{-1},
\]
under i.i.d.\ initialization, where \(C_G<\infty\) is independent of
\(N\) and \(t\). 

Entropy superadditivity, and then Pinsker's inequality, respectively convert this bound into
fixed-marginal KL and total variation estimates: for every \(1\le k\le N/2\):
\begin{align*}
    \KL(P_{N,t}^{(k)}\|\mu_t^{\otimes k})
    &\le
    2\log C_G\,\frac{k}{N}
    \bigl(e^{\kappa t}-1\bigr),\\
    \|P_{N,t}^{(k)}-\mu_t^{\otimes k}\|_{\mathrm{TV}}
    &\le
    \sqrt{\log C_G}\,
    \sqrt{\frac{k}{N}}\,
    \bigl(e^{\kappa t}-1\bigr)^{1/2}.
\end{align*}

Finally, we combine the finite-time PoC rates (provided by either of the above approaches) with the target-convergence estimates by balancing these bounds at a logarithmic cutoff, described in Section~\ref{sec:cutoff}. This yields our main uniform-in-time PoC results with polynomial rates in KSD, $W_2$, and total variation (TV) distance. These estimates control the last iterate at every physical time, rather than the time-averaged measure treated in \cite{balasubramanian2026}.

In Section~\ref{app:entropy-W2}, we demonstrate how the moving-entropy method, under an additional transportation inequality assumption on the initial distribution, gives an alternate approach to finite-time and uniform-in-time PoC for the empirical distribution and particle marginals in $W_2$. Although both approaches yield qualitatively similar rates, the associated exponents are sensitive to the regularity of the kernel. In Section~\ref{sec:W2-method-comparison}, model-dependent estimates of the
PoC exponents and explicit examples show that the moving-entropy bounds can
be sharper for kernels with large spatial derivatives or strongly
oscillatory behavior.  Conversely, synchronous coupling can yield better
rates when the dynamics possess a favorable pathwise contractive structure.

Section~\ref{sec:smooth-stationary-kernels} verifies the assumptions for a
broad class of symmetric \(C_b^4\) stationary kernels and for targets with
both convex and nonconvex potentials.  Section~\ref{sec:future}
discusses a possible localization route to uniform-in-time PoC in
probability under weaker moment assumptions, as well as the dependence of
the rates on the regularization strength \(\varepsilon\).  Extensions to
other stochastic variants of SVGD provide a further direction for future
work
\cite{zhang2020stochastic,liEtAl2020randomBatch,
nuskenRenger2023,dasNagaraj2023}.

We note here that a related cutoff argument was used in \cite{guillin2021uniform} for mean-field kinetic particles.
Under their assumptions, the interacting \(N\)-particle Gibbs measures
satisfy a log-Sobolev inequality uniformly in \(N\), which yields
\(N\)-uniform hypocoercive relaxation and provide the long-time estimate of their cutoff argument. In our work, we do not assume such a corresponding functional inequality for
the invariant law of the noisy-SVGD particle system. Instead, our long-time control follows directly from Stein entropy dissipation relative to the product target \(\pi^{\otimes N}\), up to the finite-\(N\) diagonal correction.\\

\textbf{Summary of PoC results: }
Table~\ref{tab:main-results} summarizes our main
propagation of chaos results.  Every row assumes
i. the core target and initial-law conditions in
Assumption~\ref{ass:core-target}, and ii. the dissipativity and sub-Gaussian initial-tail conditions in
Assumption~\ref{ass:concentration}.

The remaining assumption on the kernel is stated in the first column.

The constants \(\kappa\) and \(L_\mathrm{cpl}\) are the
single-exponential growth rates in the moving-entropy and
synchronous-coupling estimates, respectively; see \eqref{eq:kappa} and
Proposition~\ref{prop:refined-coupling}.  The term \(r_{N,d}\), defined in
\eqref{eq:rNd}, is the i.i.d.\ empirical \(W_2^2\) sampling rate.  In the
fixed-marginal rows, \(k\) is fixed and \(N\ge2k\), and \(C_k\) is
independent of \(N\) and \(t\).

The last row gives empirical and fixed-marginal PoC bounds in \(W_2^2\) under the additional assumption that
\(\mu_0\in T_2(C_0)\). There, $L_{\mathrm{ent}} = \kappa + \ell$, where $\mu_t \in T_2(C e^{\ell t})$, shown in Corollary~\ref{cor:mu-T2}.

\begin{table}[!htbp]
\centering
\caption{Main metric propagation of chaos bounds.}
\label{tab:main-results}
\footnotesize
\setlength{\tabcolsep}{3pt}
\renewcommand{\arraystretch}{1.35}
\begin{tabular}{@{}p{0.25\textwidth}p{0.29\textwidth}p{0.37\textwidth}@{}}
\hline
Discrepancy and kernel condition
& Finite-time estimate
& Uniform-in-time conclusion \\
\hline
Empirical Langevin KSD
\newline\emph{KSD kernel condition}
&
\(\displaystyle \E\KSD_\pi(\mu_t^N,\mu_t)\)
\newline
\(\displaystyle \lesssim N^{-1/2}e^{\left(\kappa \wedge (L_\mathrm{cpl}/2)\right) t}\)
&
\(\displaystyle
  \sup_{t\ge0}\E\KSD_\pi(\mu_t^N,\mu_t)\)
\newline
\(\displaystyle
  \lesssim
  N^{-\varepsilon\alpha/[2(\kappa \wedge (L_\mathrm{cpl}/2) +\varepsilon\alpha)]}
  +N^{-1/2}\)
\\
\hline
Empirical \(W_2^2\)
\newline\emph{Wasserstein kernel condition}
&
\(\displaystyle \E W_2^2(\mu_t^N,\mu_t)\)
\newline
\(\displaystyle
  \lesssim
  N^{-1}(e^{L_\mathrm{cpl}t}-1)+r_{N,d}\)
&
\(\displaystyle
  \sup_{t\ge0}\E W_2^2(\mu_t^N,\mu_t)\)
\newline
\(\displaystyle
  \lesssim
  N^{-2\varepsilon\alpha/
    (L_\mathrm{cpl}+2\varepsilon\alpha)}
  +r_{N,d}\)
\\
\hline
Fixed-\(k\) \(W_2^2\)
\newline\emph{Wasserstein kernel condition}
&
\(\displaystyle
  W_2^2(P_{N,t}^{(k)},\mu_t^{\otimes k})\)
\newline
\(\displaystyle
  \lesssim
  \frac{k}{N}(e^{L_\mathrm{cpl}t}-1)\)
&
\(\displaystyle
  \sup_{t\ge0}
  W_2^2(P_{N,t}^{(k)},\mu_t^{\otimes k})\)
\newline
\(\displaystyle
  \le C_k\left\{
  N^{-2\varepsilon\alpha/
    (L_\mathrm{cpl}+2\varepsilon\alpha)}
  +N^{-1}\right\}\)
\\
\hline
Fixed-\(k\) total variation
\newline\emph{Either kernel condition}
&
\(\displaystyle
  \|P_{N,t}^{(k)}-\mu_t^{\otimes k}\|_{\mathrm{TV}}\)
\newline
\(\displaystyle
  \lesssim
  \sqrt{\frac{k}{N}}\,(e^{\kappa t}-1)^{1/2}\)
&
\(\displaystyle
  \sup_{t\ge0}
  \|P_{N,t}^{(k)}-\mu_t^{\otimes k}\|_{\mathrm{TV}}\)
\newline
\(\displaystyle
  \le C_k\left\{
  N^{-\varepsilon\alpha/(\kappa+2\varepsilon\alpha)}
  +N^{-1/2}\right\}\)
\\
\hline
Entropy-derived \(W_2^2\)
\newline empirical and fixed-\(k\)
\newline\emph{Wasserstein kernel condition and
\(\mu_0\in T_2(C_0)\)}
&
\(\displaystyle
  \E W_2^2(\mu_t^N,\mu_t)\)
  \newline
  \(\displaystyle\lesssim N^{-1}e^{L_{\mathrm{ent}}t}+r_{N,d}\)
\newline
\newline
\(\displaystyle
  W_2^2(P_{N,t}^{(k)},\mu_t^{\otimes k})
  \lesssim \frac{k}{N}e^{L_{\mathrm{ent}}t}\)
&
\(\displaystyle
  \sup_{t\ge0}\E W_2^2(\mu_t^N,\mu_t)\)
  \newline
  \(\displaystyle\lesssim
  N^{-2\varepsilon\alpha/
    (L_{\mathrm{ent}}+2\varepsilon\alpha)}
  +r_{N,d}\)
\newline
\newline
\(\displaystyle
  \sup_{t\ge0}
  W_2^2(P_{N,t}^{(k)},\mu_t^{\otimes k})\)
  \newline
\(\displaystyle \le C_k\left\{
  N^{-2\varepsilon\alpha/
    (L_{\mathrm{ent}}+2\varepsilon\alpha)}
  +N^{-1}\right\}\)
\\
\hline
\end{tabular}
\end{table}

To pass from the finite-time estimates in the second column to the
uniform-in-time conclusions in the third via the cutoff argument of
Section~\ref{sec:cutoff}, we establish long-time bounds for the same four
discrepancies by comparing both the particle and mean-field systems with the
target.  These target-regime estimates are also of independent interest:
\begin{align*}
  \E\KSD_\pi(\mu_t^N,\mu_t)
  &\lesssim e^{-\varepsilon\alpha t}+N^{-1/2},\\
  \E W_2^2(\mu_t^N,\mu_t)
  &\lesssim e^{-2\varepsilon\alpha t}+r_{N,d},\\
  \|P_{N,t}^{(k)}-\mu_t^{\otimes k}\|_{\mathrm{TV}}
  &\le C_k(e^{-\varepsilon\alpha t}+N^{-1/2}),\\
  W_2^2(P_{N,t}^{(k)},\mu_t^{\otimes k})
  &\le C_k(e^{-2\varepsilon\alpha t}+N^{-1}).
\end{align*}

\subsection{Organization}

In Section~\ref{sec:setup}, we introduce the model and notation, then in
Section~\ref{sec:assumptions}, we state the assumptions and establish
well-posedness. We then derive quantitative target
convergence and uniform moment bounds in Section~\ref{sec:target}, which is followed by Section~\ref{sec:finite-time}, where we  develop the concentration and synchronous-coupling estimates. Finally in Section~\ref{sec:cutoff}, we convert the estimates in Section~\ref{sec:finite-time} into uniform-in-time PoC rates.

Section~\ref{sec:moving-entropy} develops the moving-product entropy method,
its marginal, total variation, KSD, and optional \(W_2\) estimates, and
compares it with the result from synchronous coupling. Section~\ref{sec:smooth-stationary-kernels}
verifies the assumptions in concrete convex and nonconvex examples, and
Section~\ref{sec:future} discusses extensions and open directions.  The
appendices justify the entropy evolution formulas and collect the required
regularity results.

\section{Setup and notation}\label{sec:setup}

Let
\[
    \pi(dx)=Z^{-1}e^{-V(x)}\,dx
\]
be the target law on $\R^d$.  Let $k$ be a scalar symmetric positive-definite kernel, and let $\mathcal H^d$ be the associated vector-valued RKHS.  Following \cite{balasubramanian2026}, define the Stein velocity
\begin{equation}\label{eq:velocity}
    v_\rho(x)
    :=\int_{\R^d}
      \{-k(x,y)\nabla V(y)+\nabla_2k(x,y)\}\,\rho(dy).
\end{equation}
For later calculations we write
\begin{equation}\label{eq:B}
    B(x,y):=-k(x,y)\nabla V(y)+\nabla_2k(x,y),
    \qquad v_\rho(x)=\int B(x,y)\,\rho(dy).
\end{equation}
The Langevin--Stein operator is
\[
    \mathcal T_\pi\phi(x)
    :=-\nabla V(x)\cdot\phi(x)+\nabla\cdot\phi(x),
    \qquad \phi\in\mathcal H^d.
\]
Set
\[
    \xi_\pi(x)
    :=-k(\cdot,x)\nabla V(x)+\nabla_2k(\cdot,x).
\]
Whenever \(x\mapsto\xi_\pi(x)\) is strongly measurable and
\[
    \int_{\mathbb R^d}
       \|\xi_\pi(x)\|_{\mathcal H^d}\,\rho(dx)<\infty,
\]
define the Stein witness by the Bochner integral
\begin{equation}\label{eq:witness}
    S_\pi(\rho)
    :=\int_{\mathbb R^d}\xi_\pi(x)\,\rho(dx)
    \in\mathcal H^d.
\end{equation}
Lemma~\ref{lem:witness-lipschitz} shows that, under either kernel
assumption, this definition applies to every
\(\rho\in\mathcal P_1(\mathbb R^d)\), and that \(S_\pi(\pi)=0\).
The one-sample KSD and two-sample Langevin KSD are
\begin{align}
    \KSD(\rho\|\pi)
    &:=\|S_\pi(\rho)\|_{\mathcal H^d},
    \label{eq:one-sample-ksd}\\
    \KSD_\pi(\mu,\nu)
    &:=\|S_\pi(\mu)-S_\pi(\nu)\|_{\mathcal H^d}.
    \label{eq:two-sample-ksd}
\end{align}
Thus $\KSD_\pi$ is symmetric and satisfies the triangle inequality.  The
cutoff Stein identity in Lemma~\ref{lem:witness-lipschitz} gives
$S_\pi(\pi)=0$, and therefore
\begin{equation}\label{eq:ksd-triangle-target}
    \KSD_\pi(\rho,\pi)=\KSD(\rho\|\pi),
    \qquad
    \KSD_\pi(\mu,\nu)
    \le \KSD(\mu\|\pi)+\KSD(\nu\|\pi).
\end{equation}

The diffusion below is the continuous-time interacting dynamics associated
with the discrete-time noisy-SVGD recursion of \cite{priser2024}.  For
$1\le i\le N$, consider
\begin{equation}\label{eq:particle-sde}
    dX_i^N(t)
    =\{v_{\mu_t^N}(X_i^N(t))-\varepsilon\nabla V(X_i^N(t))\}\,dt
     +\sqrt{2\varepsilon}\,dW_i(t),
\end{equation}
where
\[
    \mu_t^N:=\frac1N\sum_{i=1}^N\delta_{X_i^N(t)}.
\]
The nonlinear process and its law are
\begin{equation}\label{eq:nonlinear-sde}
    d\bar X(t)
    =\{v_{\mu_t}(\bar X(t))-\varepsilon\nabla V(\bar X(t))\}\,dt
     +\sqrt{2\varepsilon}\,dW(t),
    \qquad \mu_t:=\Law(\bar X(t)).
\end{equation}
Equivalently,
\begin{equation}\label{eq:mf-pde}
    \partial_t\mu_t
    =-\nabla\cdot(\mu_t v_{\mu_t})
      +\varepsilon\nabla\cdot\left(
          \mu_t\nabla\log\frac{\mu_t}{\pi}
      \right).
\end{equation}
The above flow is a Langevin-regularized RKHS projected gradient flow of the Kullback-Leibler divergence (or relative entropy), which we now formally define.
For probability measures \(\rho,\nu\in\mathcal P(\R^d)\), their
Kullback--Leibler divergence is
\[
    \KL(\rho\|\nu)
    :=
    \begin{cases}
    \displaystyle
    \int_{\R^d}
        \log\left(\frac{d\rho}{d\nu}\right)\,d\rho,
        & \rho\ll\nu,\\[1.2ex]
    +\infty, & \text{otherwise}.
    \end{cases}
\]
Equivalently, if \(f=d\rho/d\nu\), then
\(\KL(\rho\|\nu)=\int f\log f\,d\nu\), with the convention
\(0\log 0=0\).

For two probability laws on a common measurable space, we denote
\[
    \|\rho-\nu\|_{\mathrm{TV}}
    :=
    \sup_A|\rho(A)-\nu(A)|
    =\frac12\int
       \left|\frac{d\rho}{d\lambda}
             -\frac{d\nu}{d\lambda}\right|\,d\lambda,
\]
where \(\lambda\) is any measure dominating both \(\rho\) and \(\nu\).
With this convention, Pinsker's inequality reads
\[
    \|\rho-\nu\|_{\mathrm{TV}}
    \le \sqrt{\frac12\KL(\rho\|\nu)}.
\]

Now we define another closely related quantity.  For probability laws
\(\rho\ll\nu\) on a Euclidean space, with \(f=d\rho/d\nu\), the relative
Fisher information is
\[
    I(\rho\|\nu)
    :=
    4\int\left\|\nabla\sqrt{f}\right\|^2\,d\nu,
\]
whenever \(\sqrt f\) belongs to the corresponding weighted Sobolev space,
and \(I(\rho\|\nu):=+\infty\) otherwise.  If \(f\) is positive and
sufficiently regular, this can equivalently be written as
\[
    I(\rho\|\nu)
    =
    \int
    \left\|
        \nabla\log\left(\frac{d\rho}{d\nu}\right)
    \right\|^2\,d\rho .
\]
We write $P_t^N$ for the joint law of the particle system (which exists by standard parabolic regularity), $P_{N,t}^{(k)}$ for its $k$-particle marginal, and
\[
    \Pi_N:=\pi^{\otimes N}.
\]
For a probability law $\rho$ with finite second moment, set
\[
    m_2(\rho):=\int_{\R^d}\|x\|^2\rho(dx).
\]
We use the convention that $\rho\in T_2(C)$ means
\begin{equation}\label{eq:T2-convention}
    W_2^2(\nu,\rho)\le 2C\KL(\nu\|\rho)
\end{equation}
for every $\nu\ll\rho$.  In this convention, a log-Sobolev inequality (LSI) with constant $\alpha$
\[
    I(\nu\|\pi)\ge 2\alpha\KL(\nu\|\pi)
\]
implies $\pi\in T_2(1/\alpha)$. 

Unless stated otherwise, \(C\) denotes a finite constant that may depend on
the standing target, initial-law, kernel, dimension, and regularization
parameters, but not on \(N\) or \(t\).  The notation \(C_k\) allows the
constant to depend additionally on \(k\); it remains independent of \(N\)
and \(t\).

For $X = (x_1,\ldots,x_N) \in (\R^d)^N$, we will denote the associated empirical measure by
$$
\mu^N_X := \frac{1}{N}\sum_{i=1}^N\delta_{x_i}.
$$

\section{Assumptions and well-posedness}\label{sec:assumptions}

The following assumptions (or a subset of them) will be made for our results. The target LSI implies
\(\pi\in T_2(1/\alpha)\), which is distinct from the additional initial-law
condition \(\mu_0\in T_2(C_0)\), stated as
Assumption~\ref{ass:initial-T2} and used only in
Section~\ref{app:entropy-W2}.

\begin{assumption}[Core target and initial law]\label{ass:core-target}
The potential $V\in C^3(\R^d)$ satisfies
\begin{equation}\label{eq:hessian-bound}
    \sup_{x\in\R^d}\|\nabla^2V(x)\|_{\mathrm{op}}\le L_V.
\end{equation}
The target $\pi$ has finite second moment and satisfies
\begin{equation}\label{eq:target-lsi}
    I(\nu\|\pi)\ge 2\alpha\KL(\nu\|\pi)
\end{equation}
for some $\alpha>0$. The mean-field initial law $\mu_0$ has a smooth positive density and
\begin{equation}\label{eq:initial-core}
    \KL(\mu_0\|\pi)<\infty.
\end{equation}
Unless stated otherwise, the particles are initialized independently:
\[
    P_0^N=\mu_0^{\otimes N}.
\]
\end{assumption}

\begin{assumption}[KSD kernel assumptions]\label{ass:ksd-kernel}
The kernel $k$ is symmetric, positive definite, and belongs to $C_b^4(\R^d\times\R^d)$. Define
\begin{equation}\label{eq:Cstar-function}
    C^\star(x)
    :=\nabla_2k(x,x)\cdot\nabla V(x)
      +k(x,x)\Delta V(x)-\Delta_2k(x,x),
\end{equation}
and assume
\begin{equation}\label{eq:Cstar-bound}
    C^\star
    :=\max\left\{0,\sup_{x\in\R^d}C^\star(x)\right\}<\infty.
\end{equation}
\end{assumption}

\begin{assumption}[Wasserstein kernel assumptions]\label{ass:W2}
The kernel is symmetric and positive definite, belongs to $C^3(\R^d\times\R^d)$, and
\[
    k,\ \nabla_1k,\ \nabla_2k,\
    \nabla_{12}^2k,\ \nabla_{22}^2k,\ \nabla_{112}^3k
\]
are bounded. The diagonal correction \eqref{eq:Cstar-bound} holds.
\end{assumption}

\begin{assumption}[Strong-concentration regime]\label{ass:concentration}
There exist $m>0$ and $b\ge0$ such that
\begin{equation}\label{eq:dissipativity}
    x\cdot\nabla V(x)\ge m\|x\|^2-b,
\end{equation}
and the initial law satisfies
\begin{equation}\label{eq:initial-subgaussian}
    \int e^{a_0\|x\|^2}\mu_0(dx)<\infty
\end{equation}
for some $a_0>0$.
\end{assumption}


\begin{remark}[Kernel scope and comparison with the deterministic assumptions]\label{rem:assumption-comparison}
Our assumptions cover any stationary positive-definite kernel
\(k(x,y)=\Psi(x-y)\) with \(\Psi\in C_b^4(\mathbb R^d)\);
see Section~\ref{sec:smooth-stationary-kernels}.

The bounded-Hessian and dissipativity conditions are comparable to the
potential assumptions used in Sections~3--4 of
\cite{balasubramanian2026}; after a Young's inequality their dissipativity
condition also gives a quadratic radial lower bound.  The assumptions are
nevertheless not inclusive of one another.  Our target argument additionally
uses a target LSI, and the strong finite-time estimates in expectation use a
sub-Gaussian initial law.
Conversely, the deterministic $W_1$ and $W_2$ results use specialized Langevin-semigroup stability assumptions and normalized or unnormalized bilinear--Mat\'ern kernels.

Thus our results use stronger tail information in exchange for a stronger
conclusion: polynomial, uniform-in-time estimates in expectation.
Section~\ref{weakprob} discusses a possible finite-moment localization of
the coupling argument.
\end{remark}

\begin{remark}[Scope of the concentration assumption]\label{rem:exp-scope}
Assumption~\ref{ass:concentration} is not used for target convergence, the
full-law target entropy estimate, or the uniform second-moment bound.  In the
moving-entropy proof it gives the exponential law of large numbers in
Proposition~\ref{prop:exp-lln} and the Stein-witness concentration in
Lemma~\ref{lem:witness-concentration}.  The coupling proof in
Section~\ref{app:refined-coupling} uses the same assumption differently,
through a product-path exponential moment and the \(L^2\) estimate of
Proposition~\ref{prop:exp-lln}.
\end{remark}

\begin{proposition}[Well-posedness]\label{prop:wellposed}
Under Assumption~\ref{ass:core-target} and either Assumption~\ref{ass:ksd-kernel} or Assumption~\ref{ass:W2}, the particle system \eqref{eq:particle-sde} and the nonlinear SDE \eqref{eq:nonlinear-sde} admit unique global strong solutions.  Their second moments are finite on every finite time interval.
\end{proposition}

\begin{proof}
Because the target LSI implies $T_2(1/\alpha)$, Assumption~\ref{ass:core-target} gives $W_2(\mu_0,\pi)<\infty$ and hence $\mu_0\in\mathcal P_2(\R^d)$. By \eqref{eq:B-weighted-y-lipschitz},
\begin{equation}\label{eq:B-local-lip}
    \|B(x,y)-B(x',y')\|
    \le C(1+ \min\{\|y\|,\|y'\|\})
       \{\|x-x'\|+\|y-y'\|\}.
\end{equation}
Let \(\gamma\) be an optimal \(W_2\)-coupling of \(\rho\) and \(\nu\).
Then \eqref{eq:B-local-lip} and Cauchy--Schwarz inequality give
\begin{align}
    \|v_\rho(x)-v_\nu(x')\|
    &\le C\{1+m_2(\rho)^{1/2}+m_2(\nu)^{1/2}\}
       \{\|x-x'\|+W_2(\rho,\nu)\}.
    \label{eq:measure-lip}
\end{align}
Moreover, boundedness of $k$ and $\nabla_2k$, together with the linear growth of $\nabla V$, gives
\[
    \|v_\rho(x)\|
    \le C\{1+m_2(\rho)^{1/2}\}.
\]
Thus the finite-dimensional drift is locally Lipschitz and has at
most linear growth.  Let
\[
    \tau_R:=\inf\left\{t\ge0:
      \frac1N\sum_{i=1}^N\|X_i^N(t)\|^2\ge R\right\}.
\]
The preceding growth bounds and It\^o's formula give
\[
    \E\frac1N\sum_{i=1}^N
       \|X_i^N(t\wedge\tau_R)\|^2
    \le
    \E\frac1N\sum_{i=1}^N\|X_i^N(0)\|^2
    +C\int_0^t
       \left(1+\E\frac1N\sum_{i=1}^N
          \|X_i^N(s\wedge\tau_R)\|^2\right)\,ds.
\]
Hence Gr\"onwall's lemma yields, for every \(T<\infty\),
\[
    \sup_{R>0}\sup_{0\le t\le T}
    \E\frac1N\sum_{i=1}^N
       \|X_i^N(t\wedge\tau_R)\|^2
    \le C_T\left(1+\E\frac1N\sum_{i=1}^N\|X_i^N(0)\|^2\right),
\]
for some finite constant $C_T$.
Since
\[
    \mathbb P(\tau_R\le T)
    \le \frac1R\E\frac1N\sum_{i=1}^N
       \|X_i^N(T\wedge\tau_R)\|^2,
\]
letting \(R\to\infty\) rules out finite-time explosion.  A further
application of Fatou's lemma gives
\[
    \sup_{0\le t\le T}\E\frac1N\sum_{i=1}^N\|X_i^N(t)\|^2
    \le C_T\left(1+\E\frac1N\sum_{i=1}^N\|X_i^N(0)\|^2\right).
\] 
For the nonlinear SDE, \eqref{eq:measure-lip} gives local Lipschitz continuity of the McKean--Vlasov map on sets with bounded second moment.  A Picard iteration with truncated coefficients gives a unique local solution; as in the particle case above, the same quadratic estimate is uniform in the truncation and removes the stopping time. 

To prove pathwise uniqueness, let \(\bar X\) and \(\bar Y\) be two
solutions driven by the same Brownian motion and with
\(\bar X(0)=\bar Y(0)\) almost surely.
Fix \(T<\infty\).  The finite-time second-moment bounds and
\eqref{eq:measure-lip}, together with the boundedness of
\(\nabla^2V\), imply that, for \(0\le t\le T\),
\[
\begin{split}
    &\bigl\|
       v_{\Law(\bar X(t))}(\bar X(t))
       -v_{\Law(\bar Y(t))}(\bar Y(t))
       -\varepsilon\bigl(
          \nabla V(\bar X(t))-\nabla V(\bar Y(t))
       \bigr)
     \bigr\|  \\
    &\hspace{3cm}
    \le C_T\left(
       \|\bar X(t)-\bar Y(t)\|
       +W_2(\Law(\bar X(t)),\Law(\bar Y(t)))
    \right).
\end{split}
\]
The noise cancels under this synchronous coupling.  Hence, applying
the chain rule to
\(\|\bar X(t)-\bar Y(t)\|^2\) and taking expectations gives
\[
    \E\|\bar X(t)-\bar Y(t)\|^2
    \le
    C_T\int_0^t
       \left\{
          \E\|\bar X(s)-\bar Y(s)\|^2
          +W_2^2(\Law(\bar X(s)),\Law(\bar Y(s)))
       \right\}\,ds.
\]
Since the joint law of \((\bar X(s),\bar Y(s))\) is a coupling of
their marginal laws,
\[
    W_2^2(\Law(\bar X(s)),\Law(\bar Y(s)))
    \le \E\|\bar X(s)-\bar Y(s)\|^2.
\]
Gr\"onwall's lemma therefore yields
\[
    \E\|\bar X(t)-\bar Y(t)\|^2=0,
    \qquad 0\le t\le T.
\]
As \(T\) is arbitrary, \(\bar X\) and \(\bar Y\) are indistinguishable,
which proves pathwise uniqueness. This yields the global nonlinear solution and the finite-time moment bounds, as claimed.
\end{proof}


\section{Quantitative convergence to the target}\label{sec:target}

\subsection{Target-entropy decay rates}\label{sec:target-KSD}

We now present entropy decay estimates in the mean-field regime.
The following proposition and corollary were also established in \cite[Propositions 3 and 4]{priser2024} under slightly different assumptions. We include the argument for completeness.

In the following, all applications of integration by parts on the Euclidean space are understood after
multiplication by standard cutoffs \(\chi_R\uparrow1\).  The coefficient
growth bounds, finite-time moment estimates, and the relevant
entropy--Fisher-information integrability imply that all cutoff remainders
vanish as \(R\to\infty\); see proof of Lemma~\ref{lem:finite-dimensional-entropy}.

\subsubsection*{Mean-field target entropy}

\begin{proposition}[Mean-field Stein--Langevin dissipation]
\label{prop:mf-entropy}
Under Assumption~\ref{ass:core-target} and either kernel assumption, the map
\(t\mapsto\KL(\mu_t\|\pi)\) is locally absolutely continuous and, for
\(0\le s\le t\),
\begin{equation}\label{eq:mf-entropy-integrated}
 \KL(\mu_t\|\pi)-\KL(\mu_s\|\pi)
 =-\int_s^t\KSD^2(\mu_u\|\pi)\,du
  -\varepsilon\int_s^t I(\mu_u\|\pi)\,du.
\end{equation}
Consequently,
\begin{equation}\label{eq:mf-KL-decay}
 \KL(\mu_t\|\pi)
 \le e^{-2\varepsilon\alpha t}\KL(\mu_0\|\pi).
\end{equation}
\end{proposition}

\begin{proof}
Apply Lemma~\ref{lem:finite-dimensional-entropy} with \(N=1\),
\(P_t=\mu_t\), and \(Q_t=\pi\).  Here \(Q_0=\pi\), \(Q_t=\pi\) for every
\(t\), and its drift is \(-\varepsilon\nabla V\).  The drift difference is
therefore \(v_{\mu_t}\).  The hypotheses of the lemma follow from
Lemma~\ref{lem:nonlinear-drift-regularity}, and hence
\begin{align}
 \KL(\mu_t\|\pi)-\KL(\mu_s\|\pi)
 &=\int_s^t\int
   v_{\mu_u}\cdot
   \nabla\log\frac{d\mu_u}{d\pi}\,d\mu_u\,du
 -\varepsilon\int_s^t I(\mu_u\|\pi)\,du.
 \label{eq:mf-chain-rule-before-stein}
\end{align}
For the first term, integration by parts gives
\[
 \int v_{\mu_u}\cdot
   \nabla\log\frac{d\mu_u}{d\pi}\,d\mu_u
 =\int\{\nabla V\cdot v_{\mu_t}-\nabla\cdot v_{\mu_t}\}\,d\mu_t
 =-\int\mathcal T_\pi v_{\mu_u}\,d\mu_u
 =-\|S_\pi(\mu_u)\|_{\mathcal H^d}^2.
\]
Substitution into \eqref{eq:mf-chain-rule-before-stein} proves
\eqref{eq:mf-entropy-integrated}. The target LSI gives, for
almost every \(t\),
\[
 \frac{d}{dt}\KL(\mu_t\|\pi)
 \le-2\varepsilon\alpha\KL(\mu_t\|\pi).
\]
Gr\"onwall's lemma now proves
\eqref{eq:mf-KL-decay} for every \(t\ge0\).
\end{proof}

\begin{corollary}[Pointwise mean-field Wasserstein convergence]
\label{cor:mf-W2-target}
Under Assumption~\ref{ass:core-target} and either kernel assumption,
\begin{equation}\label{eq:mf-W2-target}
    W_2(\mu_t,\pi)
    \le \left(\frac{2}{\alpha}\KL(\mu_0\|\pi)\right)^{1/2}
       e^{-\varepsilon\alpha t},
\end{equation}
\end{corollary}

\begin{proof}
The LSI \eqref{eq:target-lsi} implies $T_2(1/\alpha)$, namely
$W_2^2(\nu,\pi)\le(2/\alpha)\KL(\nu\|\pi)$.  Combining this with
\eqref{eq:mf-KL-decay} proves the claim.
\end{proof}

\begin{corollary}[Pointwise mean-field KSD convergence]\label{cor:mf-KSD-target}
Under Assumptions~\ref{ass:core-target} and \ref{ass:ksd-kernel},
\begin{equation}\label{eq:mf-KSD-target}
    \KSD(\mu_t\|\pi)
    \le C e^{-\varepsilon\alpha t}.
\end{equation}
\end{corollary}

\begin{proof}
Let $(X,Y)$ be an optimal coupling of $\mu_t$ and $\pi$.  Since
$S_\pi(\pi)=0$, Lemma~\ref{lem:witness-lipschitz} and Cauchy--Schwarz give
\[
    \KSD(\mu_t\|\pi)
    \le C\left(\int(1+\|y\|)^2\pi(dy)\right)^{1/2}
        W_2(\mu_t,\pi).
\]
Assumption~\ref{ass:core-target} gives $m_2(\pi)<\infty$, and
Corollary~\ref{cor:mf-W2-target} proves
\eqref{eq:mf-KSD-target}.
\end{proof}

\subsubsection*{Finite-particle target entropy}

Let $p_t^N$ be the density of $P_t^N$, and define
\[
    H_N(t):=\KL(P_t^N\|\Pi_N).
\]

\begin{lemma}[Empirical Stein identity with diagonal correction]\label{lem:configuration-stein}
For $x,y\in\R^d$, define the Stein kernel
\[
    \mathcal K_\pi(x,y)
    :=k(x,y)\nabla V(x)\cdot\nabla V(y)
      -\nabla_1k(x,y)\cdot\nabla V(y)
      -\nabla_2k(x,y)\cdot\nabla V(x)
      +\operatorname{tr}\nabla_{12}^2k(x,y).
\]
Then, for $X=(x_1,\ldots,x_N) \in (\R^d)^N$,
\[
    \KSD^2(\mu_X^N\|\pi)
    =\frac1{N^2}\sum_{i,j=1}^N\mathcal K_\pi(x_i,x_j).
\]
Moreover, if
\[
    F_i^N(X):=\frac1N\sum_{j=1}^NB(x_i,x_j),
\]
then, 
\begin{equation}\label{eq:configuration-stein}
    \sum_{i=1}^N
    \{-\nabla_i\cdot F_i^N(X)+\nabla V(x_i)\cdot F_i^N(X)\}
    =-N\KSD^2(\mu_X^N\|\pi)
      +\frac1N\sum_{i=1}^NC^\star(x_i).
\end{equation}
\end{lemma}

\begin{proof}
The derivative-reproducing identities give
\[
    \langle\xi_\pi(x),\xi_\pi(y)\rangle_{\mathcal H^d}
    =\mathcal K_\pi(x,y).
\]
Consequently,
\[
    \KSD^2(\mu_X^N\|\pi)
    =\left\|\frac1N\sum_{i=1}^N\xi_\pi(x_i)\right\|_{\mathcal H^d}^2
    =\frac1{N^2}\sum_{i,j=1}^N\mathcal K_\pi(x_i,x_j),
\]
which proves the first assertion.
If $i\ne j$, differentiation with respect to $x_i$ acts only on the first argument of $B$, and direct expansion gives
\[
    -\nabla_i\cdot B(x_i,x_j)
    +\nabla V(x_i)\cdot B(x_i,x_j)
    =-\mathcal K_\pi(x_i,x_j).
\]
If $i=j$, the map $x_i\mapsto B(x_i,x_i)$ depends on both arguments.  The first-argument derivative gives the same term $-\mathcal K_\pi(x_i,x_i)$, while the additional second-argument derivative is
\[
    -\nabla_2\cdot B(x_i,x_i)
    =\nabla_2k(x_i,x_i)\cdot\nabla V(x_i)
      +k(x_i,x_i)\Delta V(x_i)-\Delta_2k(x_i,x_i)
    =C^\star(x_i).
\]
Summing over $i,j$ and dividing by $N$ proves \eqref{eq:configuration-stein}.
\end{proof}

The following theorem establishes entropy dissipation at the particle level. This gives long-time target-convergence rates and is also a crucial estimate for our uniform-in-time PoC results. For deterministic SVGD, such an approach was taken in \cite{banerjee2024,balasubramanian2026}. In several entropy arguments below, we use the general entropy evolution equation derived in Appendix~\ref{app:entropy-chain-rule}.

\begin{theorem}[Full-law target entropy]\label{prop:full-law}
Under Assumption~\ref{ass:core-target} and either Assumption~\ref{ass:ksd-kernel} or Assumption~\ref{ass:W2},
the map \(t\mapsto H_N(t)\) is locally absolutely continuous and, for
\(0\le s\le t\),
\begin{align}
    H_N(t)-H_N(s)
    &=-N\int_s^t\E\KSD^2(\mu_u^N\|\pi)\,du
     -\varepsilon\int_s^t I(P_u^N\|\Pi_N)\,du
    \notag\\
    &\quad
     +\frac1N\int_s^t
       \E\sum_{i=1}^NC^\star(X_i^N(u))\,du.
    \label{eq:full-law-integrated}
\end{align}
Consequently,
\begin{equation}\label{eq:full-law-bound}
    H_N(t)
    \le e^{-2\varepsilon\alpha t}H_N(0)
      +\frac{C^\star}{2\varepsilon\alpha}
       \{1-e^{-2\varepsilon\alpha t}\}.
\end{equation}
For $P_0^N=\mu_0^{\otimes N}$,
\begin{equation}\label{eq:entropy-per-particle}
    \frac{H_N(t)}N
    \le e^{-2\varepsilon\alpha t}\KL(\mu_0\|\pi)
      +\frac{C^\star}{2\varepsilon\alpha N}.
\end{equation}
\end{theorem}

\begin{proof}
Let \(Q^N\) be the law of \(N\) independent Langevin diffusions with drift
\(-\varepsilon\nabla V\) and diffusion coefficient
\(\sqrt{2\varepsilon}I\), initialized by
\(
    Q_0^N=\Pi_N.
\)
Since \(\Pi_N\) is invariant for this dynamics,
\(
    Q_t^N=\Pi_N, t\ge0.
\)
Apply Lemma~\ref{lem:finite-dimensional-entropy} to \(P_t^N\) and
\(Q_t^N\).  Their drift difference is
\(F^N=(F_1^N,\ldots,F_N^N)\), where we recall
\(
    F_i^N(X)=\frac1N\sum_{j=1}^NB(x_i,x_j).
\)
The lemma gives
\begin{align}
 H_N(t)-H_N(s)
 =\int_s^t\int
   \sum_{i=1}^NF_i^N\cdot
   \nabla_i\log\frac{p_u^N}{\Pi_N}\,p_u^N\,dX\,du
 -\varepsilon\int_s^t I(P_u^N\|\Pi_N)\,du.
 \label{eq:full-law-before-stein}
\end{align}
Applying integration by parts, for almost every
\(u>0\),
\[
 \int\sum_{i=1}^NF_i^N\cdot
   \nabla_i\log\frac{p_u^N}{\Pi_N}\,p_u^N\,dX
 =-N\E\KSD^2(\mu_u^N\|\pi)
  +\frac1N\E\sum_{i=1}^NC^\star(X_i^N(u)).
\]
Substitution into \eqref{eq:full-law-before-stein} proves
\eqref{eq:full-law-integrated}.

Tensorization of the target LSI gives
\(I(P_u^N\|\Pi_N)\ge2\alpha H_N(u)\).  Thus, for almost every \(u>0\),
\[
 H_N'(u)\le-2\varepsilon\alpha H_N(u)+C^\star.
\]
This gives
\eqref{eq:full-law-bound}.  Product initialization gives
\(H_N(0)=N\KL(\mu_0\|\pi)\), proving
\eqref{eq:entropy-per-particle}.
\end{proof}

\begin{remark}[Diagonal correction Assumption~\ref{ass:ksd-kernel} in Theorem~\ref{prop:full-law}]\label{rem:Cstar}
The condition \eqref{eq:Cstar-bound} is the same diagonal-correction hypothesis used in Assumption~2.1(iv) of \cite{balasubramanian2026}. This condition was required to obtain the particle rates for deterministic SVGD in \cite{banerjee2024}, and it is not implied by boundedness of an arbitrary nonstationary kernel.  For a symmetric translation-invariant kernel $k(x,y)=\Psi(x-y)$,
\[
    C^\star(x)=\Psi(0)\Delta V(x)-\Delta\Psi(0),
\]
so boundedness of the Hessian of $V$ is sufficient.  For the Gaussian RBF
$k(x,y)=\exp(-\|x-y\|^2/(2h^2))$,
\[
    C^\star(x)=\Delta V(x)+\frac d{h^2}.
\]
More generally, symmetry implies
\[
    C^\star(x)
    =\frac12\{\nabla_1k(x,x)+\nabla_2k(x,x)\}\cdot\nabla V(x)
      +k(x,x)\Delta V(x)-\Delta_2k(x,x).
\]
Thus only kernel values and derivatives evaluated on the diagonal
\(\{(x,x):x\in\mathbb R^d\}\) enter the correction.  If the pointwise bound is unavailable, it can be replaced in Theorem~\ref{prop:full-law} by
\[
    \log\int e^{\lambda(C^\star(x))_+}\pi(dx)<\infty
\]
for some $\lambda>0$.  The entropy variational inequality at scale $\lambda N$ and the product structure of $\Pi_N$ give
\begin{align*}
\E_{P_t^N}\frac1N\sum_{i=1}^N(C^\star(X_i^N(t)))_+
    & \le \frac{H_N(t)}{\lambda N}
      +\frac1{\lambda N}
       \log\int
       \exp\left\{\lambda\sum_{i=1}^N(C^\star(x_i))_+\right\}
       d\Pi_N(X)
    \\
    & =\frac{H_N(t)}{\lambda N}
      +\frac1\lambda\log\int e^{\lambda(C^\star(x))_+}\pi(dx).
\end{align*}
Using this estimate in
\eqref{eq:full-law-integrated} yields, with $N$ large enough that
$2\varepsilon\alpha-(\lambda N)^{-1}>0$,
\[
    H_N(t)
    \le e^{-(2\varepsilon\alpha-(\lambda N)^{-1})t}H_N(0)
      +\frac{\log\int e^{\lambda(C^\star(x))_+}\pi(dx)}
             {\lambda(2\varepsilon\alpha-(\lambda N)^{-1})}.
\]
This extension enlarges only the target-entropy part of the argument; it does not remove the uniform growth condition needed for finite-time propagation of chaos.
\end{remark}

\subsection{Results in Langevin KSD}

The integrals appearing below are finite by Assumption~\ref{ass:core-target} and Lemma~\ref{lem:witness-lipschitz}. The following result uses the above entropy dissipation to obtain rates of convergence to the target in KSD.

\begin{proposition}[Pointwise finite-particle convergence to the target in KSD]\label{prop:particle-KSD-target}
Under Assumptions~\ref{ass:core-target} and \ref{ass:ksd-kernel},
\begin{align}
    \E\KSD(\mu_t^N\|\pi)
    &\le C\left(\int(1+\|y\|)^2\pi(dy)\right)^{1/2}
       \left(\frac{2H_N(t)}{\alpha N}\right)^{1/2}
       +\frac1{\sqrt N}
        \left(\int\|\xi_\pi(y)\|_{\mathcal H^d}^2\pi(dy)\right)^{1/2}
       \label{eq:particle-KSD-entropy}\\
    &\le C e^{-\varepsilon\alpha t}+\frac{C}{\sqrt N}.
       \label{eq:particle-KSD-target}
\end{align}
\end{proposition}

\begin{proof}
Let $(X,Y)$ be an optimal coupling of $P_t^N$ and $\Pi_N$.  The coordinates $Y_1,\ldots,Y_N$ are i.i.d. with law $\pi$.  By the norm representation of the witness map,
\begin{align*}
    \KSD(\mu_X^N\|\pi)
    &\le
    \left\|\frac1N\sum_{i=1}^N
      \{\xi_\pi(X_i)-\xi_\pi(Y_i)\}\right\|_{\mathcal H^d}
    +\left\|\frac1N\sum_{i=1}^N\xi_\pi(Y_i)\right\|_{\mathcal H^d}.
\end{align*}
For the first term, Lemma~\ref{lem:witness-lipschitz} and Cauchy--Schwarz give
\begin{align*}
    \E\frac1N\sum_{i=1}^N
       \|\xi_\pi(X_i)-\xi_\pi(Y_i)\|_{\mathcal H^d}
    &\le C\left(\int(1+\|y\|)^2\pi(dy)\right)^{1/2}
      \left(\frac1N\E\sum_{i=1}^N\|X_i-Y_i\|^2\right)^{1/2}\\
    &\le C\left(\int(1+\|y\|)^2\pi(dy)\right)^{1/2}
      \left(\frac{2H_N(t)}{\alpha N}\right)^{1/2},
\end{align*}
where the last step uses the tensorized $T_2(1/\alpha)$ inequality for $\Pi_N$.
For the second term, Stein's identity gives $\E\xi_\pi(Y_i)=0$, and independence yields
\[
    \E\left\|\frac1N\sum_{i=1}^N\xi_\pi(Y_i)\right\|_{\mathcal H^d}
    \le\left(\frac1{N^2}\sum_{i=1}^N\E\|\xi_\pi(Y_i)\|_{\mathcal H^d}^2\right)^{1/2}
    =\frac1{\sqrt N}
      \left(\int\|\xi_\pi(y)\|_{\mathcal H^d}^2\pi(dy)\right)^{1/2}.
\]
This proves \eqref{eq:particle-KSD-entropy}.  Theorem~\ref{prop:full-law} gives \eqref{eq:particle-KSD-target}.
\end{proof}

\begin{remark}[Why time averaging is not needed]\label{rem:no-averaging}
Proposition~\ref{prop:particle-KSD-target} and Corollary~\ref{cor:mf-KSD-target} play the roles of Lemmas~2.3 and 2.4 in \cite{balasubramanian2026}, but they are pointwise in physical time.  In deterministic SVGD the entropy identity controls only the time integral of $\KSD^2$, which leads naturally to time-averaged measures.  Here the additional term $-\varepsilon I(\mu_t\|\pi)$, together with the target LSI, gives exponential decay of the last iterate.  At finite $N$, the full-law entropy estimate and the tensorized $T_2$ inequality transfer this pointwise decay to the empirical KSD up to the $N^{-1/2}$ sampling floor.
\end{remark}

\subsection{Results in \texorpdfstring{$W_2$}{W2} distance}\label{sec:target-W2}

The target-convergence part of the $W_2$ argument requires no KSD-to-$W_2$ comparison.  It follows directly from the target LSI, full-law entropy, and the Lipschitz property of the empirical-measure map.

\begin{lemma}[Empirical-measure map]\label{lem:empirical-map}
For $X,Y\in(\R^d)^N$,
\begin{equation}\label{eq:empirical-map}
    W_2^2(\mu_X^N,\mu_Y^N)
    \le\frac1N\sum_{i=1}^N\|x_i-y_i\|^2.
\end{equation}
\end{lemma}

\begin{proof}
The measure $N^{-1}\sum_i\delta_{(x_i,y_i)}$ is an admissible coupling of the two empirical measures.
\end{proof}
Let $Y_1,\ldots,Y_N$ be i.i.d. with law $\pi$.

The next result furnishes rates of convergence of the empirical measure process to the target in $W_2$ distance.

\begin{proposition}[Pointwise empirical convergence to the target in $W_2$]\label{prop:particle-W2-target}
Under Assumptions~\ref{ass:core-target} and either kernel assumption in Assumption~\ref{ass:ksd-kernel} or Assumption~\ref{ass:W2},
\begin{align}
    \E W_2^2(\mu_t^N,\pi)
    \le{}&
    \frac4\alpha e^{-2\varepsilon\alpha t}\KL(\mu_0\|\pi)
    +\frac{2C^\star}{\varepsilon\alpha^2N}
      \{1-e^{-2\varepsilon\alpha t}\}
    +2\E W_2^2\left(\frac1N\sum_{i=1}^N\delta_{Y_i},\pi\right).
    \label{eq:particle-W2-target}
\end{align}
\end{proposition}

\begin{proof}
Let $(X,Y)$ be an optimal coupling of $P_t^N$ and $\Pi_N$.  Lemma~\ref{lem:empirical-map} gives
\[
    \E W_2^2(\mu_X^N,\mu_Y^N)
    \le\frac1N W_2^2(P_t^N,\Pi_N).
\]
Since log-Sobolev inequalities tensorize, \(\Pi_N=\pi^{\otimes N}\)
satisfies the same LSI with constant \(\alpha\).  By the Otto--Villani
theorem \cite{ottoVillani2000}, this implies the Talagrand $T_2$ inequality
\[
    W_2^2(P_t^N,\Pi_N)
    \le \frac{2}{\alpha}\KL(P_t^N\|\Pi_N)
    =\frac{2}{\alpha}H_N(t).
\]
Using the squared triangle inequality through the empirical measure $\mu_Y^N$,
\[
    \E W_2^2(\mu_t^N,\pi)
    \le\frac4{\alpha N}H_N(t)
      +2\E W_2^2\left(\frac1N\sum_{i=1}^N\delta_{Y_i},\pi\right).
\]
Theorem~\ref{prop:full-law} completes the proof.
\end{proof}

\begin{proposition}[Uniform second moments]\label{prop:second-moments}
Under the assumptions of Proposition~\ref{prop:particle-W2-target},
\begin{equation}\label{eq:uniform-second-moments}
    \sup_{t\ge0}m_2(\mu_t)
    +\sup_{N\ge1}\sup_{t\ge0}\E m_2(\mu_t^N)<\infty.
\end{equation}
Consequently,
\[
    \int_0^T\{1+m_2(\mu_t)+\E m_2(\mu_t^N)\}\,dt
    \le C(1+T).
\]
\end{proposition}

\begin{proof}
By Assumption~\ref{ass:core-target}, the target has a finite second moment, $m_2(\pi)<\infty$.

We first control the mean-field second moment. For every
$\nu\in\mathcal P_2(\R^d)$, let $\gamma$ be an optimal coupling of
$\nu$ and $\pi$. Minkowski's inequality in $L^2(\gamma)$ gives
\begin{align}
    m_2(\nu)^{1/2}
    &=
    \left(
        \int_{\R^d\times\R^d}\|x\|^2\,\gamma(dx,dy)
    \right)^{1/2}
    \notag\\
    &\le
    \left(
        \int_{\R^d\times\R^d}\|x-y\|^2\,\gamma(dx,dy)
    \right)^{1/2}
    +
    \left(
        \int_{\R^d\times\R^d}\|y\|^2\,\gamma(dx,dy)
    \right)^{1/2}
    \notag\\
    &=
    W_2(\nu,\pi)+m_2(\pi)^{1/2}.
    \label{eq:moment-from-W2}
\end{align}
Equivalently,
\begin{equation}\label{eq:moment-from-W2-squared}
    m_2(\nu)
    \le
    2W_2^2(\nu,\pi)+2m_2(\pi).
\end{equation}
Applying this to $\nu=\mu_t$ and using
\eqref{eq:mf-W2-target}, we obtain
\begin{align}
    m_2(\mu_t)
    &\le
    2W_2^2(\mu_t,\pi)+2m_2(\pi)
    \notag\\
    &\le
    \frac{4}{\alpha}
    e^{-2\varepsilon\alpha t}
    \KL(\mu_0\|\pi)
    +2m_2(\pi).
    \label{eq:mf-second-moment-bound}
\end{align}
Consequently,
\[
    \sup_{t\ge0}m_2(\mu_t)
    \le
    \frac{4}{\alpha}\KL(\mu_0\|\pi)
    +2m_2(\pi)
    <\infty.
\]
We next control the second moment of the particle. Applying \eqref{eq:moment-from-W2-squared} with $\nu=\mu_t^N$ gives
\[
    m_2(\mu_t^N)
    \le
    2W_2^2(\mu_t^N,\pi)+2m_2(\pi).
\]
Taking expectations and invoking
Proposition~\ref{prop:particle-W2-target}, we obtain
\begin{align}
    \E m_2(\mu_t^N)
    \le{}&
    \frac{8}{\alpha}
    e^{-2\varepsilon\alpha t}
    \KL(\mu_0\|\pi)
    \notag\\
    &+
    \frac{4C^\star}{\varepsilon\alpha^2N}
    \left(1-e^{-2\varepsilon\alpha t}\right)
    +4\E W_2^2\left(\frac1N\sum_{i=1}^N\delta_{Y_i},\pi\right)
    +2m_2(\pi).
    \label{eq:particle-second-moment-bound}
\end{align}
We next check that the empirical sampling term is uniformly
bounded in $N$. Let
\[
    \widehat\pi_N:=\frac1N\sum_{i=1}^N\delta_{Y_i},
    \qquad
    Y_1,\ldots,Y_N\stackrel{\mathrm{i.i.d.}}{\sim}\pi.
\]
For every realization, the product measure
$\widehat\pi_N\otimes\pi$ is an admissible coupling of
$\widehat\pi_N$ and $\pi$. Hence
\begin{align*}
    W_2^2(\widehat\pi_N,\pi)
    &\le
    \int\|x-y\|^2\,
    \widehat\pi_N(dx)\pi(dy)\\
    &\le
    2m_2(\widehat\pi_N)+2m_2(\pi).
\end{align*}
Taking expectations and using
\[
    \E m_2(\widehat\pi_N)=m_2(\pi)
\]
gives
\begin{equation}\label{eq:iid-pi-uniform-bound}
    \E W_2^2(\widehat\pi_N,\pi)
    \le4m_2(\pi),
    \qquad N\ge1.
\end{equation}
Combining \eqref{eq:particle-second-moment-bound} and
\eqref{eq:iid-pi-uniform-bound}, and using $N^{-1}\le1$, yields
\[
    \sup_{N\ge1}\sup_{t\ge0}
    \E m_2(\mu_t^N)
    \le
    \frac{8}{\alpha}\KL(\mu_0\|\pi)
    +\frac{4C^\star}{\varepsilon\alpha^2}
    +18m_2(\pi)
    <\infty.
\]
Together with \eqref{eq:mf-second-moment-bound}, this proves
\eqref{eq:uniform-second-moments}. Finally, for every $T\ge0$,
\[
    \int_0^T
    \left\{
        1+m_2(\mu_t)+\E m_2(\mu_t^N)
    \right\}\,dt
    \le
    \left\{1+\sup_{t\ge0}m_2(\mu_t)
      +\sup_{N\ge1}\sup_{t\ge0}\E m_2(\mu_t^N)\right\}T
    \le C(1+T).
\]
This proves the final claim.
\end{proof}

We note here that uniform fourth moment control was established for a discrete-time version of our process in \cite[Lemma 1]{priser2024}.

\begin{remark}[Relation to $W_2$ estimates for deterministic SVGD]\label{rem:W2-BBK}
Proposition~\ref{prop:particle-W2-target} is the pointwise stochastic counterpart of the long-time target estimate used in Lemma~4.5 of \cite{balasubramanian2026}. The proof of Proposition~\ref{prop:particle-W2-target} uses full-law target-entropy estimates and the target $T_2(1/\alpha)$ inequality implied by the target LSI, in contrast with a direct KSD-to-$W_2$ comparison made in \cite{balasubramanian2026}. Similarly, under stronger assumptions, Proposition~\ref{prop:second-moments} yields a stronger pointwise second moment bound in comparison to the time-integrated moment estimate in \cite[Lemma~4.2]{balasubramanian2026}.
\end{remark}

\section{Finite-time propagation of chaos}\label{sec:finite-time}

\subsection{Exponential moment estimates under target and mean-field laws}\label{sec:concentration-regime}

The results of Section~\ref{sec:target} provide quantitative convergence to the target and uniform second-moment bounds under Assumptions~\ref{ass:core-target} together with either Assumption~\ref{ass:ksd-kernel} or Assumption~\ref{ass:W2}. To obtain the strong finite-time propagation of chaos estimates in expectation, we now impose Assumption~\ref{ass:concentration} and first derive the corresponding uniform concentration estimates.

\begin{lemma}[Gaussian tails under the strong-concentration regime]\label{lem:target-tail}
Under Assumption~\ref{ass:concentration}, there is $a_\pi>0$ such that
\[
    \int e^{a_\pi\|x\|^2}\pi(dx)<\infty.
\]
In particular, $\pi$ has moments of every polynomial order.
\end{lemma}

\begin{proof}
Write $x=r\theta$ with $r\ge1$ and $\|\theta\|=1$.  Then
\[
    \frac{d}{dr}V(r\theta)
    =\theta\cdot\nabla V(r\theta)
    \ge mr-\frac br.
\]
Integrating from $1$ to $r$, using continuity on the unit sphere and absorbing $b\log r$ into a smaller quadratic term, gives
\[
    V(x)\ge \frac m4\|x\|^2-C
\]
for large $\|x\|$.  The claim follows from $\pi(dx)\propto e^{-V(x)}dx$.
\end{proof}

\begin{proposition}[Uniform sub-Gaussian moments of the nonlinear law]\label{prop:subgaussian}
Under Assumptions~\ref{ass:core-target} and \ref{ass:concentration}, together with either kernel assumption, there are $a>0$ and $C<\infty$ such that
\begin{equation}\label{eq:uniform-subgaussian}
    \sup_{t\ge0}\int e^{a\|x\|^2}\mu_t(dx)\le C.
\end{equation}
\end{proposition}

\begin{proof}
Proposition~\ref{prop:second-moments} gives a uniform first moment bound in time.  Since
$\|\nabla V(y)\|\le\|\nabla V(0)\|+L_V\|y\|$ and $k,\nabla_2k$ are bounded,
\begin{equation}\label{eq:velocity-bounded}
    M_v := \sup_{t\ge0}\sup_{x\in\R^d}\|v_{\mu_t}(x)\|
    \le C\left(1+\sup_{t\ge0}\int\|y\|\mu_t(dy)\right)
    <\infty.
\end{equation}
The generator of \eqref{eq:nonlinear-sde} satisfies, for $a>0$,
\[
    \frac{\mathcal L_t e^{a\|x\|^2}}{e^{a\|x\|^2}}
    =2a x\cdot v_{\mu_t}(x)
      -2a\varepsilon x\cdot\nabla V(x)
      +2a\varepsilon d+4a^2\varepsilon\|x\|^2.
\]
Using \eqref{eq:velocity-bounded}, Young's inequality, and Assumption~\ref{ass:concentration}, we thus have for $t \ge 0$,
\begin{align*}
    \frac{\mathcal L_t e^{a\|x\|^2}}{e^{a\|x\|^2}}
    &\le
    2aM_v\|x\|
    -2a\varepsilon m\|x\|^2
    +2a\varepsilon(b+d)
    +4a^2\varepsilon\|x\|^2 \\
    &\le
    \frac{aM_v^2}{\varepsilon m}
    +2a\varepsilon(b+d)
    -a\varepsilon(m-4a)\|x\|^2.
\end{align*}
Consequently,
\begin{equation}\label{eq:path-generator}
    \mathcal L_t e^{a\|x\|^2}
    \le
    \bigl(C_{\mathrm{gen}}-c_0\|x\|^2\bigr)e^{a\|x\|^2},
    \qquad t\ge0,
\end{equation}
where \(c_0=a\varepsilon(m-4a)>0\) for $a$ sufficiently small
and one may take
\[
    C_{\mathrm{gen}}
    :=
    \frac{aM_v^2}{\varepsilon m}
    +2a\varepsilon(b+d).
\]
Choose $0<a<\min\{a_0,m/8\}$ and $\delta>0$. For a sufficiently large $R$, $C_{\mathrm{gen}}-c_0\|x\|^2 \le -\delta$ on $\{\|x\|\ge R\}$.  On the ball $\{\|x\|\le R\}$, $\mathcal L_t e^{a\|x\|^2}+\delta e^{a\|x\|^2}$ is uniformly bounded. Hence
\[
    \mathcal L_t e^{a\|x\|^2}\le C-\delta e^{a\|x\|^2}
\]
with finite constant $C$ independent of $t$.  Applying It\^o's formula to the stopped process at
the exit time
\(
  \tau_n:=\inf\{s\ge0:\|\bar X(s)\|\ge n\}
\)
gives
\[
 \E\!\left[e^{\delta(t\wedge\tau_n)}
   e^{a\|\bar X(t\wedge\tau_n)\|^2}\right]
 \le \E e^{a\|\bar X(0)\|^2}
      +\frac C\delta(e^{\delta t}-1).
\]
Since $a<a_0$, Fatou's lemma as \(n\to\infty\) thus yields \eqref{eq:uniform-subgaussian}.
\end{proof}

\begin{corollary}[Empirical $W_2$ floors]\label{cor:FG}
Let $\rho$ be a probability measure on $\R^d$ with a finite $q$th
moment for some $q>4$, and let $Y_1,\ldots,Y_N$ be i.i.d.\ with law
$\rho$. Then
\[
    \E W_2^2\left(\frac1N\sum_{i=1}^N\delta_{Y_i},\rho\right)
    \le C_\rho
    \begin{cases}
      N^{-1/2}, & d<4,\\
      N^{-1/2}\log(1+N), & d=4,\\
      N^{-2/d}, & d>4.
    \end{cases}
\]
Under Assumption~\ref{ass:concentration}, this holds for $\rho = \pi$. Moreover, under the assumptions made in Proposition~\ref{prop:subgaussian}, the constant can be
chosen uniformly in $t$ such that the above holds for $\rho=\mu_t$, $t\ge0$.
\end{corollary}

\begin{proof}
The first assertion is the Fournier--Guillin estimate
\cite{fournier2015}. Lemma~\ref{lem:target-tail} and
Proposition~\ref{prop:subgaussian} provide, respectively, the required
moment bound for $\pi$ and a uniform-in-$t$ moment bound for $\mu_t$.
\end{proof}

\subsection{Stein-velocity fluctuations}
Proposition~\ref{prop:subgaussian} gives the state-space exponential moment
for the bounded-kernel model.  We next convert it into the exponential
law of large numbers for the Stein-velocity defect:
\begin{equation}\label{eq:GN}
    G_t^N(x_1,\ldots,x_N)
    :=\sum_{i=1}^N
      \left\|\frac1N\sum_{j=1}^NB(x_i,x_j)-v_{\mu_t}(x_i)\right\|^2.
\end{equation}

\begin{lemma}[Sub-Gaussian empirical averages]\label{lem:subgaussian-average}
Let $Z_1,\ldots,Z_n$ be centered i.i.d. vectors in $\R^d$.  If
$\E e^{a\|Z_1\|^2}\le M$ for some $a,M \in (0,\infty)$, then there are $a',M' \in (0,\infty)$, depending only on $a,M,d$, such that
\begin{equation}\label{eq:subgaussian-average}
    \E\exp\left(a'n\left\|\frac1n\sum_{j=1}^nZ_j\right\|^2\right)
    \le M'.
\end{equation}
\end{lemma}

\begin{proof}
The exponential-square moment implies
$(\E|e\cdot Z_1|^p)^{1/p}\le C\sqrt p$, for some finite positive constant $C$, uniformly over unit vectors $e$ and $p\ge2$.  Expanding the exponential series and using Stirling's formula gives
\[
    \E e^{\theta\cdot Z_1}\le e^{C\|\theta\|^2},
    \qquad \theta\in\R^d,
\]
possibly after increasing $C$. Thus $n^{-1/2}\sum_jZ_j$ satisfies the same bound.  The Gaussian integral identity gives, for any $c \ge 0$,
\[
    e^{c\|z\|^2}
    =(2\pi)^{-d/2}\int_{\R^d}
       e^{\sqrt{2c}\,g\cdot z-\|g\|^2/2}\,dg.
\]
Consequently,
\[
    \E\exp\left(c\left\|\frac1{\sqrt n}\sum_{j=1}^nZ_j\right\|^2\right)
    \le (2\pi)^{-d/2}\int_{\R^d}
       e^{-(1/2-2Cc)\|g\|^2}\,dg,
\]
which is finite for sufficiently small $c>0$, uniformly in $n$.
\end{proof}

\begin{proposition}[Exponential law of large numbers for the Stein velocity]\label{prop:exp-lln}
Under Assumptions~\ref{ass:core-target} and \ref{ass:concentration}, together with either kernel assumption, there are $\lambda>0$ and $1\le C_G<\infty$, independent of $N$ and $t$, such that
\begin{equation}\label{eq:exp-lln}
    \int e^{\lambda G_t^N}\,d\mu_t^{\otimes N}\le C_G.
\end{equation}
\end{proposition}

\begin{proof}
The bounded-kernel assumptions and \eqref{eq:hessian-bound} imply
\begin{equation}\label{eq:B-linear}
    \|B(x,y)\|\le K(1+\|y\|)
\end{equation}
with $K$ independent of $x$.  Fix $t$ and let $X_1,\ldots,X_N$ be i.i.d. with law $\mu_t$.  By the definition of $G_t^N$, generalized H\"older's inequality with exponent $N$, followed by exchangeability, gives
\begin{align}
    \E e^{\lambda G_t^N}
    &=\E\prod_{i=1}^N
       \exp\left\{\lambda\left\|\frac1N\sum_{j=1}^N
       \{B(X_i,X_j)-v_{\mu_t}(X_i)\}\right\|^2\right\} \notag
    \\
    & \le\prod_{i=1}^N
       \left(\E\exp\left\{\lambda N\left\|\frac1N\sum_{j=1}^N
       \{B(X_i,X_j)-v_{\mu_t}(X_i)\}\right\|^2\right\}\right)^{1/N} \notag
    \\
    & =\E\exp\left\{\lambda N\left\|\frac1N\sum_{j=1}^N
       \{B(X_1,X_j)-v_{\mu_t}(X_1)\}\right\|^2\right\}.\label{eq:one-row}
\end{align}
For $N\ge2$, set
\[
    S_1:=\frac1{N-1}\sum_{j=2}^N
      \{B(X_1,X_j)-v_{\mu_t}(X_1)\}.
\]
Since
\[
    \frac1N\sum_{j=1}^N\{B(X_1,X_j)-v_{\mu_t}(X_1)\}
    =\frac{N-1}{N}S_1
      +\frac1N\{B(X_1,X_1)-v_{\mu_t}(X_1)\}
\]
and $(N-1)^2/N\le N-1$,
\begin{equation}\label{eq:row-split}
    N\left\|\frac1N\sum_{j=1}^N
       \{B(X_1,X_j)-v_{\mu_t}(X_1)\}\right\|^2
    \le2(N-1)\|S_1\|^2
       +\frac2N\|B(X_1,X_1)-v_{\mu_t}(X_1)\|^2.
\end{equation}
Conditionally on $X_1=x$, the summands in $S_1$ are centered and i.i.d.  By \eqref{eq:B-linear} and Proposition~\ref{prop:subgaussian}, their exponential-square moments are bounded uniformly in $x$ and $t$.  Lemma~\ref{lem:subgaussian-average} therefore yields
\[
    \sup_{N\ge2,t,x}
    \E\left[e^{c(N-1)\|S_1\|^2}\mid X_1=x\right]<\infty.
\]
Moreover, \eqref{eq:B-linear} gives $\|B(X_1,X_1)-v_{\mu_t}(X_1)\|\le C(1+\|X_1\|)$, so the diagonal term also has a uniform exponential-square moment.  From \eqref{eq:row-split} and Cauchy--Schwarz,
\begin{align*}
    \E & \exp\left\{\lambda N\left\|\frac1N\sum_{j=1}^N
       \{B(X_1,X_j)-v_{\mu_t}(X_1)\}\right\|^2\right\}
    \le\\
    & \qquad \qquad \left(\E e^{4\lambda(N-1)\|S_1\|^2}\right)^{1/2}  \left(\E\exp\left\{\frac{4\lambda}{N}
       \|B(X_1,X_1)-v_{\mu_t}(X_1)\|^2\right\}\right)^{1/2}.
\end{align*}
Both factors are uniformly bounded for small enough $\lambda$.  Combining with \eqref{eq:one-row} proves the claim; $N=1$ is the diagonal estimate.
\end{proof}

Let
\begin{equation}\label{eq:rNd}
    r_{N,d}:=
    \begin{cases}
      N^{-1/2},&d<4,\\
      N^{-1/2}\log(1+N),&d=4,\\
      N^{-2/d},&d>4.
    \end{cases}
\end{equation}
If $\bar X_1(t),\ldots,\bar X_N(t)$ are independent with law $\mu_t$,
Corollary~\ref{cor:FG} gives
\begin{equation}\label{eq:iid-W2-floor}
    \sup_{t\ge0}\E W_2^2\left(
      \frac1N\sum_{i=1}^N\delta_{\bar X_i(t)},\mu_t
    \right)\le Cr_{N,d}.
\end{equation}

\subsection{Synchronous coupling and metric estimates}
\label{app:refined-coupling}

In this subsection, we obtain quantitative finite-time PoC rates by controlling the discrepancy between the noisy SVGD particle system and a coupled system consisting of independent copies of the McKean--Vlasov process \eqref{eq:nonlinear-sde} driven by the same collection of Brownian motions.

Proposition~\ref{prop:exp-lln} gives the following uniform-in-time second moment bound as an immediate corollary.
\begin{corollary}[Uniform square-integrability of the Stein-velocity fluctuation]
\label{cor:GN-L2}
Under the assumptions of Proposition~\ref{prop:exp-lln},
\begin{equation}\label{eq:GN-L2}
  \sup_{N\ge1}\sup_{t\ge0}
  \int (G_t^N)^2\,d\mu_t^{\otimes N}<\infty.
\end{equation}
\end{corollary}

The next lemma gives quantitative estimates for the integrated second moment process of the McKean--Vlasov system.

\begin{lemma}[Exponential moment under the product path law]
\label{lem:path-exponential}
Under the assumptions of Proposition~\ref{prop:subgaussian}, let
$\bar X_1,\ldots,\bar X_N$ be independent nonlinear processes solving
\eqref{eq:nonlinear-sde}.  There are $\eta_0,C,L_0>0$, independent of
$N,s,t$, such that, for $0\le s\le t$,
\begin{equation}\label{eq:path-exponential}
  \E\exp\left\{
    \eta_0\int_s^t\frac1N\sum_{i=1}^N\|\bar X_i(u)\|^2\,du
  \right\}
  \le Ce^{L_0(t-s)}.
\end{equation}
\end{lemma}

\begin{proof}
Recall from \eqref{eq:path-generator}, for
\(F(x):=e^{a\|x\|^2}\) and \(a>0\) sufficiently small,
\begin{equation}\label{eq:path-generator2}
    \mathcal L_uF(x)
    \le
    \bigl(C_{\mathrm{gen}}-c_0\|x\|^2\bigr)F(x),
    \qquad u\ge0,
\end{equation}
where \(c_0=a\varepsilon(m-4a)>0\)
and
\(
    C_{\mathrm{gen}}
    :=
    \frac{aM_v^2}{\varepsilon m}
    +2a\varepsilon(b+d),
\)
where $b$ is as in Assumption~\ref{ass:concentration}.

Fix \(0\le s\le t\), and set
\[
    \tau_R:=\inf\{u\ge s:\|\bar X(u)\|\ge R\}.
\]
For \(u\in[s,t]\), define
\[
    Y_u^R
    :=
    \exp\left\{
        c_0\int_s^{u\wedge\tau_R}\|\bar X(r)\|^2\,dr
    \right\}
    F(\bar X(u\wedge\tau_R)).
\]
It\^o's formula and \eqref{eq:path-generator2} give for \(u\in[s,t]\),
\[
    Y_u^R
    \le
    F(\bar X(s))
    +C_{\mathrm{gen}}\int_s^u Y_v^R
       \mathbf 1_{\{v<\tau_R\}}\,dv
    +(M_u^R-M_s^R),
\]
where \(M^R\) is a martingale on \([s,t]\), since the stochastic
integrand is bounded on the stopped interval. Taking expectations and
applying Gr\"onwall's inequality yields
\[
    \E Y_t^R
    \le
    e^{C_{\mathrm{gen}}(t-s)}
    \E F(\bar X(s)).
\]
By nonexplosion, \(\tau_R\uparrow\infty\) almost surely. Fatou's lemma,
\(F\ge1\), and Proposition~\ref{prop:subgaussian} therefore give
\begin{equation}\label{eq:one-path-integrated-moment}
\begin{aligned}
    \E\exp\left\{
        c_0\int_s^t\|\bar X(u)\|^2\,du
    \right\}
    &\le
    \E\left[
        \exp\left\{
            c_0\int_s^t\|\bar X(u)\|^2\,du
        \right\}F(\bar X(t))
    \right] \\
    &\le
    e^{C_{\mathrm{gen}}(t-s)}
    \E F(\bar X(s))
    \le
    C e^{C_{\mathrm{gen}}(t-s)}.
\end{aligned}
\end{equation}
Thus, Jensen's inequality and the independence of
\(\bar X_1,\ldots,\bar X_N\) give
\begin{align*}
\E\exp\left\{
  c_0\int_s^t\frac1N\sum_{i=1}^N\|\bar X_i(u)\|^2\,du
\right\}
&=
\left[
  \E\exp\left\{
    \frac{c_0}{N}\int_s^t\|\bar X(u)\|^2\,du
  \right\}
\right]^N \\
&\le
  \E\exp\left\{
    c_0\int_s^t\|\bar X(u)\|^2\,du
  \right\}
\le Ce^{C_{\mathrm{gen}}(t-s)}.
\end{align*}
This proves \eqref{eq:path-exponential} with
\(\eta_0=c_0=a\varepsilon(m-4a)\) and 
\begin{equation}\label{eq:lgendef}
    L_0=C_{\mathrm{gen}}=
    \frac{aM_v^2}{\varepsilon m}
    +2a\varepsilon(b+d).
\end{equation}
\end{proof}

Let $\bar X_1,\ldots,\bar X_N$ be independent copies of
\eqref{eq:nonlinear-sde}.  Couple them with \eqref{eq:particle-sde} by taking
\(
  X_i^N(0)=\bar X_i(0),
\)
and by using the same driving Brownian motions $\{W_i : 1 \le i \le N\}$ in the two equations. We call this the \emph{synchronous coupling}. The following proposition provides the key quantitative control on the finite-time discrepancy between the coupled systems, that directly furnish finite-time PoC rates. 

\begin{proposition}[Controlling discrepancy under synchronous coupling]
\label{prop:refined-coupling}
Under Assumptions~\ref{ass:core-target} and \ref{ass:concentration}, together
with either kernel assumption and under the synchronous coupling just described,
define
\begin{equation}\label{eq:DN}
  D_N(t):=\frac1N\sum_{i=1}^N
    \|X_i^N(t)-\bar X_i(t)\|^2.
\end{equation}
There are constants $C<\infty$ and $L_{\mathrm{cpl}}>0$, independent of
$N,t$, such
that
\begin{equation}\label{eq:refined-coupling}
  \E D_N(t)
  \le \frac{C}{N}\bigl(e^{L_{\mathrm{cpl}}t}-1\bigr),
  \qquad t\ge0.
\end{equation}
\end{proposition}

\begin{proof}
Set $Z_i:=X_i^N-\bar X_i$ and, within this proof, write
\[
  Q_N(t):=\frac1N\sum_{j=1}^N\|\bar X_j(t)\|^2,
  \qquad
  R_i(t):=\frac1N\sum_{j=1}^NB(\bar X_i(t),\bar X_j(t))
            -v_{\mu_t}(\bar X_i(t)).
\]
Then $G_t^N(\bar X_1(t),\ldots,\bar X_N(t))=\sum_i\|R_i(t)\|^2$.

Since the Brownian terms cancel under the synchronous coupling, each
\(Z_i=X_i^N-\bar X_i\) is absolutely continuous and we get, pathwise for almost every $t$,
\begin{align}\label{eq:DN-compute}
  D_N'(t) &={}
\frac2N\sum_{i=1}^N Z_i\cdot
 \frac1N\sum_{j=1}^N
 \{B(X_i^N,X_j^N)-B(\bar X_i,\bar X_j)\} \notag\\
&-\frac{2\varepsilon}{N}\sum_{i=1}^N
  Z_i\cdot\{\nabla V(X_i^N)-\nabla V(\bar X_i)\}
+\frac2N\sum_{i=1}^NZ_i\cdot R_i.
\end{align}

By \eqref{eq:B-weighted-y-lipschitz}, with $w_j:=1+\|\bar X_j\|$,
\begin{align}\label{Dbd}
\left\|\frac1N\sum_{j=1}^N
  \{B(X_i^N,X_j^N)-B(\bar X_i,\bar X_j)\}\right\|
\le C_B\left\{
  \left(\frac1N\sum_{j=1}^Nw_j\right)\|Z_i\|
  +\frac1N\sum_{j=1}^Nw_j\|Z_j\|
\right\}.
\end{align}
Set
\[
  \overline w_N:=\frac1N\sum_{j=1}^Nw_j,
  \qquad
  M_{w,N}:=\frac1N\sum_{j=1}^Nw_j^2.
\]
Then
\[
  \overline w_N\le M_{w,N}^{1/2},\qquad
  \frac1N\sum_{j=1}^Nw_j\|Z_j\|
  \le M_{w,N}^{1/2}D_N^{1/2},\qquad
  \frac1N\sum_{i=1}^N\|Z_i\|\le D_N^{1/2}.
\]
Since $M_{w,N}\le C(1+Q_N)$, Cauchy--Schwarz in the two terms of \eqref{Dbd} gives
\begin{align*}
\frac2N\sum_{i=1}^N Z_i\cdot
 \frac1N\sum_{j=1}^N
 \{B(X_i^N,X_j^N)-B(\bar X_i,\bar X_j)\}
\le C(1+Q_N)^{1/2}D_N.
\end{align*}
The bounded Hessian of $V$ gives
\[
  -\frac{2\varepsilon}{N}\sum_{i=1}^N
  Z_i\cdot\{\nabla V(X_i^N)-\nabla V(\bar X_i)\}
  \le2\varepsilon L_VD_N,
\]
and Young's inequality gives
\[
  \frac2N\sum_{i=1}^NZ_i\cdot R_i
  \le D_N+\frac1N G_t^N(\bar X_1,\ldots,\bar X_N).
\]
Using the above estimates in \eqref{eq:DN-compute}, we obtain for every $\eta>0$, pathwise for almost every $t$,
\begin{equation}\label{eq:DN-differential}
  D_N'(t) 
  \le\{C_\eta+\eta Q_N(t)\}D_N(t)
    +\frac{C}{N}G_t^N(\bar X_1(t),\ldots,\bar X_N(t)).
\end{equation}
Notice that this bound is expressed only in terms of the i.i.d mean-field particles.

Since $D_N(0)=0$, the pathwise Gr\"onwall inequality yields
\begin{align*}
D_N(t)
\le\frac{C}{N}\int_0^t
 &\exp\left\{C_\eta(t-s)+\eta\int_s^tQ_N(u)\,du\right\}\\
&\times G_s^N(\bar X_1(s),\ldots,\bar X_N(s))\,ds.
\end{align*}
Choose $\eta>0$ so that $2\eta\le\eta_0$ in
Lemma~\ref{lem:path-exponential}.  Cauchy--Schwarz,
\eqref{eq:path-exponential}, and \eqref{eq:GN-L2} give
\begin{align*}
\E D_N(t)
\le\frac{C}{N}\int_0^t
 e^{(C_\eta+L_0/2)(t-s)}\,ds
\le\frac{C}{N}\bigl(e^{L_{\mathrm{cpl}}t}-1\bigr),
\end{align*}
where we may take \(L_{\mathrm{cpl}}=C_\eta+L_0/2\). This proves \eqref{eq:refined-coupling}.
\end{proof}

\subsection{Finite-time propagation of chaos in \texorpdfstring{$W_2$}{W2}}
\label{sec:finite-W2}

Let $L_\mathrm{cpl}$ denote the
positive exponent obtained from Proposition~\ref{prop:refined-coupling} under Assumption~\ref{ass:W2}.

\begin{corollary}[Finite-time empirical $W_2$ propagation of chaos]
\label{cor:finite-W2}
Under Assumptions~\ref{ass:core-target},
\ref{ass:W2}, and \ref{ass:concentration},
\begin{equation}\label{eq:finite-W2}
  \E W_2^2(\mu_t^N,\mu_t)
  \le C\left\{
    \frac{e^{L_\mathrm{cpl}t}-1}{N}+r_{N,d}
  \right\}
  \le Cr_{N,d}e^{L_\mathrm{cpl}t}.
\end{equation}
\end{corollary}

\begin{proof}
Let $\bar\mu_t^N=N^{-1}\sum_i\delta_{\bar X_i(t)}$ for the nonlinear copies
in Proposition~\ref{prop:refined-coupling}.  Lemma~\ref{lem:empirical-map}
gives $W_2^2(\mu_t^N,\bar\mu_t^N)\le D_N(t)$.  The squared triangle
inequality, \eqref{eq:refined-coupling}, and \eqref{eq:iid-W2-floor} give the
first bound; $N^{-1}\le C r_{N,d}$ gives the second, with the numerical
constant absorbed into the generic prefactor.
\end{proof}

\begin{corollary}[Finite-time fixed-marginal $W_2$ chaos]
\label{cor:finite-marginal-W2}
Under the same assumptions, for every $1\le k\le N$,
\begin{equation}\label{eq:finite-marginal-W2}
  W_2^2(P_{N,t}^{(k)},\mu_t^{\otimes k})
  \le C\frac{k}{N}\bigl(e^{L_\mathrm{cpl}t}-1\bigr).
\end{equation}
\end{corollary}

\begin{proof}
The first $k$ coordinate pairs in the labelled synchronous coupling form an admissible
coupling of the two marginal laws.  Exchangeability and
\eqref{eq:refined-coupling} give
\[
  W_2^2(P_{N,t}^{(k)},\mu_t^{\otimes k})
  \le\sum_{i=1}^k\E\|X_i^N(t)-\bar X_i(t)\|^2
  =k\E D_N(t),
\]
which proves the result.
\end{proof}

\subsection{Finite-time propagation of chaos in KSD}

\begin{corollary}[Finite-time KSD propagation of chaos]
\label{cor:finite-KSD-coupling}
Under Assumptions~\ref{ass:core-target},
\ref{ass:ksd-kernel}, and \ref{ass:concentration},
\begin{equation}\label{eq:finite-KSD-coupling}
  \E\KSD_\pi(\mu_t^N,\mu_t)
  \le \frac{C}{\sqrt N}e^{L_{\mathrm{cpl}}t/2}.
\end{equation}
\end{corollary}

\begin{proof}
Set $\bar\mu_t^N=N^{-1}\sum_i\delta_{\bar X_i(t)}$ for the nonlinear copies
in Proposition~\ref{prop:refined-coupling}.  The triangle inequality for the
RKHS norm gives
\begin{align*}
\KSD_\pi(\mu_t^N,\mu_t)
\le{}\frac1N\sum_{i=1}^N
 \|\xi_\pi(X_i^N(t))-\xi_\pi(\bar X_i(t))\|_{\mathcal H^d}
+\left\|\frac1N\sum_{i=1}^N
 \{\xi_\pi(\bar X_i(t))-S_\pi(\mu_t)\}\right\|_{\mathcal H^d}.
\end{align*}
By Lemma~\ref{lem:witness-lipschitz}, Cauchy--Schwarz, the uniform second
moment bound of $\mu_t$ which follows from Proposition~\ref{prop:subgaussian}, and Proposition~\ref{prop:refined-coupling}, the
expectation of the first term is at most
\[
 C\left(\E D_N(t)\right)^{1/2}
 \le\frac{C}{\sqrt N}\bigl(e^{L_{\mathrm{cpl}}t}-1\bigr)^{1/2}.
\]
The Hilbert-valued summands in the second term are independent and centered.
The witness growth bound in Lemma~\ref{lem:witness-lipschitz} and the uniform
second moment bound give
\[
\E\left\|\frac1N\sum_{i=1}^N
 \{\xi_\pi(\bar X_i(t))-S_\pi(\mu_t)\}\right\|_{\mathcal H^d}
\le\frac{C}{\sqrt N}.
\]
Combining the two estimates proves \eqref{eq:finite-KSD-coupling}.
\end{proof}

\section{Uniform-in-Time Propagation of Chaos using Cutoff}\label{sec:cutoff}

Now we combine our long-time target convergence estimates with the quantitative finite-time PoC rates to obtain uniform-in-time PoC with polynomial decay rates in $N$, by adapting the cutoff argument of \cite{balasubramanian2026}.

By a \emph{discrepancy} on \(\mathcal P_2(\R^d)\) we mean a Borel measurable
map
\[
    d:\mathcal P_2(\R^d)\times\mathcal P_2(\R^d)
      \longrightarrow[0,\infty)
\]
such that
\[
    d(\rho,\rho)=0,
    \qquad \rho\in\mathcal P_2(\R^d).
\]
We do not require \(d\) to be symmetric, to separate probability
measures, or to satisfy the triangle inequality. For the cutoff
argument, the only additional structural property needed is the
following comparison through the fixed target \(\pi\): there exists
\(C_{\mathrm{tar}}\ge1\) such that
\begin{equation}\label{eq:target-comparison}
    d(\rho,\nu)
    \le
    C_{\mathrm{tar}}
    \bigl\{d(\rho,\pi)+d(\nu,\pi)\bigr\},
    \qquad \rho,\nu\in\mathcal P_2(\R^d).
\end{equation}

\begin{proposition}[Abstract polynomial cutoff technique]\label{prop:cutoff}
Let \(d\) be a discrepancy satisfying \eqref{eq:target-comparison}.
Let $A,B,c>0$ and $L \ge 0$.  Suppose that
$a_N\in(0,1]$ and $b_N\in[0,1]$, with $a_N,b_N\to0$, and
\begin{align}
    \E d(\mu_t^N,\mu_t)&\le A\{a_Ne^{Lt}+b_N\},
      \label{eq:cutoff-short}\\
    \E d(\mu_t^N,\pi)&\le B\{e^{-ct}+b_N\},
      \label{eq:cutoff-particle-target}\\
    d(\mu_t,\pi)&\le Be^{-ct}.
      \label{eq:cutoff-mf-target}
\end{align}
Then
\begin{equation}\label{eq:cutoff-result}
    \sup_{t\ge0}\E d(\mu_t^N,\mu_t)
    \le C_{\mathrm{cut}}\{a_N^{c/(L+c)}+b_N\},
\end{equation}
where $C_{\mathrm{cut}}$ depends only on $A,B$, and
$C_{\mathrm{tar}}$.  For $d=\KSD_\pi$ or $d=W_2$, the constant in
\eqref{eq:target-comparison} can be taken to be $1$; for $d=W_2^2$, it can
be taken to be $2$.
\end{proposition}

\begin{proof}
Set
\[
    T_N:=\frac{1}{L+c}\log\frac1{a_N}.
\]
For $t\le T_N$, \eqref{eq:cutoff-short} gives
$\E d(\mu_t^N,\mu_t)\le
A\{a_N^{c/(L+c)}+b_N\}$.
For $t>T_N$, use \eqref{eq:target-comparison} and the two target estimates.  Since $e^{-cT_N}=a_N^{c/(L+c)}$, the result follows.
\end{proof}

\begin{remark}[Single versus double exponential finite-time estimates]
The synchronous coupling estimate in Proposition~\ref{prop:refined-coupling}
grows at most single-exponentially on finite horizons, so the logarithmic
cutoff produces polynomial rates in \(N\).  A double-exponential finite-time
bound, which arises in deterministic SVGD for bounded kernels, would give only a negative power of \(\log N\); see
\cite{balasubramanian2026}.
\end{remark}

\subsection{Uniform-in-time PoC in KSD}

\begin{theorem}[Uniform-in-time propagation of chaos in Langevin KSD]
\label{cor:coupling-KSD-cutoff}
Under Assumptions~\ref{ass:core-target},
\ref{ass:ksd-kernel}, and \ref{ass:concentration}, let $L_\mathrm{cpl}>0$ be the coupling exponent
obtained from Proposition~\ref{prop:refined-coupling} under these assumptions.
Then there is $C<\infty$ such that
\begin{equation}
\sup_{t\ge0}\E\KSD_\pi(\mu_t^N,\mu_t)
\le C\left\{
N^{-\frac{\varepsilon\alpha}{L_\mathrm{cpl}+2\varepsilon\alpha}}
+N^{-1/2}\right\},
\label{eq:KSD-cutoff-coupling}
\end{equation}
\end{theorem}

\begin{proof}
Corollary~\ref{cor:finite-KSD-coupling} gives the finite-time estimate
\begin{equation}\label{eq:KSD-coupling-short}
    \E\KSD_\pi(\mu_t^N,\mu_t)
    \le
    C N^{-1/2}e^{L_{\mathrm{cpl}}t/2}.
\end{equation}
For the long-time regime, the triangle inequality in \eqref{eq:ksd-triangle-target}, together with Proposition~\ref{prop:particle-KSD-target} and
Corollary~\ref{cor:mf-KSD-target}, yield
\begin{equation}\label{eq:KSD-coupling-long}
    \E\KSD_\pi(\mu_t^N,\mu_t)
    \le
    C\left\{e^{-\varepsilon\alpha t}+N^{-1/2}\right\}.
\end{equation}

Thus Proposition~\ref{prop:cutoff} applies to \(d=\KSD_\pi\) with
\[
    a_N=b_N=N^{-1/2},
    \qquad
    L=L_\mathrm{cpl}/2,
    \qquad
    c=\varepsilon\alpha,
\]
which gives \eqref{eq:KSD-coupling-long}.
\end{proof}

\subsection{Uniform-in-time PoC in \texorpdfstring{$W_2$}{W2}}

\begin{theorem}[Uniform-in-time empirical $W_2$ propagation of chaos]\label{thm:W2-cutoff}
Under Assumptions~\ref{ass:core-target}, \ref{ass:W2}, and \ref{ass:concentration},
\begin{equation}\label{eq:W2-cutoff}
    \sup_{t\ge0}\E W_2^2(\mu_t^N,\mu_t)
    \le C\left\{
       N^{-2\varepsilon\alpha/
          (L_\mathrm{cpl}+2\varepsilon\alpha)}
       +r_{N,d}\right\}
    \le C\left\{
       r_{N,d}^{\,2\varepsilon\alpha/
          (L_\mathrm{cpl}+2\varepsilon\alpha)}
       +r_{N,d}\right\}.
\end{equation}
\end{theorem}

\begin{proof}
Corollary~\ref{cor:finite-W2} gives
\[
  \E W_2^2(\mu_t^N,\mu_t)
  \le C\{N^{-1}e^{L_\mathrm{cpl}t}+r_{N,d}\}.
\]
Proposition~\ref{prop:particle-W2-target}, Corollary~\ref{cor:FG}, and
\eqref{eq:mf-W2-target} give the corresponding target-regime estimate
\[
  \E W_2^2(\mu_t^N,\mu_t)
  \le C\{e^{-2\varepsilon\alpha t}+r_{N,d}\}.
\]
Apply Proposition~\ref{prop:cutoff} to $d=W_2^2$ with
\[
  a_N=N^{-1},\qquad b_N=r_{N,d},\qquad
  L=L_\mathrm{cpl},\qquad c=2\varepsilon\alpha.
\]
\end{proof}

\begin{remark}[Direct and indirect KSD estimates]
\label{rem:direct-KSD-worthwhile}
When both kernel assumptions hold,
Lemma~\ref{lem:witness-lipschitz} and the uniform second-moment bound imply
\[
    \KSD_\pi(\mu_t^N,\mu_t)\le C W_2(\mu_t^N,\mu_t),
\]
where $C$ is time-independent.
Thus Theorem~\ref{thm:W2-cutoff} also implies a KSD estimate, but it loses a
square root and the rates worsen, in comparison to Theorem~\ref{cor:coupling-KSD-cutoff}, because of the curse of dimensionality inherited from the empirical Wasserstein rate $r_{N,d}$.
\end{remark}

\subsection{Uniform-in-time PoC for fixed-marginal laws}

Finite-time PoC estimates in $W_2$ were obtained in Corollary~\ref{cor:finite-marginal-W2}. To get the long-time bounds, we first present the following entropy superadditivity property. This helps obtain quantitative finite-marginal bounds from the full entropy bound derived in Theorem~\ref{prop:full-law}.

\begin{lemma}[Block entropy bound for exchangeable laws]
\label{lem:block-entropy}
Let \(P^N\) be an exchangeable probability measure on
\((\R^d)^N\), let \(\rho\in\mathcal P(\R^d)\), and let
\(1\le k\le N\). Then
\begin{equation}\label{eq:block-entropy}
    \KL(P^{N,k}\|\rho^{\otimes k})
    \le
    \frac{1}{\lfloor N/k\rfloor}
    \KL(P^N\|\rho^{\otimes N}),
\end{equation}
where \(P^{N,k}\) denotes the common \(k\)-coordinate marginal of
\(P^N\).
\end{lemma}

\begin{proof}
Set
\(
    q:=\left\lfloor\frac Nk\right\rfloor
\)
and divide the first \(qk\) coordinates into \(q\) consecutive blocks of
size \(k\). Since relative entropy decreases under measurable projections,
\begin{equation}\label{eq:block-data-processing}
    \KL(P^N\|\rho^{\otimes N})
    \ge
    \KL\bigl(\Law_{P^N}(X_1,\ldots,X_{qk})
       \,\|\,\rho^{\otimes qk}\bigr).
\end{equation}

Regard these coordinates as \(q\) block variables. The chain rule
\cite[Thm. 2.6]{budhiraja2019analysis} for relative entropy gives
\begin{align*}
    &\KL\bigl(\Law_{P^N}(X_1,\ldots,X_{qk})
       \,\|\,\rho^{\otimes qk}\bigr)\\
    &\quad=
    \sum_{\ell=1}^q
    \E\!\left[
       \KL\left(
          \Law_{P^N}\left(
             X_{(\ell-1)k+1},\ldots,X_{\ell k}
             \,\middle|\,X_1,\ldots,X_{(\ell-1)k}
          \right)
          \,\middle\|\,\rho^{\otimes k}
       \right)
    \right].
\end{align*}
For \(\ell=1\), the conditioning is omitted. By convexity of relative
entropy in its first argument \cite[Lem. 2.4]{budhiraja2019analysis}, each
term in the sum satisfies
\begin{align*}
    &\E\!\left[
       \KL\left(
          \Law_{P^N}\left(
             X_{(\ell-1)k+1},\ldots,X_{\ell k}
             \,\middle|\,X_1,\ldots,X_{(\ell-1)k}
          \right)
          \,\middle\|\,\rho^{\otimes k}
       \right)
    \right]\\
    &\qquad\ge
    \KL\bigl(
       \Law_{P^N}(X_{(\ell-1)k+1},\ldots,X_{\ell k})
       \,\|\,\rho^{\otimes k}\bigr).
\end{align*}
The law on the right is obtained by integrating the regular conditional law
in the preceding display against the law of the previous coordinates.
Consequently,
\[
    \KL\bigl(\Law_{P^N}(X_1,\ldots,X_{qk})
       \,\|\,\rho^{\otimes qk}\bigr)
    \ge
    \sum_{\ell=1}^q
    \KL\bigl(
       \Law_{P^N}(X_{(\ell-1)k+1},\ldots,X_{\ell k})
       \,\|\,\rho^{\otimes k}\bigr).
\]
Exchangeability implies that every block marginal in this sum coincides, up
to a relabeling of coordinates, with \(P^{N,k}\). Therefore,
\[
    \KL\bigl(\Law_{P^N}(X_1,\ldots,X_{qk})
       \,\|\,\rho^{\otimes qk}\bigr)
    \ge
    q\,\KL(P^{N,k}\|\rho^{\otimes k}).
\]
Combining this with \eqref{eq:block-data-processing} and recalling
that \(q=\lfloor N/k\rfloor\) proves \eqref{eq:block-entropy}.
\end{proof}

\begin{proposition}[Fixed-marginal target bounds]\label{prop:marginal-target}
Under Assumption~\ref{ass:core-target} and either kernel assumption, there exists a finite constant $C$ such that for fixed $k$ and $N\ge2k$, and any $t \ge 0$,
\begin{align}
    \KL(P_{N,t}^{(k)}\|\pi^{\otimes k})
    &\le Ck e^{-2\varepsilon\alpha t}+C\frac{k}{N},
      \label{eq:marginal-particle-target}\\
    \KL(\mu_t^{\otimes k}\|\pi^{\otimes k})
    &=k\KL(\mu_t\|\pi)
    \le Ck e^{-2\varepsilon\alpha t}.
      \label{eq:marginal-mf-target}
\end{align}
\end{proposition}

\begin{proof}
Apply Lemma~\ref{lem:block-entropy} to $H_N(t)=\KL(P_t^N\|\Pi_N)$ and use Theorem~\ref{prop:full-law}.  The mean-field identity follows by tensorization along with Proposition~\ref{prop:mf-entropy}.
\end{proof}

\begin{lemma}[Long-time fixed-marginal target comparison]
\label{lem:marginal-W2-long}
Under Assumption~\ref{ass:core-target} and either kernel assumption, for
fixed $k$, $N\ge2k$, and $t\ge0$,
\begin{equation}\label{eq:marginal-W2-long}
    W_2^2(P_{N,t}^{(k)},\mu_t^{\otimes k})
    \le C_k\left(e^{-2\varepsilon\alpha t}+N^{-1}\right).
\end{equation}
\end{lemma}

\begin{proof}
Observe that
\[
W_2^2(P_{N,t}^{(k)},\mu_t^{\otimes k}) \le 2 W_2^2(P_{N,t}^{(k)},\pi^{\otimes k}) + 2 W_2^2(\mu_t^{\otimes k},\mu_t^{\otimes k}).
\]
The target $T_2(1/\alpha)$ inequality, implied by the target LSI, tensorizes.  Apply it to each term on the right hand side above and use
Proposition~\ref{prop:marginal-target} to get the result.
\end{proof}

\begin{theorem}[Uniform-in-time fixed-marginal chaos in $W_2$]\label{thm:marginal-W2}
Under Assumptions~\ref{ass:core-target}, \ref{ass:W2}, and \ref{ass:concentration}, for fixed \(k\) and \(N\ge2k\),
\begin{equation}\label{eq:marginal-W2}
    \sup_{t\ge0}
    W_2^2(P_{N,t}^{(k)},\mu_t^{\otimes k})
    \le C_k\left\{
       N^{-\frac{2\varepsilon\alpha}
          {L_\mathrm{cpl}+2\varepsilon\alpha}}
       +N^{-1}\right\},
\end{equation}
where \(P^{N,k}\) denotes the \(k\)-coordinate marginal of
\(P^N\).
\end{theorem}

\begin{proof}
For short times, Corollary~\ref{cor:finite-marginal-W2} gives
\begin{equation}\label{eq:marginal-W2-short}
    W_2^2(P_{N,t}^{(k)},\mu_t^{\otimes k})
    \le C_kN^{-1}e^{L_\mathrm{cpl}t}.
\end{equation}
For long times, Lemma~\ref{lem:marginal-W2-long} gives
\begin{equation}\label{eq:marginal-W2-long-use}
    W_2^2(P_{N,t}^{(k)},\mu_t^{\otimes k})
    \le
    C_k\left\{
        e^{-2\varepsilon\alpha t}+N^{-1}
    \right\}.
\end{equation}
Taking
\[
    T_N=(L_{\mathrm{cpl}}+2\varepsilon\alpha)^{-1}\log N
\]
balances \eqref{eq:marginal-W2-short} and
\eqref{eq:marginal-W2-long-use} and proves the result.
\end{proof}

\section{PoC using the moving-product entropy method}
\label{sec:moving-entropy}

In this section, we describe an alternate approach to studying PoC using the moving-product relative entropy
\begin{equation}\label{eq:JN}
    J_N(t):=\KL(P_t^N\|\mu_t^{\otimes N}).
\end{equation}
This argument follows the quantitative relative entropy framework of
\cite{jabinWang2018}, which compares the joint particle law with the
evolving product law and closes the entropy evolution using an
exponential law of large numbers under the latter.  

Proposition~\ref{prop:moving-entropy} below gives the key quantitative moving-entropy bound. In Section~\ref{sec:fpcon}, we use this bound along with entropy superadditivity, and Pinsker's inequality, to obtain \emph{PoC results in KL divergence and total variation distance} for fixed particle marginal laws. In Section~\ref{sec:finite-KSD}, we use Proposition~\ref{prop:moving-entropy} to obtain an alternate route to PoC results for KSD. Finally, in Section~\ref{app:entropy-W2}, we show how this moving-entropy bound, under an additional $T_2$ inequality assumption on the initial distribution, gives alternate finite-time and uniform-in-time PoC rates for the empirical distribution and particle marginals in $W_2$. The rates in Sections~\ref{sec:finite-KSD} and \ref{app:entropy-W2} can, in fact, produce better bounds than the synchronous coupling method, especially when the kernel is less regular. A discussion of when one method outperforms the other, with examples, is given in Section~\ref{sec:W2-method-comparison}.

\begin{proposition}[Single-exponential moving-product entropy]\label{prop:moving-entropy}
Under the assumptions of Proposition~\ref{prop:exp-lln},
the map \(t\mapsto J_N(t)\) is locally absolutely continuous and, for almost every \(t>0\),
\begin{equation}\label{eq:moving-entropy-ode}
    J_N'(t)
    \le\frac{1}{2\varepsilon\lambda}J_N(t)
      +\frac{\log C_G}{2\varepsilon\lambda}.
\end{equation}
If $P_0^N=\mu_0^{\otimes N}$, then, with
\begin{equation}\label{eq:kappa}
    \kappa:=\frac{1}{2\varepsilon\lambda},
\end{equation}
\begin{equation}\label{eq:moving-entropy-bound}
    J_N(t)\le (e^{\kappa t}-1)\log C_G.
\end{equation}
\end{proposition}

\begin{proof}
Write
\[
    q_t^N:=\mu_t^{\otimes N}
\]
(where, by abuse of notation, we identify $\mu_t$ with its density), and introduce the particle and mean-field drifts
\[
\begin{split}
    b_i^N(X)
    &:=
    \frac1N\sum_{j=1}^N B(x_i,x_j)
    -\varepsilon\nabla V(x_i),\\
    \bar b_{i,t}(X)
    &:=
    v_{\mu_t}(x_i)-\varepsilon\nabla V(x_i).
\end{split}
\]
The corresponding Fokker--Planck equations are
\begin{align}
    \partial_t p_t^N
    &=
    -\sum_{i=1}^N\nabla_i\cdot\bigl(b_i^Np_t^N\bigr)
    +\varepsilon\sum_{i=1}^N\Delta_i p_t^N,
    \label{eq:particle-FP-JN}\\
    \partial_t q_t^N
    &=
    -\sum_{i=1}^N\nabla_i\cdot
       \bigl(\bar b_{i,t}q_t^N\bigr)
    +\varepsilon\sum_{i=1}^N\Delta_i q_t^N.
    \label{eq:product-FP-JN}
\end{align}
\eqref{eq:product-FP-JN} follows by differentiating
\(q_t^N(X)=\prod_{i=1}^N\mu_t(x_i)\) and using the nonlinear
Fokker--Planck equation in each coordinate.

By Lemma \ref{lem:nonlinear-drift-regularity}, $b^N,\bar b^N\in C_{\mathrm{loc}}^{\alpha/2,\,1+\alpha}((0,T]\times(\mathbb R^d)^N;(\mathbb R^d)^N)$ for every $\alpha\in(0,1)$ and satisfy the linear-growth condition.
Hence, by Lemma~\ref{lem:finite-dimensional-entropy}, for almost every $t>0$,
\begin{align}
    J_N'(t)
    =
    -\varepsilon\sum_{i=1}^N
       \int
       \left\|\nabla_i\log\frac{p_t^N}{q_t^N}\right\|^2
       p_t^N\,dX \
   +
    \sum_{i=1}^N\int
       \left\langle
          b_i^N(X)-\bar b_{i,t}(X),
          \nabla_i\log\frac{p_t^N}{q_t^N}
       \right\rangle
       p_t^N\,dX.
    \label{eq:JN-exact-derivative}
\end{align}
The common Langevin drift cancels from the drift difference:
\[
    b_i^N(X)-\bar b_{i,t}(X)
    =
    \frac1N\sum_{j=1}^N B(x_i,x_j)
    -v_{\mu_t}(x_i).
\]
Thus \eqref{eq:JN-exact-derivative} simplifies to
\begin{align*}
    J_N'(t)
    &=
    -\varepsilon
       \int
       \left\|
          \nabla\log\frac{p_t^N}{q_t^N}
       \right\|^2p_t^N\,dX\\
    &\quad+
    \sum_{i=1}^N\int
       \left\langle
          \frac1N\sum_{j=1}^N B(x_i,x_j)-v_{\mu_t}(x_i),
          \nabla_i\log\frac{p_t^N}{q_t^N}
       \right\rangle p_t^N\,dX.
\end{align*}
Applying Young's inequality coordinatewise gives
\[
    J_N'(t)
    \le
    -\frac{\varepsilon}{2}
       \int
       \left\|
          \nabla\log\frac{p_t^N}{q_t^N}
       \right\|^2p_t^N\,dX
    +\frac1{2\varepsilon}\int G_t^N\,dP_t^N,
\]
and hence
\[
    J_N'(t)
    \le\frac1{2\varepsilon}\int G_t^N\,dP_t^N.
\]
The entropy variational inequality \cite[Prop. 2.3]{budhiraja2019analysis} with reference $q_t^N$ and test function $\lambda G_t^N$ gives
\[
    \int G_t^N\,dP_t^N
    \le\frac1\lambda J_N(t)
      +\frac1\lambda\log\int e^{\lambda G_t^N}\,dq_t^N.
\]
Proposition~\ref{prop:exp-lln} proves \eqref{eq:moving-entropy-ode}; solving it with $J_N(0)=0$ gives \eqref{eq:moving-entropy-bound}.
\end{proof}

\subsection{Fixed-particle marginal estimates}\label{sec:fpcon}

For \(1\le k\le N\), let \(P_{N,t}^{(k)}\) denote the
\(k\)-particle marginal of \(P_t^N\).

\begin{corollary}[Finite-time fixed-marginal propagation of chaos]
\label{cor:finite-marginal-TV}
Under Assumptions~\ref{ass:core-target} and
\ref{ass:concentration}, together with either kernel assumption, and
with product initialization
\[
    P_0^N=\mu_0^{\otimes N},
\]
the following bounds hold for every \(1\le k\le N/2\):
\begin{equation}\label{eq:finite-marginal-KL}
    \KL(P_{N,t}^{(k)}\|\mu_t^{\otimes k})
    \le
    2\log C_G\,\frac{k}{N}
    \bigl(e^{\kappa t}-1\bigr),
\end{equation}
and
\begin{equation}\label{eq:finite-marginal-TV}
    \|P_{N,t}^{(k)}-\mu_t^{\otimes k}\|_{\mathrm{TV}}
    \le
    \sqrt{\log C_G}\,
    \sqrt{\frac{k}{N}}\,
    \bigl(e^{\kappa t}-1\bigr)^{1/2},
\end{equation}
where $C_G$ is as in Proposition~\ref{prop:exp-lln}.
\end{corollary}

\begin{proof}
Because the initial law is exchangeable and the particle dynamics
are invariant under permutations of the particle labels, \(P_t^N\)
is exchangeable for every \(t\ge0\). We may therefore apply
Lemma~\ref{lem:block-entropy} with \(P^N=P_t^N\) and
\(\rho=\mu_t\). This gives
\begin{equation}\label{eq:marginal-entropy-from-JN}
    \KL(P_{N,t}^{(k)}\|\mu_t^{\otimes k})
    \le
    \frac{J_N(t)}{\lfloor N/k\rfloor}.
\end{equation}
Since $\lfloor N/k\rfloor\ge N/(2k)$,
using
Proposition~\ref{prop:moving-entropy}, we obtain \eqref{eq:finite-marginal-KL}.

Pinsker's inequality gives
\[
    \|P_{N,t}^{(k)}-\mu_t^{\otimes k}\|_{\mathrm{TV}}
    \le
    \left\{
       \frac12
       \KL(P_{N,t}^{(k)}\|\mu_t^{\otimes k})
    \right\}^{1/2}.
\]
Using \eqref{eq:finite-marginal-KL}, we conclude that
\[
    \|P_{N,t}^{(k)}-\mu_t^{\otimes k}\|_{\mathrm{TV}}
    \le
    \sqrt{\log C_G}\,
    \sqrt{\frac{k}{N}}\,
    \bigl(e^{\kappa t}-1\bigr)^{1/2},
\]
which is \eqref{eq:finite-marginal-TV}.
\end{proof}

Next we obtain finite-marginal uniform-in-time PoC rates.

\begin{theorem}[Uniform-in-time fixed-marginal chaos in total variation]\label{thm:marginal-TV}
Under Assumptions~\ref{ass:core-target} and \ref{ass:concentration}, together with either kernel assumption, for fixed $k$ and $N\ge2k$,
\begin{equation}\label{eq:marginal-TV}
    \sup_{t\ge0}
    \|P_{N,t}^{(k)}-\mu_t^{\otimes k}\|_{\mathrm{TV}}
    \le C_k\left\{
       N^{-\frac{\varepsilon\alpha}{\kappa+2\varepsilon\alpha}}
       +N^{-1/2}\right\},
\end{equation}
where \(C_k\) may depend on \(k\) and the standing model parameters, but is
independent of \(N\) and \(t\).
\end{theorem}

\begin{proof}
For short times, Corollary~\ref{cor:finite-marginal-TV} gives
\begin{equation}\label{eq:marginal-TV-short}
    \|P_{N,t}^{(k)}-\mu_t^{\otimes k}\|_{\mathrm{TV}}
    \le
    C_kN^{-1/2}e^{\kappa t/2}.
\end{equation}
For long times, insert the target product measure
\(\pi^{\otimes k}\). By the triangle inequality and Pinsker's
inequality,
\begin{align*}
    \|P_{N,t}^{(k)}-\mu_t^{\otimes k}\|_{\mathrm{TV}}
    &\le
    \|P_{N,t}^{(k)}-\pi^{\otimes k}\|_{\mathrm{TV}}
    +
    \|\mu_t^{\otimes k}-\pi^{\otimes k}\|_{\mathrm{TV}}\\
    &\le
    \left\{
       \frac12
       \KL(P_{N,t}^{(k)}\|\pi^{\otimes k})
    \right\}^{1/2}
    +
    \left\{
       \frac12
       \KL(\mu_t^{\otimes k}\|\pi^{\otimes k})
    \right\}^{1/2}.
\end{align*}
Proposition~\ref{prop:marginal-target} therefore yields
\begin{equation}\label{eq:marginal-TV-long}
    \|P_{N,t}^{(k)}-\mu_t^{\otimes k}\|_{\mathrm{TV}}
    \le
    C_k\left(e^{-\varepsilon\alpha t}+N^{-1/2}\right).
\end{equation}

Choose
\[
    T_N:=\frac{\log N}{\kappa+2\varepsilon\alpha}.
\]
Then
\[
    N^{-1/2}e^{\kappa T_N/2}
    =
    e^{-\varepsilon\alpha T_N}
    =
    N^{-\varepsilon\alpha/
       (\kappa+2\varepsilon\alpha)}.
\]
For \(t\le T_N\), use \eqref{eq:marginal-TV-short}, while for
\(t>T_N\), use \eqref{eq:marginal-TV-long}. This gives
\[
    \sup_{t\ge0}
    \|P_{N,t}^{(k)}-\mu_t^{\otimes k}\|_{\mathrm{TV}}
    \le
    C_k\left\{
       N^{-\varepsilon\alpha/
          (\kappa+2\varepsilon\alpha)}
       +N^{-1/2}
    \right\}.
\]
\end{proof}

\begin{remark}[Why the cutoff argument does not give uniform-in-time KL propagation of chaos]
\label{rem:no-uniform-KL}
The preceding estimates do not yield a corresponding uniform-in-time
PoC bound in relative entropy.  Indeed,
Corollary~\ref{cor:finite-marginal-TV} gives finite-time control of
\(
    \KL(P_{N,t}^{(k)}\|\mu_t^{\otimes k}),
\)
while Proposition~\ref{prop:marginal-target} shows that both \(P_{N,t}^{(k)}\) and
\(\mu_t^{\otimes k}\) are close to the common target
\(\pi^{\otimes k}\) in relative entropy.  Unlike total variation or
\(W_2\), however, relative entropy does not satisfy a triangle
inequality.  Thus one cannot insert \(\pi^{\otimes k}\) to deduce a
bound on
\(\KL(P_{N,t}^{(k)}\|\mu_t^{\otimes k})\) from the two target-entropy
bounds.  Consequently, the cutoff argument used in
Theorem~\ref{thm:marginal-TV} does not by itself provide
uniform-in-time propagation of chaos in KL.  Such a result would
require an additional direct long-time estimate comparing
\(P_{N,t}^{(k)}\) with \(\mu_t^{\otimes k}\).
\end{remark}

\subsection{Moving-product entropy-based PoC in Langevin KSD}\label{sec:finite-KSD}

The following lemma gives exponential moment bounds for the KSD under the law $\mu_t^{\otimes N}$, which is then combined with the entropy variational inequality and Proposition~\ref{prop:moving-entropy} to obtain the quantitative finite-time bound in Corollary~\ref{cor:finite-KSD}.

\begin{lemma}[Concentration of the Stein witness map]\label{lem:witness-concentration}
Under Assumptions~\ref{ass:core-target}, \ref{ass:ksd-kernel} and \ref{ass:concentration}, there are constants $c_{\mathrm{KSD}}>0$ and $1\le C_{\mathrm{KSD}}<\infty$ such that
\begin{equation}\label{eq:witness-concentration}
    \int\exp\left(
      c_{\mathrm{KSD}}\sqrt N\,
      \KSD_\pi(\mu_X^N,\mu_t)
    \right)\,d\mu_t^{\otimes N}(X)
    \le C_{\mathrm{KSD}}
\end{equation}
uniformly in $N$ and $t$.
\end{lemma}

\begin{proof}
By Proposition~\ref{prop:second-moments}, 
\[
    \sup_{t\ge0}\int_{\R^d}\|x\|\,\mu_t(dx)<\infty.
\]
Because \(\mathbb R^d\) is separable and \(k\) is continuous, the scalar
RKHS \(\mathcal H\), and hence \(\mathcal H^d\), is separable.  The
derivative reproducing identities and the smoothness assumption of \(k\) show
that \(x\mapsto\xi_\pi(x)\) is continuous as an
\(\mathcal H^d\)-valued map.  It is therefore strongly (Bochner)
measurable.
By Lemma~\ref{lem:witness-lipschitz}, there is a constant
\(C<\infty\) such that
\(
    \|\xi_\pi(x)\|_{\mathcal H^d}
    \le C(1+\|x\|).
\)
The preceding moment bound then gives Bochner integrability, so
\[
    S_\pi(\mu_t)=\int_{\R^d}\xi_\pi(y)\,\mu_t(dy)
\]
is well defined and satisfies
\[
\begin{aligned}
    \sup_{t\ge0}\|S_\pi(\mu_t)\|_{\mathcal H^d}
    \le
    \sup_{t\ge0}
    \int_{\R^d}
       \|\xi_\pi(y)\|_{\mathcal H^d}\,\mu_t(dy)
    <\infty.
\end{aligned}
\]
After increasing \(C\) if necessary,
\[
\begin{aligned}
    \|\xi_\pi(x)-S_\pi(\mu_t)\|_{\mathcal H^d}
    \le
    \|\xi_\pi(x)\|_{\mathcal H^d}
    +\|S_\pi(\mu_t)\|_{\mathcal H^d}
    \le C(1+\|x\|),
\end{aligned}
\]
and hence
\[
    \|\xi_\pi(x)-S_\pi(\mu_t)\|_{\mathcal H^d}^2
    \le C(1+\|x\|^2).
\]
Under Assumption~\ref{ass:ksd-kernel}, Lemma~\ref{lem:witness-lipschitz} and Proposition~\ref{prop:subgaussian} thus give
\[
    \sup_{t\ge0}\int
      \exp\left(a\|\xi_\pi(x)-S_\pi(\mu_t)\|_{\mathcal H^d}^2\right)
      \mu_t(dx)<\infty
\]
for some $a>0$.  

For i.i.d. $X_i\sim\mu_t$, the centered Hilbert-valued variables
$\xi_\pi(X_i)-S_\pi(\mu_t)$ satisfy a uniform Bernstein moment condition.
Indeed, for any \(b>0\), the preceding exponential-square bound gives
\[
    \sup_{t\ge0}
    \E_{\mu_t}\exp\left(
       \frac{\|\xi_\pi(X_1)-S_\pi(\mu_t)\|_{\mathcal H^d}}{b}
    \right)
    <\infty,
\]
and \(u^m\le m!e^u\) gives, for some \(C<\infty\) and every \(m\ge2\),
\[
    \sup_{t\ge0}\E_{\mu_t}
       \|\xi_\pi(X_1)-S_\pi(\mu_t)\|_{\mathcal H^d}^m
    \le m!Cb^m
    =\frac{m!}{2}(2Cb^2)b^{m-2}.
\]
Thus the Bernstein parameters \(b\) and \(2Cb^2\) are independent
of \(t\), \(N\), and \(i\).
The Hilbert-space Bernstein inequality of Pinelis \cite{pinelis1994} yields
\[
    \mathbb P\left(
       \left\|\frac1N\sum_{i=1}^N
          \{\xi_\pi(X_i)-S_\pi(\mu_t)\}\right\|_{\mathcal H^d}>r
    \right)
    \le2\exp\left\{-\frac{Nr^2}{C(\sigma^2+br)}\right\}
    \le2e^{-cN\min(r^2,r)},
\]
where \(\sigma^2\) and \(b\) are uniform Bernstein parameters and the
constants in the last bound are independent of \(t\) and \(N\).
To pass from this tail bound to the exponential moment,
set
\[
 Z_N:=\sqrt N\left\|\frac1N\sum_{i=1}^N
 \{\xi_\pi(X_i)-S_\pi(\mu_t)\}\right\|_{\mathcal H^d}
 =\sqrt N\,\KSD_\pi(\mu_X^N,\mu_t).
\]
For every \(u\ge0\),
\[
 \mathbb P(Z_N>u)
 \le2\exp\{-c\min(u^2,u\sqrt N)\}
 \le2\exp\{-c\min(u^2,u)\}.
\]
Thus, for any \(0<\theta<c\), the tail-integration identity gives
\begin{align*}
 \E e^{\theta Z_N}
 &=1+\theta\int_0^\infty e^{\theta u}\mathbb P(Z_N>u)\,du\\
 &\le1+2\theta\int_0^\infty
      \exp\{\theta u-c\min(u^2,u)\}\,du
 <\infty.
\end{align*}
The last bound is independent of \(N\) and \(t\).  Taking
\(c_{\mathrm{KSD}}=\theta\) and denoting this bound by
\(C_{\mathrm{KSD}}\) proves \eqref{eq:witness-concentration}.
\end{proof}

\begin{corollary}[Finite-time KSD propagation of chaos]\label{cor:finite-KSD}
Under Assumptions~\ref{ass:core-target}, \ref{ass:concentration}, and \ref{ass:ksd-kernel},
\begin{equation}\label{eq:finite-KSD}
    \E\KSD_\pi(\mu_t^N,\mu_t)
    \le\frac{J_N(t)+\log C_{\mathrm{KSD}}}
             {c_{\mathrm{KSD}}\sqrt N}
    \le\frac{C}{\sqrt N}e^{\kappa t}.
\end{equation}
\end{corollary}

\begin{proof}
Apply the entropy variational inequality first to the truncation of
$c_{\mathrm{KSD}}\sqrt N\,\KSD_\pi(\mu_X^N,\mu_t)$ at level $M$, with
reference $\mu_t^{\otimes N}$.  Letting $M\to\infty$ by monotone
convergence and using Lemma~\ref{lem:witness-concentration}, we obtain
\begin{align*}
    c_{\mathrm{KSD}}\sqrt N\,
    \E_{P_t^N}
       \KSD_\pi(\mu_t^N,\mu_t)
    \le
    J_N(t)
    +
    \log\int
       \exp\left(
          c_{\mathrm{KSD}}\sqrt N\,
          \KSD_\pi(\mu_X^N,\mu_t)
       \right)
       d\mu_t^{\otimes N}(X).
\end{align*}
Lemma~\ref{lem:witness-concentration} bounds the last integral by
\(C_{\mathrm{KSD}}\), uniformly in \(N\) and \(t\). Therefore,
\begin{equation}\label{eq:finite-KSD-step}
    \E\KSD_\pi(\mu_t^N,\mu_t)
    \le
    \frac{J_N(t)+\log C_{\mathrm{KSD}}}
         {c_{\mathrm{KSD}}\sqrt N}.
\end{equation}
By Proposition~\ref{prop:moving-entropy},
\[
    J_N(t)
    \le
    \bigl(e^{\kappa t}-1\bigr)\,\log C_G.
\]
Substituting this into
\eqref{eq:finite-KSD-step} gives
\[
    \E\KSD_\pi(\mu_t^N,\mu_t)
    \le
    \frac{\log C_G+\log C_{\mathrm{KSD}}}
         {c_{\mathrm{KSD}}\sqrt N}
    e^{\kappa t},
\]
which proves the result.
\end{proof}

\begin{corollary}
\label{thm:KSD-cutoff}
Under Assumptions~\ref{ass:core-target},
\ref{ass:ksd-kernel}, and \ref{ass:concentration},
\begin{equation}\label{eq:KSD-cutoff}
    \sup_{t\ge0}\E\KSD_\pi(\mu_t^N,\mu_t)
    \le C\left\{
       N^{-\frac{\varepsilon\alpha}{2(\kappa+\varepsilon\alpha)}}
       +N^{-1/2}\right\},
\end{equation}
where \(\kappa\) is defined in \eqref{eq:kappa}.
\end{corollary}

\begin{proof}
Corollary~\ref{cor:finite-KSD} gives the short-time estimate with
\(a_N=b_N=N^{-1/2}\) and growth exponent \(L=\kappa\).
Proposition~\ref{prop:particle-KSD-target} and
Corollary~\ref{cor:mf-KSD-target} give the corresponding target estimates with decay rate
\(c=\varepsilon\alpha\).  Proposition~\ref{prop:cutoff} now proves the claim.
\end{proof}

\subsection{Alternative entropy-to-Wasserstein estimates}
\label{app:entropy-W2}

This appendix records a second route to the Wasserstein estimates.  It is
not needed for the main Wasserstein theorems, which use the synchronous coupling from Section~\ref{app:refined-coupling}.  Instead, it
shows how the full joint-law relative entropy bound can be converted into
empirical and fixed-marginal $W_2$ estimates when the optional initial
transport inequality in Assumption~\ref{ass:initial-T2} is available.

\begin{assumption}[Initial transport inequality for the entropy-to-$W_2$ route]
\label{ass:initial-T2}
For the results below, we assume that
\begin{equation}\label{eq:initial-T2}
    \mu_0\in T_2(C_0)
\end{equation}
for some $C_0<\infty$.
\end{assumption}

\begin{proposition}[Time-dependent $T_2$ inequality]
\label{prop:time-T2}
Take any $\varepsilon>0$. Let $\rho_t$ be the law of
\[
    dZ_t=a_t(Z_t)\,dt+\sqrt{2\varepsilon}\,dW_t,
    \qquad Z_0\sim\rho_0.
\]
Assume that \((t,x)\mapsto a_t(x)\) is measurable and that, for every
\(T<\infty\), there exist \(L_T,C_T<\infty\) such that
\[
    \|a_t(x)-a_t(y)\|\le L_T\|x-y\|,
    \qquad
    \|a_t(x)\|\le C_T(1+\|x\|),
\]
for all \(t\in[0,T]\) and \(x,y\in\R^d\). Moreover, assume that
\begin{equation}\label{eq:one-sided-L}
    \langle x-y,a_t(x)-a_t(y)\rangle\le L\|x-y\|^2
\end{equation}
uniformly in $t$.  If $\rho_0\in T_2(C_0)$, then $\rho_t\in T_2(C_t)$, where
\begin{equation}\label{eq:Ct}
    C_t=C_0e^{2Lt}+2\varepsilon\int_0^t e^{2L(t-s)}\,ds.
\end{equation}
Thus
\[
    C_t=
    \begin{cases}
      C_0e^{2Lt}+\dfrac{\varepsilon}{L}(e^{2Lt}-1),&L\ne0,\\[1ex]
      C_0+2\varepsilon t,&L=0,
    \end{cases}
\]
and $C_t\le C(1+C_0)e^{\ell t}$ for some finite $C$ and some \(\ell\ge0\).
\end{proposition}

\begin{proof}
Fix $\nu\ll\rho_t$ with finite entropy.  Let $P$ be the path law of $Z$ and
define the terminal tilt
\[
    \frac{dQ}{dP}=\frac{d\nu}{d\rho_t}(Z_t).
\]
Then $Q_t=\nu$ and $\KL(Q\|P)=\KL(\nu\|\rho_t)<\infty$.  The
finite-entropy F\"ollmer--Girsanov representation
(see \cite[Prop. 4.4]{Fathi2016}) gives a progressively measurable control
$u_s$, with \(\E_Q\int_0^t\|u_s\|^2ds<\infty\), such that, under $Q$,
\[
    dZ_s=a_s(Z_s)\,ds+u_s\,ds+\sqrt{2\varepsilon}\,dW_s^Q,
\]
and
\begin{equation}\label{eq:Follmer}
    \KL(\nu\|\rho_t)
    =\KL(Q_0\|\rho_0)
      +\frac1{4\varepsilon}\E_Q\int_0^t\|u_s\|^2\,ds.
\end{equation}
On an enlarged probability space, optimally couple $Z_0\sim Q_0$ with
$Y_0\sim\rho_0$ and solve
\[
    dY_s=a_s(Y_s)\,ds+\sqrt{2\varepsilon}\,dW_s^Q.
\]
Then $Y_t\sim\rho_t$.  From \eqref{eq:one-sided-L},
\[
    \|Z_t-Y_t\|
    \le e^{Lt}\|Z_0-Y_0\|
      +\int_0^t e^{L(t-s)}\|u_s\|\,ds.
\]
Taking $L^2(Q)$ norms gives
\[
    (\E_Q\|Z_t-Y_t\|^2)^{1/2}
    \le e^{Lt}(\E_Q\|Z_0-Y_0\|^2)^{1/2}
      +\left(\int_0^t e^{2L(t-s)}\,ds\right)^{1/2}
       \left(\E_Q\int_0^t\|u_s\|^2ds\right)^{1/2}.
\]
Since $\rho_0\in T_2(C_0)$,
\[
    \E_Q\|Z_0-Y_0\|^2\le2C_0\KL(Q_0\|\rho_0).
\]
Applying the two-dimensional Cauchy--Schwarz inequality to the preceding sum
and using \eqref{eq:Follmer},
\[
    \E_Q\|Z_t-Y_t\|^2
    \le2\left\{C_0e^{2Lt}
       +2\varepsilon\int_0^t e^{2L(t-s)}\,ds\right\}
       \KL(\nu\|\rho_t).
\]
The terminal pair couples $\nu$ and $\rho_t$, proving \eqref{eq:Ct}.
\end{proof}

\begin{lemma}[Spatial regularity of the nonlinear drift]
\label{lem:drift-Lip}
Under Assumptions~\ref{ass:core-target} and \ref{ass:W2}, the map
$x\mapsto v_{\mu_t}(x)-\varepsilon\nabla V(x)$ is globally Lipschitz, with a
constant independent of $t$.
\end{lemma}

\begin{proof}
For $x,x'\in\R^d$,
\begin{align*}
    \|v_{\mu_t}(x)-v_{\mu_t}(x')\|
    &\le \|\nabla_1k\|_\infty\|x-x'\|
       \int\|\nabla V(y)\|\mu_t(dy)
    +\|\nabla_{12}^2k\|_\infty\|x-x'\|.
\end{align*}
The integral is uniformly bounded by
Proposition~\ref{prop:second-moments} and \eqref{eq:hessian-bound}.  Adding
the globally Lipschitz Langevin drift proves the claim.
\end{proof}

\begin{corollary}[Propagation of $T_2$ for the nonlinear law]
\label{cor:mu-T2}
Under Assumptions~\ref{ass:core-target}, \ref{ass:W2}, and
\ref{ass:initial-T2}, the nonlinear law satisfies
\begin{equation}\label{eq:mu-T2}
    \mu_t\in T_2(C_t),
    \qquad C_t\le C(1+C_0)e^{\ell t},
\end{equation}
for some finite $C,\ell$.
\end{corollary}

\begin{proof}
Apply Proposition~\ref{prop:time-T2} to \eqref{eq:nonlinear-sde}, using
Lemma~\ref{lem:drift-Lip}.
\end{proof}

For the remainder of this appendix, fix one exponent $\ell<\infty$ for
which \eqref{eq:mu-T2} holds and set
\begin{equation}\label{eq:L-ent}
    L_{\mathrm{ent}}:=\kappa+\ell,
\end{equation}
where $\kappa$ is defined in \eqref{eq:kappa}. Also recall $r_{N,d}$ from \eqref{eq:rNd}.

\begin{corollary}[Finite-time empirical $W_2$ bound from moving entropy]
\label{cor:finite-W2-entropy}
Under Assumptions~\ref{ass:core-target}, \ref{ass:concentration},
\ref{ass:W2}, and \ref{ass:initial-T2},
\begin{equation}\label{eq:finite-W2-entropy}
    \E W_2^2(\mu_t^N,\mu_t)
    \le\frac{4C_t}{N}J_N(t)+2Cr_{N,d}
    \le C\left\{
       \frac{e^{L_{\mathrm{ent}}t}}{N}+r_{N,d}
    \right\}
    \le Cr_{N,d}e^{L_{\mathrm{ent}}t}, \qquad t \ge 0.
\end{equation}
\end{corollary}

\begin{proof}
The $T_2(C_t)$ inequality tensorizes to $\mu_t^{\otimes N}$, so
\[
    W_2^2(P_t^N,\mu_t^{\otimes N})\le2C_tJ_N(t).
\]
Couple these laws optimally and apply Lemma~\ref{lem:empirical-map}; if
$\bar X_1(t),\ldots,\bar X_N(t)$ denote the product-reference coordinates,
then
\[
    \E W_2^2\left(
       \mu_t^N,\frac1N\sum_{i=1}^N\delta_{\bar X_i(t)}
    \right)
    \le\frac{2C_t}{N}J_N(t).
\]
The squared triangle inequality and \eqref{eq:iid-W2-floor} give the first
bound in \eqref{eq:finite-W2-entropy}.  The second follows from
\eqref{eq:moving-entropy-bound}, \eqref{eq:mu-T2}, and
\(N^{-1}\le C r_{N,d}\), with the numerical constant absorbed into the
generic prefactor.
\end{proof}

\begin{corollary}[Finite-time fixed-marginal $W_2$ bound from moving entropy]
\label{cor:finite-marginal-W2-entropy}
Under Assumptions~\ref{ass:core-target},
\ref{ass:W2}, \ref{ass:concentration}, and \ref{ass:initial-T2}, for fixed $k$ and $N\ge2k$,
\begin{equation}\label{eq:finite-marginal-W2-entropy}
    W_2^2(P_{N,t}^{(k)},\mu_t^{\otimes k})
    \le C\frac{k}{N}e^{L_{\mathrm{ent}}t}, \qquad t \ge 0.
\end{equation}
\end{corollary}

\begin{proof}
The $T_2(C_t)$ inequality tensorizes to $\mu_t^{\otimes k}$.  Combine it
with Lemma~\ref{lem:block-entropy}, Proposition~\ref{prop:moving-entropy},
and $\lfloor N/k\rfloor\ge N/(2k)$.
\end{proof}

\begin{corollary}[Alternative entropy-based uniform-in-time $W_2$ bounds]
\label{cor:entropy-W2-cutoff}
Under Assumptions~\ref{ass:core-target},
\ref{ass:W2}, \ref{ass:concentration}, and \ref{ass:initial-T2},
\begin{equation}\label{eq:W2-cutoff-entropy}
    \sup_{t\ge0}\E W_2^2(\mu_t^N,\mu_t)
    \le C\left\{
       N^{-2\varepsilon\alpha/
          (L_{\mathrm{ent}}+2\varepsilon\alpha)}
       +r_{N,d}\right\}
    \le C\left\{
       r_{N,d}^{\,2\varepsilon\alpha/
          (L_{\mathrm{ent}}+2\varepsilon\alpha)}
       +r_{N,d}\right\}.
\end{equation}
Moreover, for fixed $k$ and $N\ge2k$,
\begin{equation}\label{eq:marginal-W2-cutoff-entropy}
    \sup_{t\ge0}W_2^2(P_{N,t}^{(k)},\mu_t^{\otimes k})
    \le C_k\left\{
       N^{-2\varepsilon\alpha/
          (L_{\mathrm{ent}}+2\varepsilon\alpha)}
       +N^{-1}\right\}.
\end{equation}
\end{corollary}

\begin{proof}
Corollary~\ref{cor:finite-W2-entropy} gives
\[
    \E W_2^2(\mu_t^N,\mu_t)
    \le
    C\left\{N^{-1}e^{L_{\mathrm{ent}}t}+r_{N,d}\right\}.
\]
Moreover, Proposition~\ref{prop:particle-W2-target},
Corollary~\ref{cor:FG}, and \eqref{eq:mf-W2-target} imply
\[
    \E W_2^2(\mu_t^N,\pi)
    \le C\left\{e^{-2\varepsilon\alpha t}+r_{N,d}\right\},
    \qquad
    W_2^2(\mu_t,\pi)
    \le Ce^{-2\varepsilon\alpha t},
\]
where the \(N^{-1}\) term in
\eqref{eq:particle-W2-target} has been absorbed into \(r_{N,d}\).
Thus Proposition~\ref{prop:cutoff} applies to \(d=W_2^2\) with
\[
    a_N=N^{-1},\qquad b_N=r_{N,d},\qquad
    L=L_{\mathrm{ent}},\qquad c=2\varepsilon\alpha.
\]
This gives the empirical bound \eqref{eq:W2-cutoff-entropy}.

For fixed marginals, writing $T_N=(L_{\mathrm{ent}}+2\varepsilon\alpha)^{-1}\log N$, Corollary~\ref{cor:finite-marginal-W2-entropy}
gives, for \(t\le T_N\),
\[
    W_2^2(P_{N,t}^{(k)},\mu_t^{\otimes k})
    \le
    C_kN^{-1}e^{L_{\mathrm{ent}}T_N}.
\]
For \(t>T_N\), Lemma~\ref{lem:marginal-W2-long} gives
\[
    W_2^2(P_{N,t}^{(k)},\mu_t^{\otimes k})
    \le
    C_k\left\{
        e^{-2\varepsilon\alpha T_N}+N^{-1}
    \right\}.
\]
Combining the two time regimes proves
\eqref{eq:marginal-W2-cutoff-entropy}.
\end{proof}

\subsection{Comparison of the coupling and entropy routes}
\label{sec:W2-method-comparison}

The preceding results provide two distinct routes to Wasserstein propagation
of chaos.  The synchronous-coupling argument works directly at the level of
particle trajectories and does not require a transportation inequality for
the initial law.  By contrast, the entropy argument first controls the
relative entropy of the joint particle law with respect to the moving
product law.  It therefore gives stronger information, including marginal
relative entropy and total variation bounds, but its conversion to
Wasserstein distance requires the additional assumption
\(\mu_0\in T_2(C_0)\).
Since both cutoff arguments use the same long-time estimate, the
difference between the resulting exponents comes directly from the finite-time growth
constants.  The entropy bound has the smaller growth exponent, and hence a better uniform-in-time PoC rate, when
$$
    L_{\mathrm{ent}}<L_{\mathrm{cpl}},
$$
while the coupling bound gives a better rate when
\(L_{\mathrm{cpl}}<L_{\mathrm{ent}}\).  This comparison applies only to the
estimates proved here, which need not be sharp for a particular model.  For
the empirical \(W_2^2\) bound, the sampling term \(r_{N,d}\) may still
determine the overall rate.

We now try to quantify a more explicit estimate of the finite-time growth constants.  Set
$$
    M_1:=\sup_{t\geq0}\int_{\R^d}\|x\|\,\mu_t(dx),
    \qquad
    \mathfrak M_a
    :=
    \sup_{t\geq0}\int e^{a\|x\|^2}\,\mu_t(dx),
$$

and introduce the three kernel--potential quantities
\begin{align*}
K_{\mathrm{amp}}
&:=
\|k\|_\infty\bigl(\|\nabla V(0)\|+L_V\bigr)
+\|\nabla_2k\|_\infty,\\
K_x
&:=
\|\nabla_1k\|_\infty
\bigl(\|\nabla V(0)\|+L_VM_1\bigr)
+\|\nabla_{12}^2k\|_\infty,\\
K_{\mathrm{lip}}
&:=
\bigl(\|\nabla V(0)\|+L_V\bigr)
\bigl(\|\nabla_1k\|_\infty+\|\nabla_2k\|_\infty\bigr)
+L_V\|k\|_\infty
+\|\nabla_{12}^2k\|_\infty
+\|\nabla_{22}^2k\|_\infty .
\end{align*}
Here \(K_{\mathrm{amp}}\) controls the magnitude of the interaction,
$$
    \|B(x,y)\|\leq K_{\mathrm{amp}}(1+\|y\|),
$$
\(K_x\) controls the spatial Lipschitz constant of
\(x\mapsto v_{\mu_t}(x)\),
and and the bound for \(\|B(x,y)-B(x',y')\|\) in
\eqref{eq:B-weighted-y-lipschitz} holds with
\(K_{\mathrm{lip}}\) in place of \(C_B\), hence, \[ \begin{aligned}
\|B(x,y)-B(x',y')\|
\le{}&
K_{\mathrm{lip}}
\bigl(1+\min\{\|y\|,\|y'\|\}\bigr)
\bigl(\|x-x'\|+\|y-y'\|\bigr).
\end{aligned}\] 

For any \(0<a<\min\{a_0,m/8\}\), let
\(C_{a,\mathfrak M_a,d}>0\) be sufficiently large so that
\[
    \lambda
    =\frac{1}
           {2C_{a,\mathfrak M_a,d}K_{\mathrm{amp}}^2}
\]
is an admissible exponent in Proposition~\ref{prop:exp-lln}. Consequently, 
$
    \kappa
    =
    \frac{1}{2\varepsilon\lambda}
    =
    C_{a,\mathfrak M_a,d}
    \frac{K_{\mathrm{amp}}^2}{\varepsilon}.
$
If \(K_{\mathrm{amp}}=0\), then \(B\equiv0\), and one may take \(\kappa=0\). The generic propagation estimate for the transportation inequality allows the choice $ 
    \ell=2\bigl(K_x+\varepsilon L_V\bigr)$ in Proposition~\ref{prop:time-T2}. If a sharper growth estimate for the \(T_2\) constants of \(\mu_t\)
is known directly, that smaller value of \(\ell\) should instead be used. The constants obtained in the proofs then satisfy, up to numerical factors,
\begin{align}
L_{\mathrm{cpl}}
&\lesssim
\frac{L_0}{2}
+1+\varepsilon L_V+K_{\mathrm{lip}}
+\frac{K_{\mathrm{lip}}^2}
{a\varepsilon(m-4a)},
\label{eq:Lcpl-schematic}\\
L_{\mathrm{ent}}
&=\kappa+\ell
\le
C_{a,\mathfrak M_a,d}
\frac{K_{\mathrm{amp}}^2}{\varepsilon}
+2\bigl(K_x+\varepsilon L_V\bigr).
\label{eq:Lent-schematic}
\end{align}
Here \(L_0\) is the exponent in \eqref{eq:path-exponential}. With
$$
    \begin{aligned}
        M_v
        &:=
        \sup_{t\ge0}\sup_{x\in\R^d}\|v_{\mu_t}(x)\| 
         \le
        \|k\|_\infty
           \bigl(\|\nabla V(0)\|+L_VM_1\bigr)
        +\|\nabla_2k\|_\infty,
    \end{aligned}
$$
we have by \eqref{eq:lgendef},
$$
    L_0 = C_{\mathrm{gen}} = 
    \frac{aM_v^2}{\varepsilon m}
    +2a\varepsilon(b+d).
$$
Thus the coupling argument is sensitive to
the two-variable difference scale of \(B\) captured by
\(K_{\mathrm{lip}}\), and in particular to
\(\|\nabla_{22}^2k\|_\infty\), through the quadratic contribution
\(K_{\mathrm{lip}}^2/\varepsilon\).  The moving-product entropy estimate
itself depends instead on the magnitude of \(B\), through
\(K_{\mathrm{amp}}\).  The weaker \(x\)-regularity quantity \(K_x\) enters
only when the entropy estimate is converted into a Wasserstein estimate by
propagating a \(T_2\) inequality.  Consequently, the entropy route can be
sharper for kernel families whose interaction amplitudes remain controlled
while their higher spatial derivatives become large. The following simple examples show that either of the two approaches could yield better
rates, depending on the specific case.
For the narrow kernels below, the entropy cutoff exponent remains stable as \(h\downarrow0\), while the exponent obtained from the coupling estimate deteriorates.  For the constant kernel, direct coupling gives an \(N^{-1}\) bound, while the entropy cutoff estimate gives a smaller power of
\(N\).

\subsubsection{A narrow-kernel regime favoring entropy}
\label{sec:entropy-better-example}

Let

$$
    \pi=\mu_0=\mathcal N(0,I_d),
    \qquad
    V(x)=\frac12\|x\|^2,
$$

and consider

$$
    k_h(x,y)
    :=
    h\exp\left(-\frac{\|x-y\|^2}{2h^2}\right),
    \qquad 0<h\leq1.
$$

The Stein identity gives \(v_\pi=0\), and hence
\(\mu_t=\pi\) for all \(t\geq0\).  Thus the \(T_2(1)\) constant is
time independent and we may take \(\ell=0\), where \(\ell\) is as in
Corollary~\ref{cor:mu-T2}.  Moreover, for fixed
\(a\in(0,1/8)\),

$$
    \mathfrak M_a
    =
    \int e^{a\|x\|^2}\,\pi(dx)
    =(1-2a)^{-d/2},
$$

independently of \(h\).

The kernel bounds

$$
    \|k_h\|_\infty=O(h),\qquad
    \|\nabla k_h\|_\infty=O(1),\qquad
    \|\nabla^2k_h\|_\infty=O(h^{-1})
$$

give

$$
    K_{\mathrm{amp}}=O(1),
    \qquad
    K_x=O(h^{-1}),
    \qquad
    K_{\mathrm{lip}}=O(h^{-1}).
$$
Since \(\mu_t=\pi\), we use the exact stationary \(T_2(1)\) estimate
and take \(\ell=0\), rather than the generic drift-based upper bound on
\(\ell\) in Corollary~\ref{cor:mu-T2}.  Consequently, \(L_{\mathrm{ent}}=\kappa\), so \(K_x\) does not
enter this exponent.  Independently, the \(h\)-uniform bounds on
\(K_{\mathrm{amp}}\) and \(\mathfrak M_a\) allow the parameter \(\lambda\)
in the exponential law of large numbers to be chosen uniformly in \(h\).
Hence

$$
    L_{\mathrm{ent}}
    =\kappa
    \le C_{a,d}\,\varepsilon^{-1}.
$$

In contrast, the general coupling estimate gives

$$
    L_{\mathrm{cpl}}
    \lesssim C_{a,d}\left(
    1+\varepsilon+h^{-1}+\frac{1}{\varepsilon h^2}\right).
$$

Thus, for fixed \(\varepsilon>0\), the moving-product entropy cutoff estimate gives us an \(h\)-independent exponent in the fixed-marginal \(W_2^2\) bound.  By contrast, using the  upper bound on \(L_{\mathrm{cpl}}\) into the  synchronous-coupling cutoff estimate gives a uniform-in-time PoC rate of the form $N^{-c_{\varepsilon,\alpha,a,d}h^2}$. The same comparison applies to the cutoff estimate of the empirical bound, although that bound also contains the sampling term \(r_{N,d}\).

This comparison concerns the power of \(N\) for each fixed \(h\), not an \(h\)-uniform multiplicative constant.  Indeed, for the
present kernel and target, the cross-term in \eqref{eq:Cstar-function} is
$$
\begin{aligned}
    C_h^\star(x)
    =
    \nabla_2k_h(x,x)\cdot\nabla V(x)
    +k_h(x,x)\Delta V(x)
    -\Delta_2k_h(x,x)
    =d\bigl(h+h^{-1}\bigr).
\end{aligned}
$$

Hence the constants in the finite-particle target estimate, and therefore
in the resulting cutoff bound, may diverge as \(h\downarrow0\).

\subsubsection{A contractive regime favoring coupling}
\label{sec:coupling-better-example}

Let

$$
    \pi=\mu_0=\mathcal N(0,I_d),
    \qquad
    V(x)=\frac12\|x\|^2,
$$

and take the constant kernel \(k(x,y)\equiv\beta>0\).  In this case,

$$
    K_{\mathrm{amp}}=K_{\mathrm{lip}}=\beta,
    \qquad
    K_x=0.
$$

Moreover, \(v_\pi=0\), so \(\mu_t=\pi\) and we may again take
\(\ell=0\).

Let \(Y_1,\ldots,Y_N\) be synchronously coupled independent nonlinear
copies.  They are stationary Ornstein--Uhlenbeck processes:

$$
    dY_i(t)
    =
    -\varepsilon Y_i(t)\,dt
    +\sqrt{2\varepsilon}\,dW_i(t).
$$

The particle system satisfies

$$
    dX_i^N(t)
    =
    \bigl\{-\beta\overline X^N(t)-\varepsilon X_i^N(t)\bigr\}\,dt
    +\sqrt{2\varepsilon}\,dW_i(t).
$$

Writing \(Z_i=X_i^N-Y_i\),
$$
    \dot Z_i(t)
    =
    -(\varepsilon+\beta)Z_i(t)-\beta\overline Y^N(t),
    \qquad Z_i(0)=0.
$$
Since
$$
    \E\|\overline Y^N(t)\|^2=\frac dN,
$$

Minkowski's inequality gives

$$
    \sup_{t\geq0}\E\|Z_i(t)\|^2
    \leq
    \frac{\beta^2d}{(\varepsilon+\beta)^2N}.
$$

Consequently, for every fixed \(k\),

$$
    \sup_{t\geq0}
    W_2^2(P_{N,t}^{(k)},\mu_t^{\otimes k})
    \leq
    \frac{k\beta^2d}{(\varepsilon+\beta)^2N}.
$$

Thus the direct synchronous coupling estimate yields the optimal \(N^{-1}\) error bound. By contrast, under the independent product law,
$$
    G_t^N=N\beta^2\|\overline Y^N(t)\|^2,
    \qquad
    \E e^{\lambda G_t^N}
    =(1-2\lambda\beta^2)^{-d/2},
    \qquad
    0<\lambda<\frac{1}{2\beta^2}.
$$
Taking, for example, \(\lambda=(4\beta^2)^{-1}\) gives
$$
    L_{\mathrm{ent}}=\kappa=\frac{2\beta^2}{\varepsilon}.
$$

With this particular admissible choice, the general
moving-product entropy cutoff estimate yields the fixed-marginal rate $ N^{-\varepsilon^2/(\beta^2+\varepsilon^2)}.$ More generally, any fixed \(0<\lambda<(2\beta^2)^{-1}\) gives the cutoff exponent $\frac{4\varepsilon^2\lambda}
         {1+4\varepsilon^2\lambda}
    <1.$ Thus the synchronous-coupling bound above is sharper than the bound obtained from the general moving-product entropy cutoff argument in this example. 

More generally, consider a possibly nonstationary
kernel family \(k_h\) such that, uniformly in \(h\),
$$
    \sup_{t\geq0}\int e^{a\|x\|^2}\mu_t^h(dx)<\infty,
    \qquad
    \mu_t^h\in T_2(Ce^{\ell_0t}),
    \qquad
    K_{\mathrm{amp}}(h)\leq C,
$$

while \(K_{\mathrm{lip}}(h)\asymp h^{-p}\) for some \(p>0\).
The transportation inequality assumption allows us to take \(\ell=\ell_0\) directly, without using the generic \(K_x\)-dependent propagation bound.  Hence
$
    L_{\mathrm{ent}}\leq C_\varepsilon+\ell_0
$
is uniform in \(h\), whereas the general coupling bound contains a term of order \(\varepsilon^{-1}h^{-2p}\).
Thus the same separation holds at the level of the cutoff powers obtained from the two arguments, without assuming that \(\mu_t^h\) is stationary.

Conversely, suppose a synchronous coupling satisfies

$$
    D_N'(t)\leq-cD_N(t)+\frac{C}{N}G_t^N,
    \qquad
    D_N(0)=0,
    \qquad
    \sup_{N,t}\E G_t^N<\infty,
$$

for some \(c>0\).  This can happen, for example, when
\(V(x)=\|x\|^2/2\), \(k\equiv\beta>0\), and
\(\mu_0=\mathcal N(m_0,\Sigma_0)\), where it is not necessarily equal to
\(\pi\).  A direct integration gives

$$
    \sup_{t\geq0}\E D_N(t)\leq\frac{C}{cN}.
$$
Hence this coupling estimate that exploits the structure of a model gives a uniform fixed-marginal \(N^{-1}\) bound, whereas the general entropy cutoff bound above has an exponent strictly smaller than \(1\) whenever \(L_{\mathrm{ent}}>0\).

\section{Concrete applications: smooth stationary kernels with convex
and nonconvex targets}
\label{sec:smooth-stationary-kernels}

The following corollary shows that our assumptions cover a broad class of
standard stationary kernels, which includes the Gaussian RBF kernel.

\begin{corollary}
\label{cor:smooth-stationary-kernels}
Let \(m,L>0\), and suppose that \(V\in C^3(\mathbb R^d)\) satisfies
\[
  m I_d \preceq \nabla^2 V(x) \preceq L I_d,
  \qquad x\in\mathbb R^d.
\]
Let
\[
  k(x,y)=\Psi(x-y),
\]
where \(\Psi\in C_b^4(\mathbb R^d)\) is even and \(k\) is positive
definite. Let \(\mu_0\) have a smooth positive density satisfying
\[
  \KL(\mu_0\Vert\pi)<\infty.
\]
Then Assumptions~\ref{ass:core-target} and~\ref{ass:ksd-kernel} hold,
with target log-Sobolev constant \(\alpha=m\). The kernel regularity
and diagonal-correction conditions in Assumption~\ref{ass:W2} also
hold. Consequently, Proposition~\ref{prop:wellposed} and the
target-convergence and uniform-second-moment results of
Sections~\ref{sec:target-KSD} and~\ref{sec:target-W2} apply.

If, in addition,
\[
  \int_{\mathbb R^d} e^{a_0\|x\|^2}\,\mu_0(dx)<\infty
\]
for some \(a_0>0\), then Assumption~\ref{ass:concentration} holds.  The
synchronous coupling and moving-product entropy bounds, and consequently the finite-time and uniform-in-time PoC results for the empirical measure in KSD, $W_2$, and finite-marginal PoC, therefore apply. If one further assumes
\(\mu_0\in T_2(C_0)\), the alternative entropy-derived Wasserstein bounds in
Section~\ref{app:entropy-W2} apply as well.
\end{corollary}

\begin{proof}
The lower Hessian bound implies that \(e^{-V}\) is integrable and the target
\(\pi(dx)=Z^{-1}e^{-V(x)}\,dx\) has finite second moment. Moreover,
the Bakry--\'Emery criterion shows that \(\pi\) satisfies
\eqref{eq:target-lsi} with \(\alpha=m\). The upper Hessian bound gives
\(
    \sup_{x\in\R^d}
    \|\nabla^2V(x)\|_{\mathrm{op}}
    \leq L,
\)
so Assumption~\ref{ass:core-target} follows.

Since \(\Psi\in C_b^4(\mathbb R^d)\), the kernel
\(k(x,y)=\Psi(x-y)\) belongs to
\(C_b^4(\mathbb R^d\times\mathbb R^d)\). In particular, it satisfies
all kernel-derivative bounds in Assumptions~\ref{ass:ksd-kernel}
and~\ref{ass:W2}.

Since \(\Psi\) is even, \(\nabla\Psi(0)=0\). Hence, on the diagonal,
\[
  k(x,x)=\Psi(0),\qquad
  \nabla_2 k(x,x)=0,\qquad
  \Delta_2 k(x,x)=\Delta\Psi(0).
\]
It follows from the definition of \(C^\star\) that
\[
  C^\star(x)
  =\Psi(0)\Delta V(x)-\Delta\Psi(0)
  \le dL\,|\Psi(0)|+|\Delta\Psi(0)|<\infty.
\]
Hence the diagonal-correction condition
\eqref{eq:Cstar-bound} holds.

We finally verify the dissipativity condition \eqref{eq:dissipativity} needed for the
strong-concentration regime. By the fundamental theorem of calculus and
Young's inequality,
\[
\begin{aligned}
    x\cdot\nabla V(x)
    &=
    x\cdot\nabla V(0)
    +\int_0^1
      x^\top\nabla^2V(sx)x\,ds \\
    &\geq
    -\|\nabla V(0)\|\,\|x\|+m\|x\|^2 
    \geq
    \frac{m}{2}\|x\|^2
    -\frac{\|\nabla V(0)\|^2}{2m}.
\end{aligned}
\]
Thus \eqref{eq:dissipativity} holds, with dissipativity constant
\(m/2\) and
\(
    b=\frac{\|\nabla V(0)\|^2}{2m}.
\)
Therefore the additional sub-Gaussian assumption on \(\mu_0\) implies
Assumption~\ref{ass:concentration}.  The conclusions follow from
Sections~\ref{sec:finite-time}, \ref{sec:cutoff} and \ref{sec:moving-entropy}, with the optional
entropy-to-Wasserstein conclusion supplied by
Section~\ref{app:entropy-W2}.
\end{proof}

\begin{corollary}[A nonconvex target covered by the theory]
\label{cor:nonconvex-example}
Let \(m,A,\omega>0\) satisfy \(A\omega^2>m\), and set
\[
    V(x)=\frac m2\|x\|^2+A\cos(\omega x_1).
\]
Let \(k(x,y)=\Psi(x-y)\), where
\(\Psi\in C_b^4(\mathbb R^d)\) is even and \(k\) is positive definite.
Suppose that \(\mu_0\) has a smooth positive density and satisfies
\[
    \KL(\mu_0\|\pi)<\infty,
    \qquad
    \int e^{a_0\|x\|^2}\,\mu_0(dx)<\infty
\]
for some \(a_0>0\).  Then Assumptions~\ref{ass:core-target} and
\ref{ass:concentration} hold, as do both kernel assumptions.
Consequently, all main propagation of chaos and uniform-in-time conclusions
apply to this nonconvex target.  
\end{corollary}

\begin{proof}
The Hessian is
\[
    \nabla^2V(x)
    =mI_d-A\omega^2\cos(\omega x_1)e_1e_1^\top.
\]
Thus
\[
    \sup_x\|\nabla^2V(x)\|_{\mathrm{op}}
    \le m+A\omega^2,
\]
whereas
\[
    e_1^\top\nabla^2V(0)e_1=m-A\omega^2<0.
\]
Hence the potential satisfies the bounded-Hessian condition but is not
convex.

Let \(\gamma_m\) denote the centered Gaussian law with potential
\(m\|x\|^2/2\). By the Gaussian log-Sobolev inequality (equivalently,
the Bakry--\'Emery criterion),
\[
    I(\nu\|\gamma_m)
    \ge 2m\,\KL(\nu\|\gamma_m).
\]
Moreover,
\[
    \pi(dx)
    \propto
    e^{-A\cos(\omega x_1)}\,\gamma_m(dx),
\]
so \(\pi\) is a bounded perturbation of \(\gamma_m\), with perturbation
\(A\cos(\omega x_1)\) of oscillation \(2A\).
The Holley--Stroock perturbation lemma
\cite{holleyStroock1987} implies that
\[
    I(\nu\|\pi)
    \ge 2m e^{-2A}\KL(\nu\|\pi).
\]
The same bounded-perturbation comparison shows directly that the normalizing
constant is finite and that \(\pi\) has a finite second moment.
In particular, Assumption~\ref{ass:core-target} holds with
\(\alpha=m e^{-2A}\).

Moreover,
\begin{align*}
    x\cdot\nabla V(x)
    &=m\|x\|^2-A\omega x_1\sin(\omega x_1)\\
    &\ge m\|x\|^2-A\omega|x_1|
    \ge\frac m2\|x\|^2-\frac{A^2\omega^2}{2m}.
\end{align*}
Thus Assumption~\ref{ass:concentration} holds, with dissipativity constant
\(m/2\).  Finally,
\[
    \Delta V(x)=dm-A\omega^2\cos(\omega x_1)
\]
is bounded, and hence
\[
    C^\star(x)
    =\Psi(0)\{dm-A\omega^2\cos(\omega x_1)\}
     -\Delta\Psi(0)
\]
is bounded above.  The remaining kernel conditions follow from
\(\Psi\in C_b^4\), as in
Corollary~\ref{cor:smooth-stationary-kernels}.
\end{proof}

\section{Conclusion and discussion}\label{sec:future}

We remark on the following possible extensions of the current work.

\subsection{Possible finite-moment localization}\label{weakprob}

When finiteness of exponential moments on the initial distribution is unavailable, \emph{uniform-in-time PoC estimates in probability}, in the spirit of \cite{balasubramanian2026}, can be sought by localization.  Set
\[
  \bar M_N(t):=\frac1N\sum_{i=1}^N\|\bar X_i(t)\|^2.
\]
Proposition~\ref{prop:second-moments} gives
\[
  \sup_{N\ge1,\,T>0}
  \E\left[\frac1T\int_0^T\bar M_N(s)\,ds\right]<\infty.
\]
Thus, for each $R>0$, the event
\[
  \left\{\frac1T\int_0^T\bar M_N(s)\,ds\le R\right\}
\]
has a complement of probability at most $C/R$, uniformly in $N$ and $T$.
On this event, the Gr\"onwall factor in \eqref{eq:DN-differential} is bounded
by $\exp\{C_R T\}$.  A localized estimate for the fluctuation
$G_t^N(\bar X(t))$ up to the cutoff time could therefore yield
uniform-in-time PoC in probability under finite-moment
assumptions.  This requires a maximal or stopped fluctuation estimate and is
not pursued here.  Such a metric localization would also not by itself
recover the fixed-marginal relative entropy and total variation estimates
provided by moving-product entropy.

\subsection{Dependence on the regularization strength}

Our estimates are proved for each
fixed $\varepsilon>0$ and are not uniform as $\varepsilon\downarrow0$.  The relative entropy decays at (exponential) rate
\(2\varepsilon\alpha\), while the resulting \(W_2\) and KSD bounds
decay at rate \(\varepsilon\alpha\), and the finite-time coupling and entropy
exponents are both of order $\varepsilon^{-1}$. The admissible
concentration constants may also depend on $\varepsilon$.  Consequently, the polynomial exponents for uniform-in-time PoC obtained from the cutoff
argument degenerate in the small-noise limit.

A possible way to avoid the degeneration coming from the convergence to the target is through a Stein log-Sobolev inequality. Writing
\(\mathcal I_k(\rho\|\pi)=\KSD^2(\rho\|\pi)\) for the Stein--Fisher
information, suppose that
\[
    \mathcal I_k(\rho\|\pi)
    \geq 2\alpha_{\mathrm S}\KL(\rho\|\pi).
\]
The mean-field dissipation identity would then give
\[
    \frac{d}{dt}\KL(\mu_t\|\pi)
    \leq
    -2(\alpha_{\mathrm S}+\varepsilon\alpha)
      \KL(\mu_t\|\pi),
\]
so the target-convergence rate would remain nondegenerate as
\(\varepsilon\downarrow0\), provided that
\(\alpha_{\mathrm S}>0\) is independent of \(\varepsilon\). Stein
log-Sobolev inequalities and related coercivity estimates have
recently been established for kernels with sufficiently strong
high-frequency behavior; see, for example,
\cite{carrilloSkrzeczkowskiWarnett2024,
chizatColomboColomboFernandezReal2026}.

To obtain an \(\varepsilon\)-uniform cutoff theorem, however, one
would additionally need a corresponding long-time target estimate for
the finite-particle system. A mean-field Stein log-Sobolev inequality
alone does not provide this; a finite-particle coercivity estimate,
a renormalized Stein-dissipation argument, or a suitable stability
theory near the target would still be required. Moreover, an \(\varepsilon\)-uniform finite-time PoC estimate would, for instance,
require controlling the random moment-dependent Gr\"onwall factor in \eqref{eq:DN-differential} uniformly in $\varepsilon$; such control does not follow directly from a
mean-field Stein log-Sobolev inequality. 

These directions generally lead beyond the bounded smooth kernels
considered here. Kernels with the required high-frequency coercivity
are typically of Riesz, Bessel, or Coulomb type and may be singular on
the diagonal. Their treatment would require new control of
self-interactions and empirical fluctuations, potentially through
removal or renormalization of the diagonal terms; see also
\cite{teolisDeHoop2026}.

\subsection*{Acknowledgements}

SB is partially supported by the NSF-CAREER award DMS-2141621 and the NSF RTG grant DMS-2134107. DK is supported by NSF CAREER award 2340762 of Prof. F. Hoffmann at California Institute of Technology.

AI-assisted tools were used in a limited capacity to help polish the writing and to check parts of the exposition and proofs. The authors independently verified all mathematical arguments and are responsible for the final content.

\appendix

\section{Entropy evolution equation}
\label{app:entropy-chain-rule}

In the following, for $\alpha\in(0,1)$, we write
\[
    b\in C_{\mathrm{loc}}^{\alpha/2,\,1+\alpha}
    \bigl((0,T]\times\mathbb R^{m};\mathbb R^{m}\bigr)
\]
if $b$ is continuously differentiable in the spatial variable and, for
every $\tau\in(0,T)$ and every compact set
$K\subset\mathbb R^{m}$,
\[
\begin{aligned}
    \sup_{(t,x)\in[\tau,T]\times K}
    \bigl(|b(t,x)|+|\nabla_xb(t,x)|\bigr)&<\infty,\\
    [b]_{t;\alpha/2,[\tau,T]\times K}
    +[\nabla_xb]_{t;\alpha/2,[\tau,T]\times K}
    +[\nabla_xb]_{x;\alpha,[\tau,T]\times K}&<\infty,
\end{aligned}
\]
where
\[
    [f]_{t;\beta,[\tau,T]\times K}
    :=
    \sup_{\substack{s,t\in[\tau,T],\,s\neq t\\x\in K}}
    \frac{|f(t,x)-f(s,x)|}{|t-s|^\beta}
\]
and
\[
    [f]_{x;\beta,[\tau,T]\times K}
    :=
    \sup_{\substack{t\in[\tau,T]\\x,y\in K,\,x\neq y}}
    \frac{|f(t,x)-f(t,y)|}{|x-y|^\beta}.
\]

 \begin{lemma}[Finite-dimensional relative entropy chain rule]
\label{lem:finite-dimensional-entropy}
Fix \(N\geq1\).  Let \(P_t^N\) and \(Q_t^N\) be the laws of two
diffusions on \((\R^d)^N\) with common diffusion coefficient
\(\sqrt{2\varepsilon}I\) and respective drifts \(b_t^N\) and
\(\bar b_t^N\), that is, with Fokker-Planck equations given by
\begin{equation}\label{eq:finite-dimensional-FP-pair}
    \partial_tP_t^N
    =
    -\nabla\cdot(b_t^NP_t^N)+\varepsilon\Delta P_t^N,
    \qquad
    \partial_tQ_t^N
    =
    -\nabla\cdot(\bar b_t^NQ_t^N)+\varepsilon\Delta Q_t^N,
\end{equation}
where \(\varepsilon>0\).
Suppose that \(P_0^N,Q_0^N\in\mathcal P_2((\R^d)^N)\)
and that, for every \(T<\infty\), \(b_t^N\) and
\(\bar b_t^N\) are in $C_{\mathrm{loc}}^{\alpha/2,\,1+\alpha}
    \bigl((0,T]\times\mathbb R^{dN};\mathbb R^{dN}\bigr)$ for some $\alpha \in (0,1)$, and
satisfy
\[
    \|b_t^N(X)\|+\|\bar b_t^N(X)\|
    \leq C_{N,T}(1+\|X\|),
    \qquad 0\leq t\leq T.
\]
If \(\KL(P_0^N\|Q_0^N)<\infty\), then 
\(u\mapsto\KL(P_u^N\|Q_u^N)\) is locally absolutely continuous and,
for \(0\leq s\leq t\),
\begin{align}
    \KL(P_t^N\|Q_t^N)-\KL(P_s^N\|Q_s^N)
    &=
    \int_s^t\int
       \left\langle
          b_u^N-\bar b_u^N,
          \nabla\log\frac{dP_u^N}{dQ_u^N}
       \right\rangle
       dP_u^N\,du
    \notag\\
    &\quad
    -\varepsilon\int_s^t
       I(P_u^N\|Q_u^N)\,du.
\label{eq:finite-dimensional-entropy-chain-rule}
\end{align}
\end{lemma}

\begin{proof}
The linear-growth bounds and the localized It\^o--Gr\"onwall argument
give, for every \(T<\infty\),
\begin{equation}\label{eq:unifsec}
    \sup_{0\leq u\leq T}
    \left\{
       \int\|X\|^2\,dP_u^N
       +
       \int\|X\|^2\,dQ_u^N
    \right\}
    <\infty.
\end{equation}
For every $0<s<T$, uniform ellipticity and the standard positivity
results for transition densities imply that, for every
$u\in[s,T]$, the measures $P_u^N$ and $Q_u^N$ admit strictly positive
densities $p_u^N$ and $q_u^N$, respectively. Moreover, the assumed
local regularity of the drifts and interior parabolic Schauder
estimates imply that, for every compact
$K\subset(\mathbb R^d)^N$,
\[
    (u,X)\longmapsto p_u^N(X),
    \qquad
    (u,X)\longmapsto q_u^N(X),
\]
belong to
\[
    C^{1+\alpha/2,\,2+\alpha}([s,T]\times K)
    \subset C^{1,2}([s,T]\times K).
\]
 Set
\[
    r_u^N:=\frac{p_u^N}{q_u^N},
    \qquad
    a_u^N:=b_u^N-\bar b_u^N.
\]
We will obtain \eqref{eq:finite-dimensional-entropy-chain-rule} through a localization argument.

Fix \(0<s<t\) and
let \(\chi_R\) be a smooth cutoff that equals one on the ball of radius
\(R\), vanishes outside the ball of radius \(2R\), and satisfies
\[
    \|\nabla\chi_R\|_\infty\leq C R^{-1},
    \qquad
    \|\nabla^2\chi_R\|_\infty\leq C R^{-2}.
\]
Define
\[
    \mathcal H_{N,R}(u)
    :=
    \int\chi_Rp_u^N\log r_u^N\,dX.
\]
Differentiating gives
\begin{equation}\label{eq:localized-entropy-first}
    \mathcal H_{N,R}'(u)
    =
    \int\chi_R(\partial_up_u^N)(1+\log r_u^N)\,dX
    -
    \int\chi_Rr_u^N\partial_uq_u^N\,dX.
\end{equation}

Substituting \eqref{eq:finite-dimensional-FP-pair}, the drift
contribution to \eqref{eq:localized-entropy-first} is
\begin{align*}
    &-\int
       \chi_R\nabla\cdot(b_u^Np_u^N)(1+\log r_u^N)\,dX
    +
    \int
       \chi_Rr_u^N\nabla\cdot(\bar b_u^Nq_u^N)\,dX
    \\
    &\quad=
    \int
       b_u^Np_u^N\cdot
       \left\{
          \chi_R\nabla\log r_u^N
          +(1+\log r_u^N)\nabla\chi_R
       \right\}\,dX
    -
    \int
       \bar b_u^Nq_u^N\cdot
       \left\{
          \chi_R\nabla r_u^N
          +r_u^N\nabla\chi_R
       \right\}\,dX
    \\
    &\quad=
    \int
       \chi_R
       \left\langle
          a_u^N,\nabla\log r_u^N
       \right\rangle
       p_u^N\,dX
    +
    \int
       \left\langle
          \bar b_u^N\log r_u^N
          +a_u^N(1+\log r_u^N),
          \nabla\chi_R
       \right\rangle
       p_u^N\,dX,
\end{align*}
where we used
\(
    \nabla r_u^N=r_u^N\nabla\log r_u^N.
\)

Similarly, the diffusion contribution is
\begin{align*}
    &\varepsilon\int
       \chi_R(\Delta p_u^N)(1+\log r_u^N)\,dX
    -
    \varepsilon\int
       \chi_Rr_u^N\Delta q_u^N\,dX
    \\
    &\quad=
    -\varepsilon\int
       \nabla p_u^N\cdot
       \left\{
          \chi_R\nabla\log r_u^N
          +(1+\log r_u^N)\nabla\chi_R
       \right\}\,dX
    +
    \varepsilon\int
       \nabla q_u^N\cdot
       \left\{
          \chi_R\nabla r_u^N
          +r_u^N\nabla\chi_R
       \right\}\,dX.
\end{align*}
Since
\[
    \nabla p_u^N
    =
    p_u^N\nabla\log r_u^N+r_u^N\nabla q_u^N,
\]
the terms not involving derivatives of \(\chi_R\) reduce to
\[
    -\varepsilon\int
       \chi_R
       \left\|\nabla\log r_u^N\right\|^2
       p_u^N\,dX.
\]
The remaining terms satisfy
\begin{align*}
    &-\varepsilon\int
       (1+\log r_u^N)\nabla p_u^N\cdot\nabla\chi_R\,dX
    +
    \varepsilon\int
       r_u^N\nabla q_u^N\cdot\nabla\chi_R\,dX
    \\
    &\quad=
    -\varepsilon\int
       \nabla(p_u^N\log r_u^N)\cdot\nabla\chi_R\,dX
=
    \varepsilon\int
       p_u^N\log r_u^N\,\Delta\chi_R\,dX.
\end{align*}
Consequently,
\begin{align}
    \mathcal H_{N,R}'(u)
    &=
    \int
       \chi_R
       \left\langle
          a_u^N,\nabla\log r_u^N
       \right\rangle
       p_u^N\,dX
    -
    \varepsilon\int
       \chi_R
       \left\|\nabla\log r_u^N\right\|^2
       p_u^N\,dX
    \notag\\
    &\quad+\mathcal R_{N,R}(u),
\label{eq:localized-entropy-exact}
\end{align}
where
\begin{align}
    \mathcal R_{N,R}(u)
    &:=
    \int
       \left\langle
          \bar b_u^N\log r_u^N
          +a_u^N(1+\log r_u^N),
          \nabla\chi_R
       \right\rangle
       p_u^N\,dX
    \notag\\
    &\quad+
    \varepsilon\int
       p_u^N\log r_u^N\,\Delta\chi_R\,dX.
\label{eq:localized-entropy-remainder}
\end{align}

We next remove the cutoff.  Let
\[
    A_R:=\{X:R\leq\|X\|\leq2R\}.
\]
The remainder is supported on \(A_R\).  The linear-growth assumption
gives, for \(u\leq T\),
\[
    \|a_u^N(X)\|+\|\bar b_u^N(X)\|
    \leq C_{N,T}(1+\|X\|).
\]
Since
\[
    (1+\|X\|)\|\nabla\chi_R(X)\|
    +|\Delta\chi_R(X)|
    \leq C
    \qquad\text{on }A_R,
\]
we obtain
\begin{equation}\label{eq:cutoff-remainder-bound}
    \int_s^t|\mathcal R_{N,R}(u)|\,du
    \leq
    C_{N,T}
    \int_s^t\int_{A_R}
       \left\{
          p_u^N|\log r_u^N|+p_u^N
       \right\}\,dX\,du.
\end{equation}
The elementary inequality
\[
    p_u^N|\log r_u^N|
    \leq
    p_u^N\log r_u^N+\frac{2}{e}q_u^N
\]
implies
\[
    \int p_u^N|\log r_u^N|\,dX
    \leq
    \KL(P_u^N\|Q_u^N)+\frac{2}{e}.
\]
Moreover, Girsanov's theorem yields
\begin{equation}\label{eq:girsanov-entropy-bound}
    \KL(P_u^N\|Q_u^N)
    \leq
    \KL(P_0^N\|Q_0^N)
    +
    \frac{1}{4\varepsilon}
    \int_0^u\int
       \|a_v^N\|^2\,dP_v^N\,dv.
\end{equation}
The drifts have linear growth, so the last integral is finite by the
finite-time second-moment estimate \eqref{eq:unifsec}.  Therefore
\[
    (u,X)\longmapsto
    p_u^N(X)|\log r_u^N(X)|+p_u^N(X)
\]
is integrable on
\([s,t]\times(\R^d)^N\).  Since
\(\mathbf 1_{A_R}(X)\to0\) pointwise, dominated convergence in
\eqref{eq:cutoff-remainder-bound} gives
\begin{equation}\label{eq:cutoff-remainder-vanishes}
    \int_s^t|\mathcal R_{N,R}(u)|\,du
    \longrightarrow0, \quad R \to \infty.
\end{equation}

It remains to pass to the limit in the two principal terms in \eqref{eq:localized-entropy-exact}.  Completing
the square gives
\begin{multline*}
    \int
       \chi_R
       \left\langle
          a_u^N,\nabla\log r_u^N
       \right\rangle
       p_u^N\,dX
    -
    \varepsilon\int
       \chi_R
       \left\|\nabla\log r_u^N\right\|^2
       p_u^N\,dX\\
    =
    \int \chi_R \frac{\|a_u^N\|^2}{4\varepsilon} p_u^N\,dX
    -
    \varepsilon \int \chi_R
    \left\|
       \nabla\log r_u^N-\frac{a_u^N}{2\varepsilon}
    \right\|^2 p_u^N\,dX.
\end{multline*}
The first term on the right is integrable by the linear drift growth and the finite-time moment
bound \eqref{eq:unifsec}.  Because \(\chi_R\uparrow1\), dominated convergence applies to
the first term and monotone convergence applies to the nonnegative
square in the second term.  Integrating
\eqref{eq:localized-entropy-exact} over \([s,t]\), using
\eqref{eq:cutoff-remainder-vanishes}, and letting \(R\to\infty\) gives
\begin{align}
    \KL(P_t^N\|Q_t^N)-\KL(P_s^N\|Q_s^N)
    &=
    \int_s^t\int
       \left\langle
          a_u^N,\nabla\log r_u^N
       \right\rangle
       p_u^N\,dX\,du
    \notag\\
    &\quad
    -
    \varepsilon\int_s^t\int
       \left\|\nabla\log r_u^N\right\|^2
       p_u^N\,dX\,du.
\label{eq:entropy-chain-rule-positive-time}
\end{align}
This proves the assertion for \(0<s\le t\).

Finally, \eqref{eq:girsanov-entropy-bound} and the finite-time moment
bound \eqref{eq:unifsec} give
\[
    \limsup_{s\downarrow0}
    \KL(P_s^N\|Q_s^N)
    \leq
    \KL(P_0^N\|Q_0^N).
\]
On the other hand, \(P_s^N\to P_0^N\) and \(Q_s^N\to Q_0^N\) weakly
as \(s\downarrow0\), so joint lower semicontinuity of relative entropy
gives the reverse inequality.  Hence
\[
    \lim_{s\downarrow0}
    \KL(P_s^N\|Q_s^N)
    =
    \KL(P_0^N\|Q_0^N).
\]
Letting \(s\downarrow0\) in
\eqref{eq:entropy-chain-rule-positive-time} proves
\eqref{eq:finite-dimensional-entropy-chain-rule} for \(s=0\) and
completes the proof.
\end{proof}

\section{Regularity results}\label{app:regularity}
The following lemma records the regularity properties of the particle and mean-field drifts needed for our purposes. 

\begin{lemma}[Regularity of the particle and mean-field drift]
\label{lem:nonlinear-drift-regularity}
Under Assumption~\ref{ass:core-target} and either kernel assumption, for
every \(T<\infty\),
\begin{equation}\label{eq:mf-time-increment}
    W_2(\mu_t,\mu_s)\le C_T|t-s|^{1/2},
    \qquad 0\le s,t\le T.
\end{equation}
Moreover,
\begin{align}
 &\|B(x,y)-B(x',y')\|
 +\|\nabla_1B(x,y)-\nabla_1B(x',y')\|
 \notag\\
 &\hspace{2cm}\le
 C_B(1+\min\{\|y\|,\|y'\|\})\left(\|x-x'\| + \|y-y'\|\right),
 \label{eq:B-weighted-y-lipschitz}
\end{align}
for all $x,x',y,y'\in\R^d$.

Consequently
\begin{equation}\label{eq:velocity-time-holder}
 \sup_{x\in\R^d}
 \left\{
   \|v_{\mu_t}(x)-v_{\mu_s}(x)\|
   +\|\nabla_xv_{\mu_t}(x)-\nabla_xv_{\mu_s}(x)\|
 \right\}
 \le C_T|t-s|^{1/2}.
\end{equation}
For every fixed \(N\), the configuration-space drifts in
\eqref{eq:particle-FP-JN} and \eqref{eq:product-FP-JN} therefore belong to
\[
 C_{\mathrm{loc}}^{\alpha/2,\,1+\alpha}
 \bigl((0,T]\times(\R^d)^N;(\R^d)^N\bigr),
 \qquad 0<\alpha<1,
\]
and have at most linear growth.
\end{lemma}

\begin{proof}
Let \(\bar X\) be the nonlinear process solving \eqref{eq:nonlinear-sde}.  The finite-time second-moment
bound and the linear growth of its drift give
\[
 \E\|\bar X(t)-\bar X(s)\|^2
 \le
 2|t-s|\int_s^t
   \E\|v_{\mu_u}(\bar X(u))
          -\varepsilon\nabla V(\bar X(u))\|^2\,du
 +4\varepsilon d|t-s|
 \le C_T|t-s|.
\]
The joint law of \((\bar X(t),\bar X(s))\) is a coupling of
\((\mu_t,\mu_s)\), which proves \eqref{eq:mf-time-increment}.

Using \eqref{eq:B}, write
\begin{align*}
B(x,y)-B(x',y')
={}&-[k(x,y)-k(x',y')]\nabla V(y')
-k(x,y)\{\nabla V(y)-\nabla V(y')\}\\
&+\nabla_2k(x,y)-\nabla_2k(x',y').
\end{align*}
The bounded first derivatives of $k$ give
\[
  |k(x,y)-k(x',y')|
  \le C(\|x-x'\|+\|y-y'\|),
\]
while \eqref{eq:hessian-bound} gives
$\|\nabla V(y)-\nabla V(y')\|\le L_V\|y-y'\|$ and
$\|\nabla V(y')\|\le\|\nabla V(0)\|+L_V\|y'\|$.  Finally, the bounded mixed and
second derivatives of $k$ imply the analogous Lipschitz bound for
$\nabla_2k$. This, along with a symmetric argument exchanging the roles of $(x,y)$ and $(x',y')$, gives the upper bound on $\|B(x,y)-B(x',y')\|$.
The bound on $\|\nabla_1B(x,y)-\nabla_1B(x',y')\|$ follows similarly using higher derivative bounds for $k$. Combining these estimates proves \eqref{eq:B-weighted-y-lipschitz}.

\eqref{eq:velocity-time-holder} follows from applying \eqref{eq:B-weighted-y-lipschitz} to an optimal coupling of
\(\mu_t\) and \(\mu_s\), and using Cauchy--Schwarz inequality, the finite-time
second-moment bound, and \eqref{eq:mf-time-increment}. 

The asserted local spatial regularity
follows from \(V\in C^3\), the kernel assumptions, and differentiation of
the drift vector-fields, along with finite-time moment bounds.  Its linear growth was established in the
proof of Proposition~\ref{prop:wellposed}.
\end{proof}

The next lemma records the required RKHS regularity of the Stein witness and formally establishes $S_\pi(\pi)=0$.

\begin{lemma}[Well-definedness and regularity of the Stein witness]
\label{lem:witness-lipschitz}
Under Assumption~\ref{ass:core-target} and either kernel assumption,
\(\mathcal H^d\) is separable, the map
\(x\mapsto\xi_\pi(x)\) is strongly measurable, and
\begin{equation}\label{eq:witness-growth}
    \|\xi_\pi(x)\|_{\mathcal H^d}\le C(1+\|x\|).
\end{equation}
Consequently, \(S_\pi(\rho)\) is well defined for every
\(\rho\in\mathcal P_1(\mathbb R^d)\), and
\[
    S_\pi(\pi)=0.
\]
Additionally, under Assumption~\ref{ass:ksd-kernel}, one also has
\begin{equation}\label{eq:witness-lipschitz}
    \|\xi_\pi(x)-\xi_\pi(y)\|_{\mathcal H^d}
    \le C(1+\|y\|)\|x-y\|.
\end{equation}
\end{lemma}

\begin{proof}
Since \(\mathbb R^d\) is separable and \(k\) is continuous, the linear span
of \(k(\cdot,q)\), \(q\in\mathbb Q^d\), is dense in \(\mathcal H\).
Thus \(\mathcal H\), and hence \(\mathcal H^d\), is separable.
The derivative-reproducing identities give
\[
    \|k(\cdot,x)\|_{\mathcal H}^2=k(x,x),
    \qquad
    \|\partial_{2,j}k(\cdot,x)\|_{\mathcal H}^2
    =\partial_{1,j}\partial_{2,j}k(x,x).
\]
The corresponding difference identities and continuity of the mixed
derivatives show that \(x\mapsto k(\cdot,x)\) and
\(x\mapsto\nabla_2k(\cdot,x)\) are continuous as RKHS-valued maps.
Thus \(\xi_\pi\) is continuous, hence strongly measurable.  Boundedness of
\(k\) and \(\nabla_{12}^2k\), together with
\(\|\nabla V(x)\|\le \|\nabla V(0)\|+L_V\|x\|\), proves
\eqref{eq:witness-growth}.  This also proves Bochner integrability under
every \(\rho\in\mathcal P_1(\mathbb R^d)\).

We next prove the Stein identity.  Let
\(\chi_R\in C_c^\infty(\mathbb R^d;[0,1])\) satisfy
\[
    \chi_R=1\ \text{on }B_R,\qquad
    \chi_R=0\ \text{outside }B_{2R},\qquad
    \|\nabla\chi_R\|_\infty\le C/R.
\]
For \(\phi\in\mathcal H^d\), the reproducing identities imply
\[
    \|\phi\|_\infty+\|\nabla\cdot\phi\|_\infty
    \le C\|\phi\|_{\mathcal H^d}.
\]
Integration by parts against \(e^{-V(x)}dx\) gives
\[
    \int \chi_R
       \{\nabla\cdot\phi-\nabla V\cdot\phi\}\,d\pi
    =-\int \nabla\chi_R\cdot\phi\,d\pi.
\]
The right-hand side tends to zero as \(R\to\infty\).  The integrand on the
left is dominated by
\(C(1+\|x\|)\|\phi\|_{\mathcal H^d}\), which is integrable because
\(\pi\in\mathcal P_2(\mathbb R^d)\).  Dominated convergence therefore yields
\[
    \int\mathcal T_\pi\phi\,d\pi=0.
\]
By the reproducing property and Bochner integrability,
\[
    \langle S_\pi(\pi),\phi\rangle_{\mathcal H^d}
    =\int\mathcal T_\pi\phi\,d\pi=0
\]
for every \(\phi\in\mathcal H^d\), proving \(S_\pi(\pi)=0\).

It remains to prove \eqref{eq:witness-lipschitz}.  Under
Assumption~\ref{ass:ksd-kernel}, the derivative-reproducing identities and
\(k\in C_b^4\) give
\[
    \|k(\cdot,x)-k(\cdot,y)\|_{\mathcal H}
    \le C\|x-y\|,
\]
and
\[
    \|\nabla_2k(\cdot,x)-\nabla_2k(\cdot,y)\|_{\mathcal H^d}
    \le C\|x-y\|.
\]
For completeness, fix $1\le j\le d$ and set
\[
    G_j(z):=\partial_{2,j}k(\cdot,z)\in\mathcal H.
\]
The derivative-reproducing calculus implies that $G_j$ is continuously
Fr\'echet differentiable as an \(\mathcal H\)-valued map and
\[
    D_\ell G_j(z)
    =\partial_{2,\ell}\partial_{2,j}k(\cdot,z).
\]
Moreover, the derivative-reproducing identity gives
\[
    \|D_\ell G_j(z)\|_{\mathcal H}^2
    =\partial_{1,\ell}\partial_{1,j}
      \partial_{2,\ell}\partial_{2,j}k(z,z)
    \le C
\]
uniformly in $z$, by $k\in C_b^4$.  The Hilbert-space fundamental
theorem of calculus along the segment from $x$ to $y$ therefore yields
\[
    \|G_j(x)-G_j(y)\|_{\mathcal H}\le C\|x-y\|.
\]
Summing over $j$ proves the second display above.  The same argument with
$z\mapsto k(\cdot,z)$, using bounded second mixed derivatives, proves the
first display.  Finally, adding and subtracting
\(-k(\cdot,x)\nabla V(y)\), and using the bounded Hessian of \(V\), gives
\eqref{eq:witness-lipschitz}.
\end{proof}

\printbibliography

\end{document}